%% file: TTUR_arXiv.tex
\documentclass{article}

\usepackage[nonatbib,final]{nips_2017}

\input{preamble}

\usepackage{sectsty}

\begin{document}

\title{GANs Trained by a Two Time-Scale Update Rule
Converge to a Local Nash Equilibrium}

\author{
 Martin Heusel
 \And
 Hubert Ramsauer
 \And
 Thomas Unterthiner
 \And
 Bernhard Nessler
 \And
 Sepp Hochreiter\\
  \\
 LIT AI Lab \& Institute of Bioinformatics, \\
 Johannes Kepler University Linz\\
 A-4040 Linz, Austria\\
\texttt{\{mhe,ramsauer,unterthiner,nessler,hochreit\}@bioinf.jku.at}
}

\maketitle

\begin{abstract}
Generative Adversarial Networks (GANs) excel at creating realistic images with
complex models for which maximum likelihood is infeasible. However, the
convergence of GAN training has still not been proved.
We propose a two time-scale update rule (TTUR) for training GANs with stochastic
gradient descent on arbitrary GAN loss functions.
TTUR has an individual learning rate for both the discriminator
and the generator. Using the theory of stochastic approximation, we
prove that the TTUR converges under mild assumptions to a stationary local Nash equilibrium.
The convergence carries over to the popular Adam optimization, for which we
prove that it follows the dynamics of a heavy ball with friction and thus
prefers flat minima in the objective landscape.
For the evaluation of the performance of GANs at image generation, we introduce
the `Fr\'{e}chet Inception Distance'' (FID) which captures the similarity of
generated images to real ones better than the Inception Score. In experiments,
TTUR improves learning for DCGANs and Improved Wasserstein GANs (WGAN-GP)
outperforming conventional GAN training on CelebA, CIFAR-10, SVHN, LSUN
Bedrooms, and the One Billion Word Benchmark.
\end{abstract}

\section*{Introduction}
\label{sec:introduction}

Generative adversarial networks (GANs) \cite{Goodfellow:14nips} have achieved
outstanding results in generating realistic images
\cite{Radford:15,Ledig:16,Isola:17,Arjovsky:17,Berthelot:17}
and  producing text \cite{Gulrajani:17}. GANs can learn complex generative models
for which maximum likelihood or a variational approximations are infeasible.
Instead of the likelihood, a discriminator network serves as objective for the
generative model, that is, the generator.
GAN learning is a game between the generator, which constructs synthetic data
from random variables, and the discriminator, which separates synthetic data
from real world data. The generator's goal is to construct data in such a way
that the discriminator cannot tell them apart from real world data.
Thus, the discriminator tries to minimize the synthetic-real discrimination
error while the generator tries to maximize this error.
Since training GANs is a game and its solution is a Nash equilibrium, gradient
descent may fail to converge
\cite{Salimans:16,Goodfellow:14nips,Goodfellow:17tutorial}.
Only \textit{local} Nash equilibria
are found, because gradient descent is a local optimization method.
If there exists a local neighborhood around a point in parameter space
where neither the generator nor the discriminator can unilaterally decrease
their respective losses, then we call this point a local Nash equilibrium.

To characterize the convergence properties of training general GANs is still an
open challenge \cite{Goodfellow:14criteria,Goodfellow:17tutorial}.
For special GAN variants, convergence can be proved under certain assumptions
\cite{Lim:17,Grnarova:17,Tolstikhin:17}. A prerequisit for many convergence
proofs is local stability \cite{Kushner:03} which was shown for GANs by
Nagarajan and Kolter \cite{Nagarajan:17} for a min-max GAN setting. However,
Nagarajan and Kolter require for their proof either 
rather strong and unrealistic assumptions or a restriction to a linear discriminator.
Recent convergence proofs for GANs hold for expectations over training samples
or for the number of examples going to infinity \cite{Li:17mmd,Mroueh:17fisher,Liu:17,Arora:17},
thus do not consider mini-batch learning which leads to a stochastic gradient
\cite{Wang:17,Hjelm:17,Mescheder:17,Li:17}.

Recently actor-critic learning has been analyzed using stochastic approximation.
Prasad et al.\ \cite{Prasad:15} showed that
a two time-scale update rule ensures that training reaches a
stationary local Nash equilibrium if the critic learns
faster than the actor.
Convergence was proved via an ordinary
differential equation (ODE), whose stable limit points coincide with stationary
local Nash equilibria. We follow the same approach.
We prove that GANs converge to a local Nash equilibrium when trained by a two
time-scale update rule (TTUR), i.e., when discriminator and generator have
separate learning rates. This also leads to better results in experiments.
The main premise is that the discriminator converges to a local minimum when the
generator is fixed. If the generator changes slowly enough, then the
discriminator still converges, since the generator perturbations are
small.
Besides ensuring convergence, the performance may also improve since the
discriminator must first learn new patterns before they are transferred to the
generator. In contrast, a generator which is overly fast, drives the
discriminator steadily into new regions without capturing its
gathered information.
In recent GAN implementations, the discriminator often learned faster than the
generator. A new objective slowed down the generator to prevent it from
overtraining on the current discriminator \cite{Salimans:16}.
The Wasserstein GAN algorithm uses more update steps for the discriminator than
for the generator \cite{Arjovsky:17}. We compare
TTUR and standard GAN training.
Fig.~\ref{fig:tturvsorig} shows at the left panel a stochastic gradient
example on CelebA for original GAN training (orig), which often leads to
oscillations, and the TTUR. On the right panel an example of a 4 node network
flow problem of Zhang et al.\ \cite{Zhang:07} is shown.
The distance between the actual parameter and its optimum for an one time-scale
update rule is shown across iterates. When the upper bounds on the errors are
small, the iterates return to a neighborhood of the optimal solution, while for
large errors the iterates may diverge (see also Appendix
Section~\ref{sec:equalTime}).

Our novel contributions in this paper are:

\begin{itemize}
\item The two time-scale update rule for GANs,
\item We proof that GANs trained with TTUR converge to a
stationary local Nash equilibrium,
\item The description of Adam as heavy ball with friction and the
  resulting second order differential equation,
\item The convergence of GANs trained with TTUR and Adam to a
stationary local Nash equilibrium,
\item We introduce the ``Fr\'{e}chet Inception Distance'' (FID) to evaluate
GANs, which is more consistent than the Inception Score.
\end{itemize}

\begin{figure} 
\centering
\includegraphics[width=0.49\textwidth, height=3.5cm]{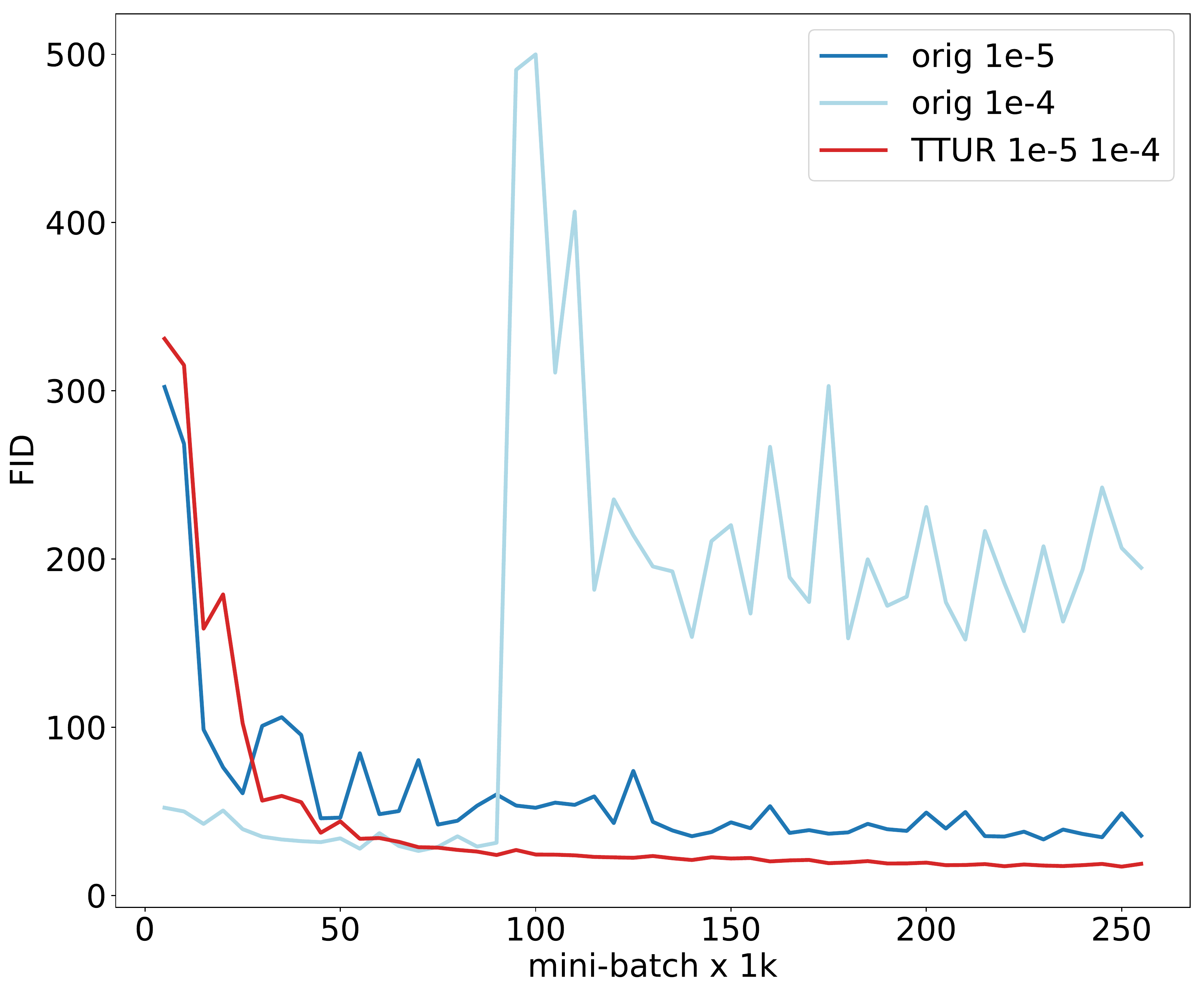}
\includegraphics[width=0.49\textwidth, height=3.5cm]{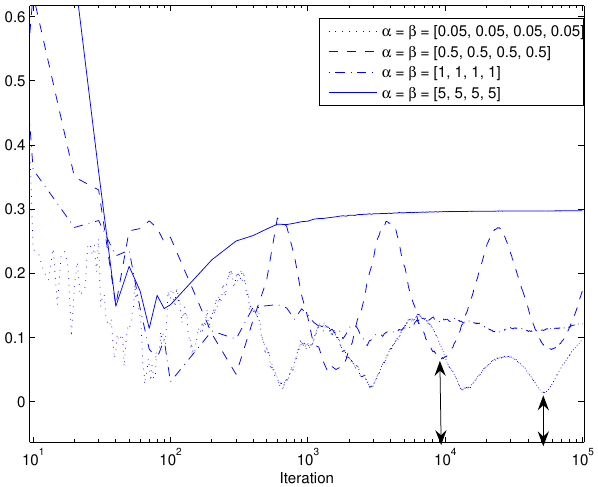}
\caption[Oscillation in GAN training]{Left: Original vs. TTUR GAN training on
CelebA. Right: Figure from Zhang 2007 \cite{Zhang:07} which shows the distance
of the parameter from the optimum for a one time-scale update
of a 4 node network flow problem.
When the upper bounds on the errors $(\alpha,\beta)$ are small,
the iterates oscillate and repeatedly return to a neighborhood of the optimal
solution (see also Appendix
Section~\ref{sec:equalTime}).
However, when the upper bounds on the errors are large, the iterates
typically diverge.
\label{fig:tturvsorig}}
\end{figure}

\section*{Two Time-Scale Update Rule for GANs}
\label{sec:TTUR}

We consider a
discriminator $D(.;\Bw)$ with parameter vector $\Bw$
and a generator $G(.;\Bth)$ with parameter vector $\Bth$.
Learning is based on a
stochastic gradient $\tilde{\Bg}(\Bth,\Bw)$
of the discriminator's loss function $\cL_D$
and a stochastic gradient $\tilde{\Bh}(\Bth, \Bw)$ of the generator's loss function $\cL_G$.
The loss functions $\cL_D$ and $\cL_G$
can be the original as introduced in Goodfellow et al. \cite{Goodfellow:14nips},
its improved versions \cite{Goodfellow:17tutorial}, or recently proposed losses for GANs like the
Wasserstein GAN \cite{Arjovsky:17}. Our setting is not restricted to min-max
GANs, but is valid for all other, more general GANs for which the
discriminator's loss function $\cL_D$ is not necessarily related to the
generator's loss function $\cL_G$.
The gradients
$\tilde{\Bg}\big(\Bth, \Bw\big)$ and $\tilde{\Bh}\big(\Bth, \Bw \big)$
are stochastic, since they use mini-batches of $m$ real world samples
$\Bx^{(i)},1\leq i \leq m$ and $m$ synthetic samples $\Bz^{(i)},1\leq
i \leq m$ which are randomly chosen.
If the true gradients are $\Bg(\Bth,\Bw)=\nabla_{w} \cL_D$ and
$\Bh(\Bth,\Bw)=\nabla_{\theta}\cL_G$, then we can define
$\tilde{\Bg}(\Bth,\Bw)=\Bg(\Bth,\Bw)+\BM^{(w)}$ and
$\tilde{\Bh}(\Bth,\Bw)=\Bh(\Bth,\Bw)+\BM^{(\theta)}$ with random
variables $\BM^{(w)}$ and $\BM^{(\theta)}$.
Thus, the gradients
$\tilde{\Bg}\big(\Bth, \Bw\big)$ and $\tilde{\Bh}\big(\Bth, \Bw \big)$
are stochastic approximations to the true gradients.
Consequently, we analyze convergence of GANs
by two time-scale stochastic approximations algorithms.
For a two time-scale update rule (TTUR),
we use the learning rates $b(n)$ and $a(n)$ for the discriminator and
the generator update, respectively:
\begin{align}
\label{eq:iterBorkar}
\Bw_{n+1}   &=  \Bw_n   +\ b(n)
\left(\Bg\big(\Bth_n, \Bw_n\big)  +  \BM^{(w)}_{n}\right),\
\Bth_{n+1}   =   \Bth_n  +  a(n) \
\left(\Bh\big(\Bth_n, \Bw_n \big)  +  \BM^{(\theta)}_{n}\right)  .
\end{align}
For more details on the following convergence proof and its assumptions
see Appendix Section~\ref{sec:convergence}.
To prove convergence of GANs learned by TTUR, we make the following
assumptions (The actual assumption is ended by $\blacktriangleleft$, the
following text are just comments and explanations):
\begin{enumerate}[label=\textbf{(A\arabic*)}]
\item The gradients $\Bh$ and  $\Bg$ are Lipschitz. $\blacktriangleleft$
Consequently, networks with Lipschitz smooth activation functions like
ELUs ($\alpha=1$) \cite{Clevert:16}
fulfill the assumption but not ReLU networks.
\item $\sum_{n} a(n) = \infty$, $\sum_{n} a^2(n) <  \infty$,
$\sum_{n} b(n) =  \infty$, $\sum_{n} b^2(n) <  \infty$, $a(n) \ = \ \Ro(b(n))$$\blacktriangleleft$
\item The stochastic gradient errors
$\{\BM^{(\theta)}_n\}$ and $\{\BM^{(w)}_n\}$ are martingale
difference sequences w.r.t.\ the increasing $\sigma$-field
$\cF_n = \sigma(\Bth_l, \Bw_l, \BM^{(\theta)}_{l}, \BM^{(w)}_{l},
l \leq n), n \geq 0$ with
$\rE \left[\|\BM^{(\theta)}_n \|^2 \mid \cF_n^{(\theta)} \right]
\leq  B_1$ and
$\rE \left[\|\BM^{(w)}_n \|^2 \mid \cF_n^{(w)} \right]
 \leq  B_2$, where $B_1$ and $B_2$ are positive deterministic
constants.$\blacktriangleleft$
The original Assumption (A3) from Borkar 1997
follows from Lemma~2 in \cite{Bertsekas:00}
(see also \cite{Ramaswamy:16}).
The assumption is fulfilled in the
Robbins-Monro setting, where mini-batches are randomly
sampled and the gradients are bounded.
\item For each $\Bth$, the ODE
$\dot{\Bw}(t) = \Bg\big(\Bth, \Bw(t)\big)$
has a local asymptotically stable attractor
$\Bla(\Bth)$ within a domain of attraction $G_{\theta}$
such that $\Bla$ is Lipschitz.
The ODE
$\dot{\Bth}(t) = \Bh\big(\Bth(t), \Bla(\Bth(t))\big)$
has a local asymptotically stable
attractor $\Bth^{*}$ within a domain of attraction.$\blacktriangleleft$
The discriminator must converge
to a minimum for fixed generator parameters and
the generator, in turn, must converge to a minimum for this fixed discriminator
minimum.
Borkar 1997 required unique global asymptotically stable equilibria \cite{Borkar:97}.
The assumption of global attractors was relaxed to local attractors via
Assumption (A6) and Theorem~2.7
in Karmakar \& Bhatnagar \cite{Karmakar:17}.
See for more details Assumption (A6) in the Appendix
Section~\ref{sec:addnoisecmp}.
Here, the GAN objectives may serve as Lyapunov functions.
These assumptions of locally stable
ODEs can be ensured by an additional weight decay term in the loss function
which increases the eigenvalues of the Hessian. Therefore, problems with a
region-wise constant discriminator that has zero second order derivatives are
avoided. For further discussion see Appendix Section~\ref{sec:background}
(C3).
\item
$\sup_n \| \Bth_n \| < \infty$ and $\sup_n \| \Bw_n \| <  \infty$.$\blacktriangleleft$
Typically ensured by the objective or a weight decay term.
\end{enumerate}
The next theorem has been proved in the seminal paper of Borkar 1997 \cite{Borkar:97}.
\begin{theorem}[Borkar]
\label{th:borkar}
If the assumptions are satisfied,
then the updates Eq.~\eqref{eq:iterBorkar}
converge to $(\Bth^{*}, \Bla(\Bth^{*}))$ a.s.
\end{theorem}
The solution $(\Bth^{*}, \Bla(\Bth^{*}))$ is a
stationary local Nash equilibrium \cite{Prasad:15}, since
$\Bth^{*}$ as well
as $\Bla(\Bth^{*})$ are local asymptotically stable
attractors with $\Bg\big(\Bth^{*}, \Bla(\Bth^{*})\big)=\BZe$ and
$\Bh\big(\Bth^{*},\Bla(\Bth^{*})\big)=\BZe$.
An alternative approach to the proof of convergence
using the Poisson equation for ensuring a
solution to the fast update rule can be found in the Appendix
Section~\ref{sec:linupnoisemc}.
This approach assumes a linear update function in the fast
update rule which, however,
can be a linear approximation to a nonlinear gradient \cite{Konda:02,Konda:03}.
For the rate of convergence see Appendix Section~\ref{sec:convergenceRate},
where Section~\ref{sec:linur} focuses on linear and
Section~\ref{sec:nonlinur} on non-linear updates.
For equal time-scales it can only be proven that
the updates revisit an environment of the solution infinitely often,
which, however, can be very large \cite{Zhang:07,DiCastro:10}.
For more details on the analysis of
equal time-scales see Appendix Section~\ref{sec:equalTime}.
The main idea of the proof of Borkar
\cite{Borkar:97} is to use
$(T,\delta)$ perturbed ODEs according to Hirsch 1989 \cite{Hirsch:89}
(see also Appendix Section~C of  Bhatnagar, Prasad, \& Prashanth 2013 \cite{Bhatnagar:13}).
The proof relies on the fact that
there eventually is a time point when the perturbation of the slow
update rule is
small enough (given by $\delta$) to allow the fast update rule to converge.
For experiments with TTUR, we aim at
finding learning rates such that the slow update is small enough
to allow the fast to converge.
Typically, the slow update is the generator and the fast update the
discriminator.
We have to adjust the two learning rates such that
the generator does not affect discriminator learning
in a undesired way and perturb it too much.
However, even a larger learning rate for the generator than for the
discriminator may ensure that the discriminator has low perturbations.
Learning rates cannot be translated directly into perturbation since
the perturbation of the discriminator by the generator is
different from the perturbation of the generator by the discriminator.

\section*{Adam Follows an HBF ODE and Ensures TTUR Convergence}
\label{sec:adamhbf}

\begin{figwindow}[3,r,{
\includegraphics[width=0.4\textwidth]{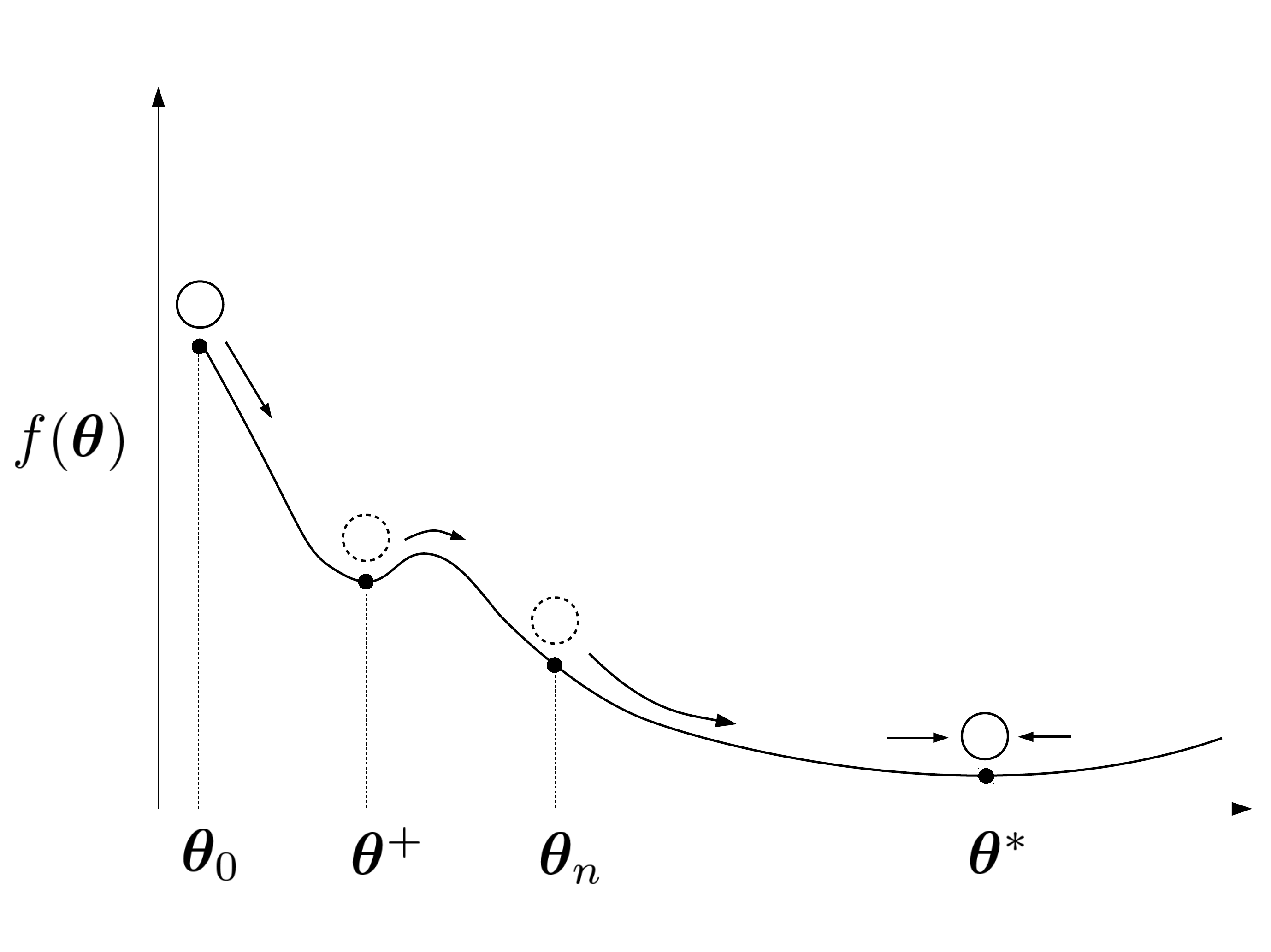}
},[Heavy Ball with Friction]{Heavy Ball with Friction,
where the ball with mass overshoots the local minimum $\Bth^+$ and settles at
the flat minimum $\Bth^*$.
  \label{fig:HBF} }]
In our experiments, we aim at using Adam stochastic approximation  to avoid mode
collapsing. GANs suffer from ``mode collapsing'' where large masses of
probability are mapped onto a few modes that cover only small regions.
While these regions represent meaningful samples, the variety of the real world
data is lost and only few prototype samples are generated.
Different methods have been proposed to avoid mode collapsing
\cite{Che:17,Metz:16}. We obviate mode collapsing by using Adam stochastic
approximation \cite{Kingma:14}. Adam can be described as Heavy Ball with
Friction (HBF) (see below), since it averages over past gradients.
This averaging corresponds to a velocity that makes the generator resistant to
getting pushed into small regions. Adam as an HBF method typically overshoots
small local minima that correspond to mode collapse and can find flat minima
which generalize well \cite{Hochreiter:97nc1}. Fig.~\ref{fig:HBF} depicts the
dynamics of HBF, where the ball settles at a flat minimum. Next, we analyze
whether GANs trained with TTUR converge when using Adam. For more details see
Appendix Section~\ref{sec:adam}.
\end{figwindow}

We recapitulate the Adam update rule at step $n$, with learning rate $a$,
exponential averaging factors $\beta_1$ for the first and $\beta_2$
for the second moment of the gradient $\nabla f(\Bth_{n-1})$:
\begin{align}
\label{eq:adam}
\Bg_n \ &\longleftarrow \  \nabla f(\Bth_{n-1}) \\ \nonumber
\Bm_n \ &\longleftarrow \ (\beta_1/(1-\beta_1^n)) \ \Bm_{n-1} \ + \
((1-\beta_1)/(1-\beta_1^n)) \ \Bg_n \\ \nonumber
\Bv_n \ &\longleftarrow \ (\beta_2/(1-\beta_2^n)) \ \Bv_{n-1} \ + \
          ((1-\beta_2)/(1-\beta_2^n)) \ \Bg_n \odot \Bg_n \\ \nonumber
\Bth_n \ &\longleftarrow \ \Bth_{n-1} \ - \ a \ \Bm_n /(\sqrt{\Bv_n}+\epsilon)
\ ,
\end{align}
where following operations are meant componentwise:
the product $\odot$,
the square root $\sqrt{.}$, and
the division $/$ in the last line.
Instead of learning rate $a$, we introduce the damping
coefficient $a(n)$ with $a(n)=a n^{-\tau}$ for $\tau \in (0,1]$.
Adam has parameters $\beta_1$ for averaging the
gradient and $\beta_2$ parametrized by a positive $\alpha$ for
averaging the squared gradient.
These parameters can be considered as defining a memory for Adam.
To characterize $\beta_1$ and $\beta_2$ in the following, we define
the exponential memory $r(n)=r$  and
the polynomial memory $r(n)=r/\sum_{l=1}^{n} a(l)$ for some positive
constant $r$. The next theorem describes Adam by a differential
equation, which in turn allows to apply the idea of
$(T,\delta)$ perturbed ODEs to TTUR.
Consequently, learning GANs with TTUR and
Adam converges.
\begin{theorem}
\label{th:adam}
If Adam is used with $\beta_1=1-a(n+1)r(n)$,
$\beta_2=1-\alpha a(n+1)r(n)$
and with $\nabla f$ as the full gradient of the lower bounded,
continuously differentiable objective $f$,
then for stationary second moments of the gradient,
Adam follows the differential equation
for Heavy Ball with Friction (HBF):
\begin{align}
\label{eq:Diff}
&\ddot{\Bth}_t \ + \ a(t) \ \dot{\Bth}_t \ + \ \nabla f(\Bth_t)
\ = \ \BZe  \ .
\end{align}
Adam converges for gradients $\nabla f$ that are $L$-Lipschitz.
\end{theorem}
\begin{proof}
Gadat et al.\ derived a discrete and stochastic
version of Polyak's Heavy Ball method \cite{Polyak:64},
the Heavy Ball with Friction (HBF) \cite{Gadat:16}:
\begin{align}
\label{eq:hbf1}
\Bth_{n+1} \ &= \ \Bth_n \ - \ a(n+1) \ \Bm_n \ , \\ \nonumber
\Bm_{n+1} \ &= \ \big( 1 \ - \  a(n+1) \ r(n) \big) \ \Bm_n \ + \  a(n+1) \ r(n) \
\big( \nabla f(\Bth_n) \ + \ \BM_{n+1} \big) \ .
\end{align}
These update rules are the first moment update rules
of Adam \cite{Kingma:14}.
The HBF can be formulated as the differential equation
Eq.~\eqref{eq:Diff} \cite{Gadat:16}.
Gadat et al.\ showed that the update rules
Eq.~\eqref{eq:hbf1} converge
for loss functions $f$ with at most quadratic grow and
stated that convergence can be proofed
for $\nabla f$ that are $L$-Lipschitz \cite{Gadat:16}.
Convergence has been proved for continuously differentiable $f$ that
is quasiconvex (Theorem~3 in Goudou \& Munier \cite{Goudou:09}).
Convergence has been proved for $\nabla f$ that is $L$-Lipschitz and
bounded from below (Theorem~3.1 in Attouch et al.\ \cite{Attouch:00}).
Adam normalizes the average $\Bm_n$ by the second moments $\Bv_n$ of
of the gradient $\Bg_n$: $\Bv_n=\rE \left[ \Bg_n \odot \Bg_n\right]$.
$\Bm_n$ is componentwise divided by the square root of
the components of $\Bv_n$.
We assume that the second moments of $\Bg_n$ are stationary,
i.e., $\Bv=\rE \left[ \Bg_n \odot \Bg_n\right]$.
In this case the normalization can be considered as additional noise
since the normalization factor randomly deviates from its mean.
In the HBF interpretation the normalization by $\sqrt{\Bv}$
corresponds to introducing gravitation.
We obtain
\begin{align}
\Bv_n \ &= \ \frac{1  -  \beta_2}{1  - \beta_2^n} \ \sum_{l=1}^{n}
        \beta_2^{n-l} \ \Bg_l \odot \Bg_l \ , \ \
\Delta \Bv_n \ = \  \Bv_n - \Bv \ = \
\frac{1  - \beta_2}{1  - \beta_2^n} \  \sum_{l=1}^{n}
        \beta_2^{n-l} \ \left( \Bg_l \odot \Bg_l - \Bv \right) \ .
\end{align}
For a stationary second moment $\Bv$ and $\beta_2=1-\alpha a(n+1)r(n)$, we have
$\Delta \Bv_n \propto a(n+1)r(n)$.
We use a componentwise
linear approximation to Adam's second moment normalization
$1/\sqrt{\Bv+\Delta \Bv_n} \approx 1/\sqrt{\Bv} - (1/(2 \Bv \odot
\sqrt{\Bv}))  \odot \Delta \Bv_n
+ \rO(\Delta^2 \Bv_n)$, where all operations are meant componentwise.
If we set $\BM_{n+1}^{(v)}=- (\Bm_n \odot \Delta \Bv_n)/(2 \Bv \odot
\sqrt{\Bv}a(n+1)r(n))$, then $\Bm_n / \sqrt{\Bv_n} \approx \Bm_n / \sqrt{\Bv} +
a(n+1)r(n)\BM_{n+1}^{(v)}$ and
$\rE \left[\BM_{n+1}^{(v)} \right] = \BZe$, since $\rE \left[\Bg_l \odot \Bg_l - \Bv \right] = \BZe$.
For a stationary second moment $\Bv$, the random variable $\{\BM^{(v)}_n\}$ is a martingale
difference sequence with a bounded second moment.
Therefore $\{\BM^{(v)}_{n+1}\}$ can be subsumed into  $\{\BM_{n+1}\}$ in update
rules Eq.~\eqref{eq:hbf1}. The factor $1 / \sqrt{\Bv}$ can be componentwise
incorporated into the gradient $\Bg$ which corresponds to rescaling
the parameters without changing the minimum.
\end{proof}
According to Attouch et al.\ \cite{Attouch:00} the energy, that is,
a Lyapunov function, is $E(t) =  1/2 | \dot{\Bth}(t) |^2 + f(\Bth(t))$
and $\dot{E}(t) = - a \ | \dot{\Bth}(t) |^2 < 0$.
Since Adam can be expressed as differential equation and has a
Lyapunov function, the idea of
$(T,\delta)$ perturbed ODEs \cite{Borkar:97,Hirsch:89,Borkar:00}
carries over to Adam.
Therefore the convergence of Adam with TTUR can be proved via two time-scale stochastic
approximation analysis like in Borkar \cite{Borkar:97}
for stationary second moments of the gradient.

In the Appendix we further discuss the convergence of two time-scale
stochastic approximation algorithms with additive noise, linear update
functions depending on Markov chains, nonlinear update functions,
and updates depending on controlled Markov processes. Futhermore, the Appendix
presents work on the rate of convergence for both linear and nonlinear update
rules using similar techniques as the local stability analysis of Nagarajan
and Kolter \cite{Nagarajan:17}.
Finally, we elaborate more on equal time-scale updates, which
are investigated for saddle point problems and actor-critic learning.

\section*{Experiments}
\label{sec:experiments}

\paragraph{Performance Measure.}
Before presenting the experiments, we introduce a quality measure
for models learned by GANs.
The objective of generative learning is that the model
produces data which matches the observed data.
Therefore, each distance between the probability of observing real world data
$p_w(.)$ and the probability of generating model data $p(.)$
can serve as performance measure for generative models.
However, defining appropriate
performance measures for generative models is difficult \cite{Theis:15}.
The best known measure is the likelihood,
which can be estimated by annealed importance sampling \cite{Wu:16}.
However, the likelihood heavily depends on the noise assumptions
for the real data and can be dominated by single samples \cite{Theis:15}.
Other approaches like density estimates have drawbacks, too \cite{Theis:15}.
A well-performing approach to measure the performance of GANs
is the ``Inception Score'' which correlates with human
judgment \cite{Salimans:16}.
Generated samples are fed into an inception model that was trained on
ImageNet. Images with meaningful objects are supposed to have low
label (output) entropy, that is, they belong to few object classes.
On the other hand, the entropy across images should be high, that is,
the variance over the images should be large.
Drawback of the Inception Score is that
the statistics of real world samples are not used and compared to
the statistics of synthetic samples.
Next, we improve the Inception Score.
The equality $p(.)=p_w(.)$ holds except for a non-measurable set
if and only if $\int p(.) f(x) dx=\int p_w(.) f(x) dx$ for
a basis $f(.)$ spanning the function space in which $p(.)$ and $p_w(.)$
live. These equalities of expectations are used to describe distributions
by moments or cumulants, where $f(x)$ are polynomials of the data $x$.
We generalize these polynomials by replacing $x$ by the coding layer of
an inception model in order to obtain vision-relevant features.
For practical reasons we only consider the first two polynomials, that
is, the first two moments: mean and covariance.
The Gaussian is the maximum entropy distribution for given
mean and covariance, therefore we assume the coding units to follow a
multidimensional Gaussian.
The difference of two Gaussians (synthetic and real-world images)
is measured by the Fr\'{e}chet
distance \cite{Frechet:57}
also known as Wasserstein-2 distance \cite{Wasserstein:69}.
We call the Fr\'{e}chet
distance $d(.,.)$ between the Gaussian with mean $(\Bm,\BC)$ obtained
from $p(.)$ and the Gaussian with mean $(\Bm_w,\BC_w)$ obtained
from $p_w(.)$ the ``Fr\'{e}chet Inception Distance'' (FID), which is
given by \cite{Dowson:82}:
\begin{align}
d^2((\Bm,\BC),(\Bm_w,\BC_w))=\|\Bm-\Bm_w\|_2^2+  \TR \bigl(\BC+\BC_w-2\bigl(
\BC\BC_w\bigr)^{1/2}\bigr) \ .
\end{align}
Next we show that the FID is consistent with
increasing disturbances and human judgment.
Fig.~\ref{fig:FIDscore} evaluates the FID for Gaussian noise,
Gaussian blur, implanted black rectangles,
swirled images, salt and pepper noise, and
CelebA dataset contaminated by ImageNet images.
The FID captures the disturbance level very well.
In the experiments we used the FID to evaluate the performance of GANs.
For more details and a comparison between FID and Inception Score
see Appendix Section~\ref{sec:fid}, where we show that FID is \textit{more
consistent} with the noise level than the Inception Score.

\begin{figure} 
\includegraphics[width=0.33\textwidth, height=3cm]{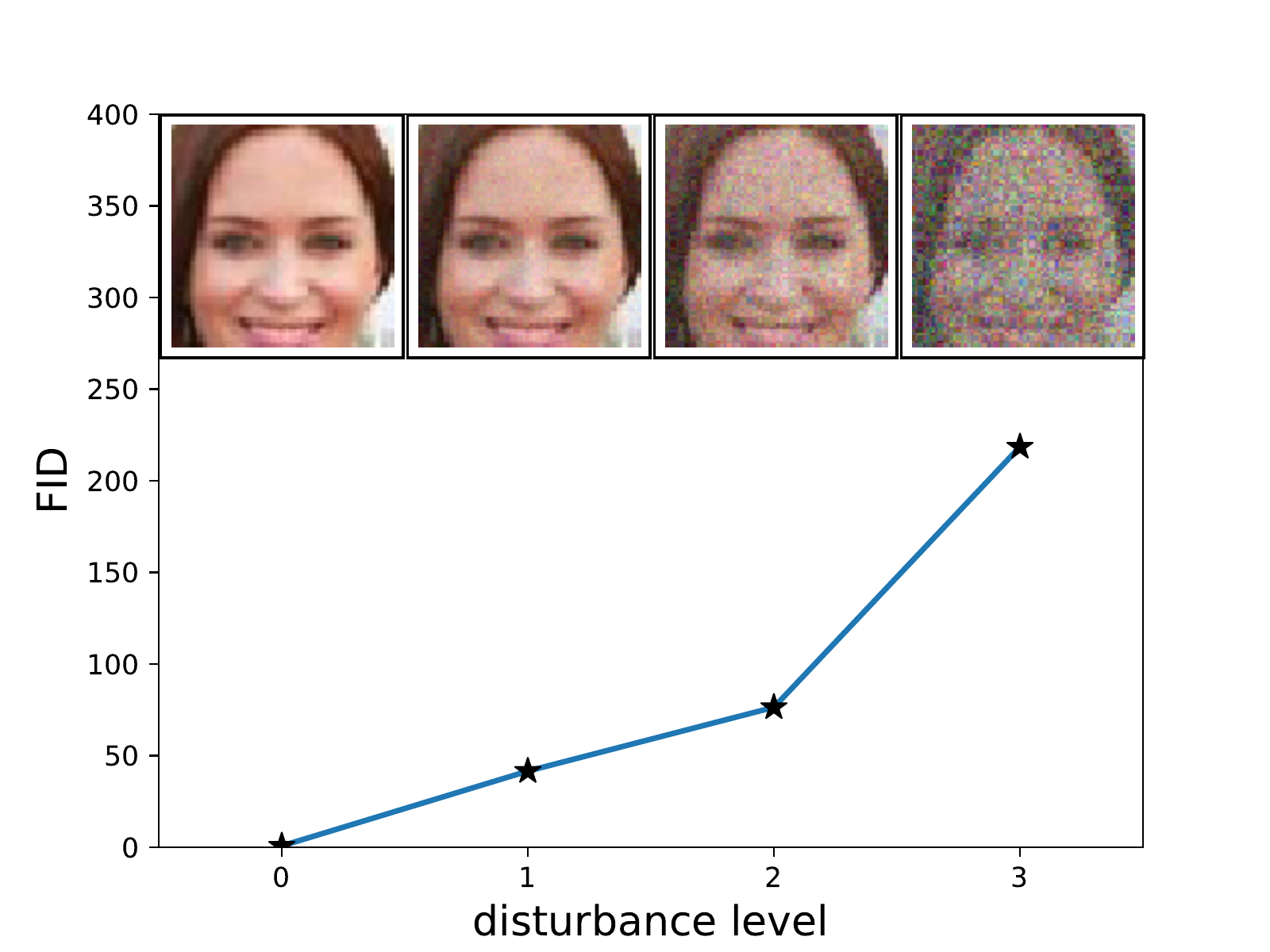}
\includegraphics[width=0.33\textwidth, height=3cm]{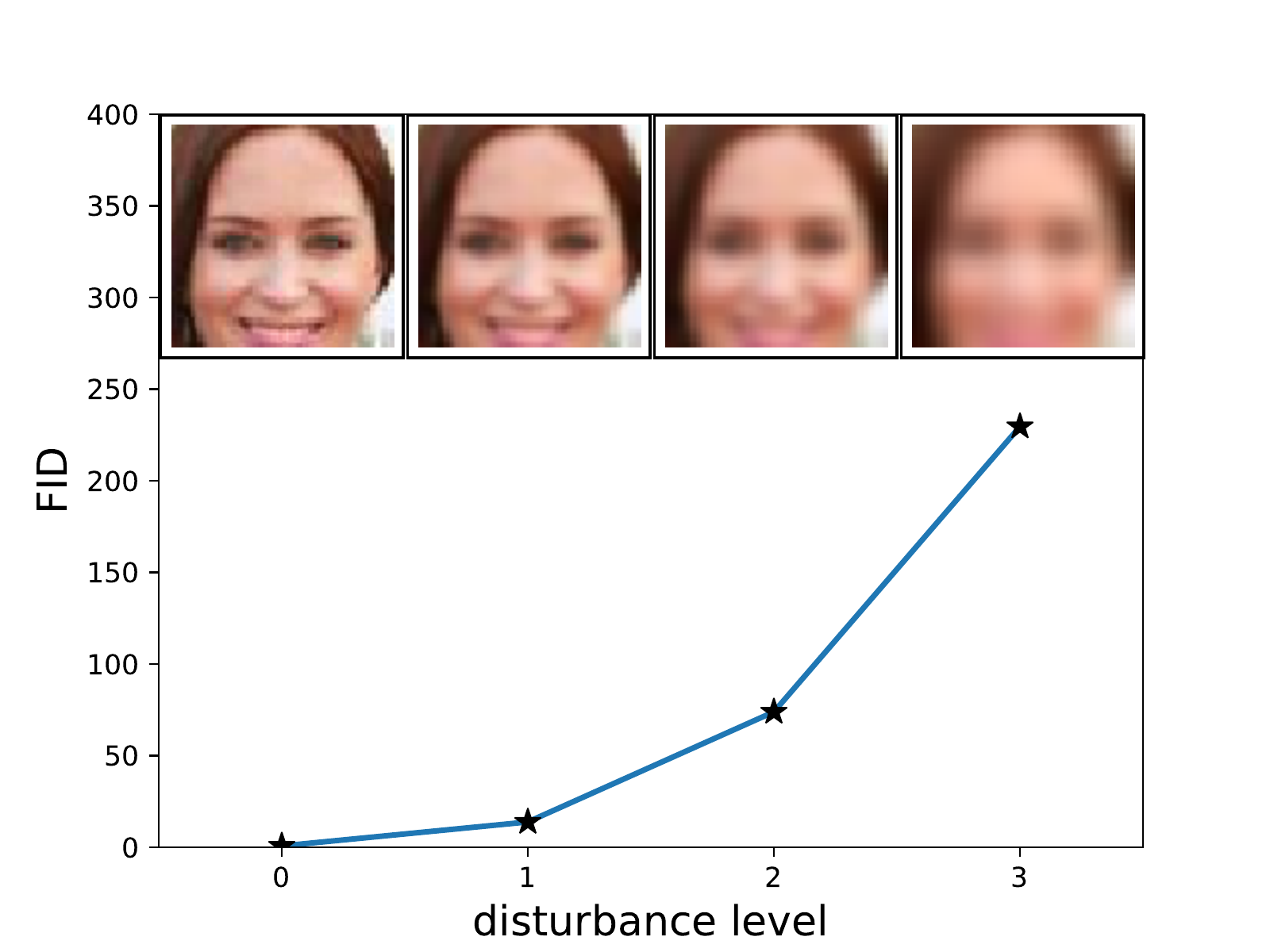}
\includegraphics[width=0.33\textwidth, height=3cm]{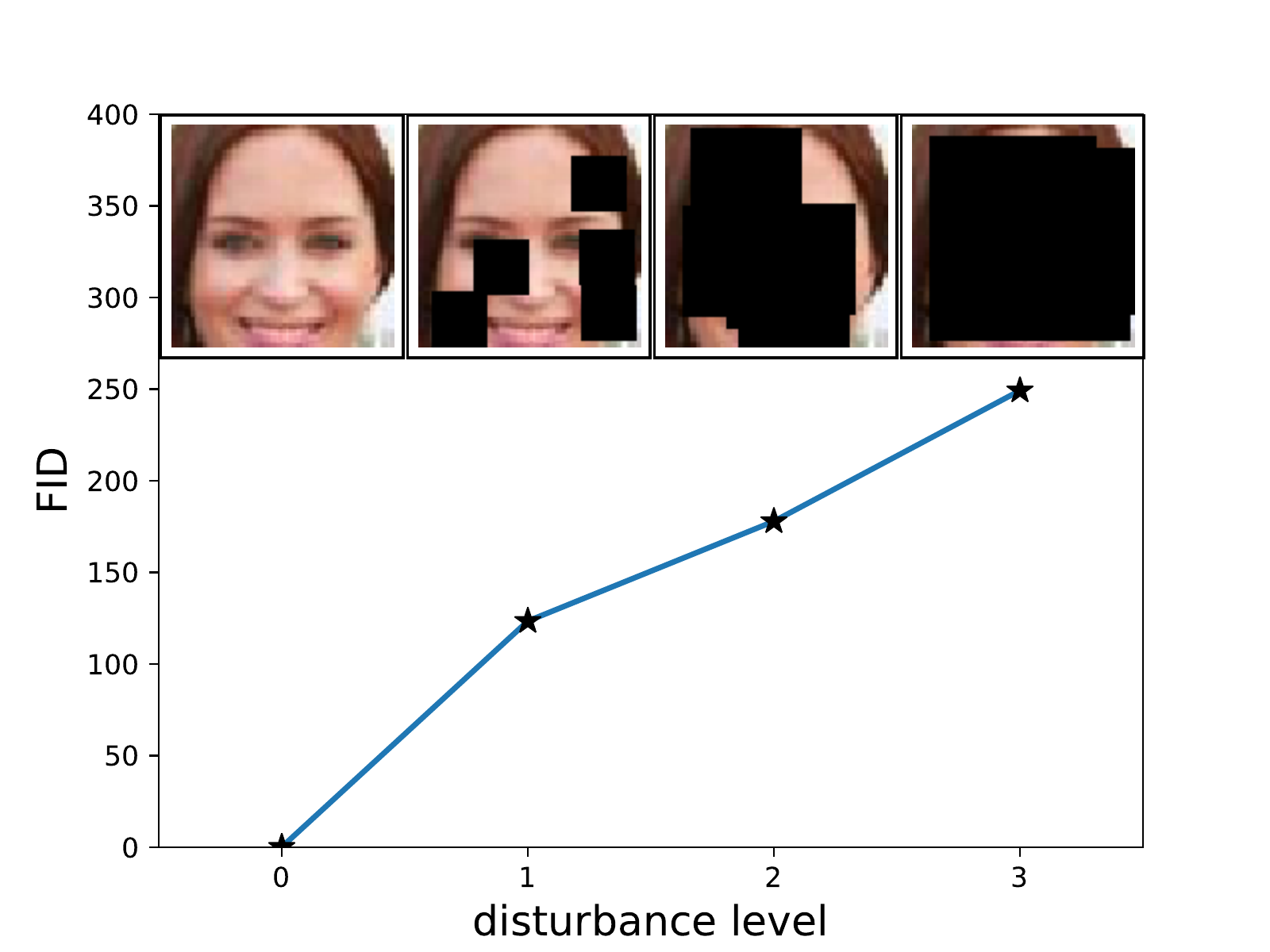}
\includegraphics[width=0.33\textwidth, height=3cm]{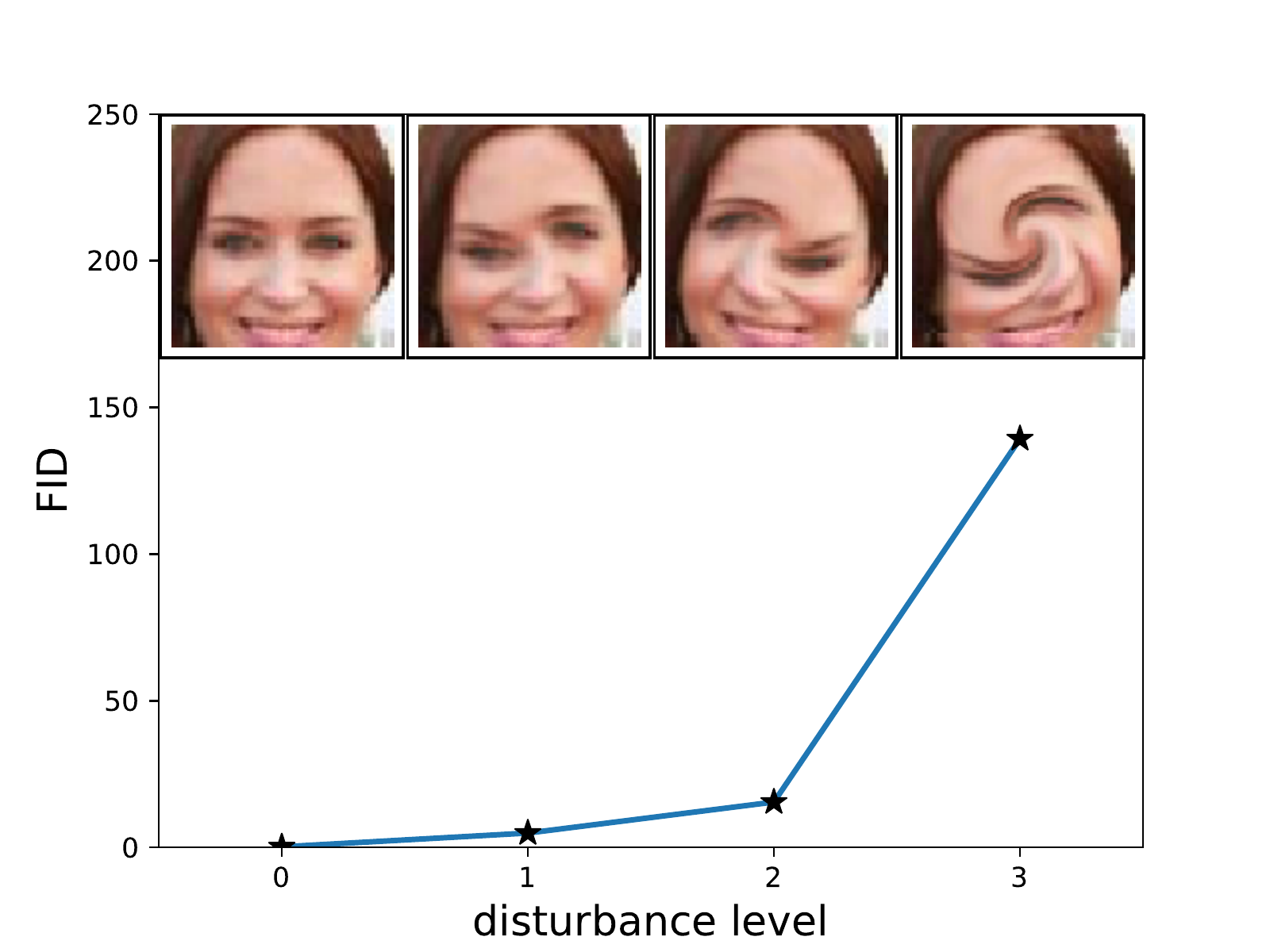}
\includegraphics[width=0.33\textwidth, height=3cm]{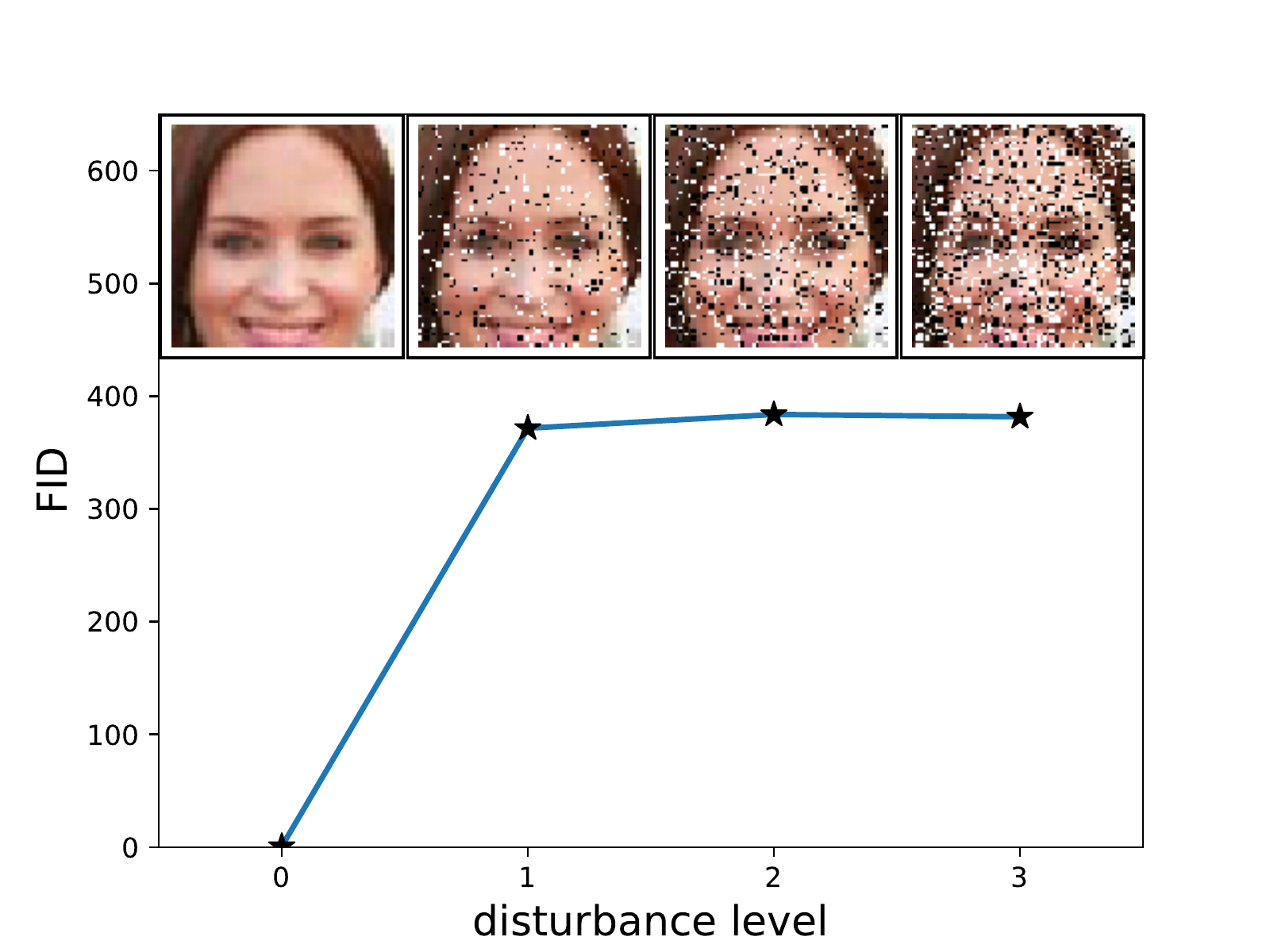}
\includegraphics[width=0.33\textwidth, height=3cm]{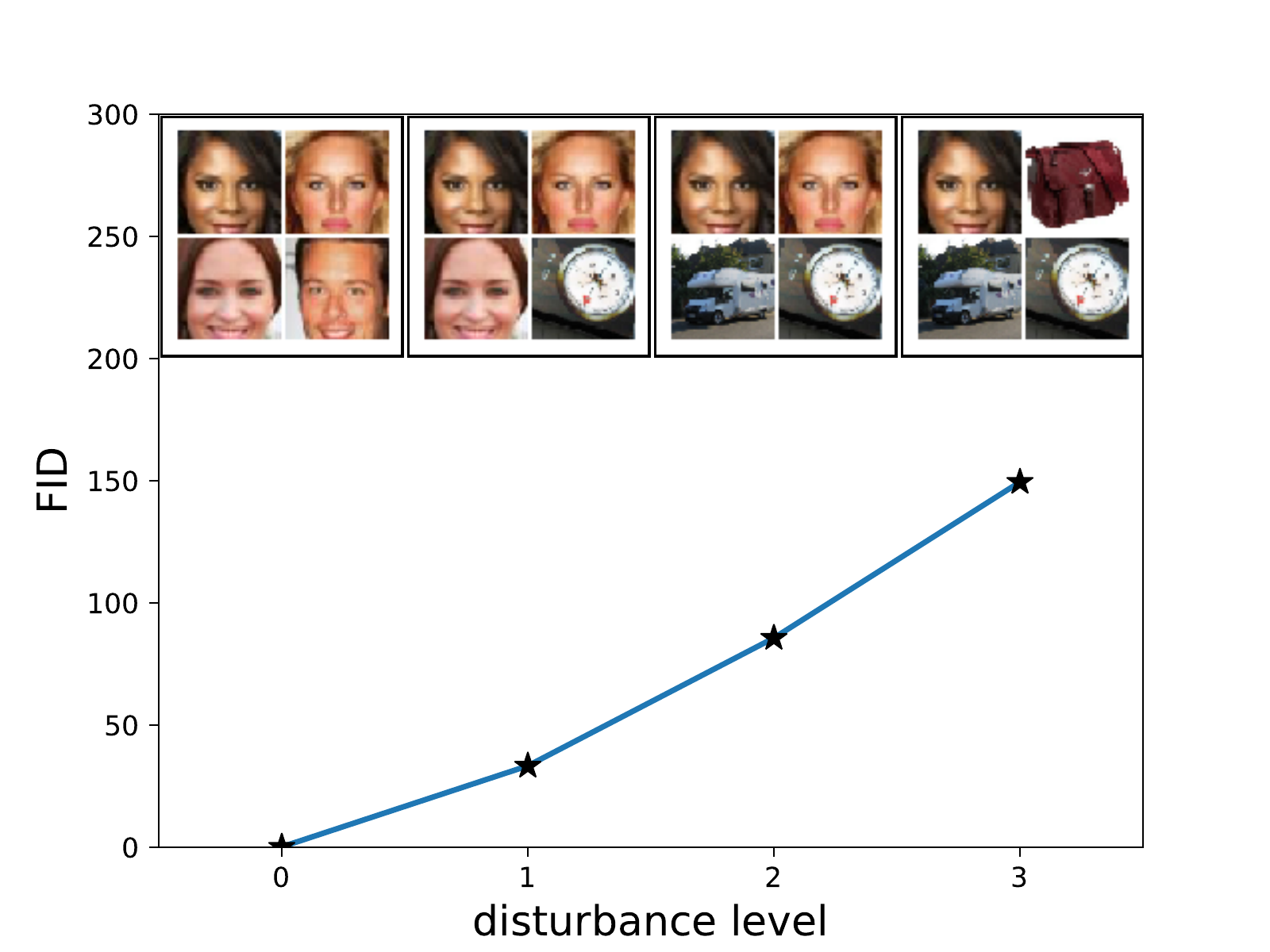}
\caption[FID evaluated for different disturbances]{FID is evaluated for {\bf
upper left:} Gaussian noise, {\bf upper middle:} Gaussian blur, {\bf upper right:} implanted black rectangles, {\bf
lower left:} swirled images, {\bf lower middle:} salt and pepper
noise, and {\bf lower right:} CelebA dataset contaminated by
ImageNet images.
The disturbance level rises from zero and increases to the highest
level. The FID captures the disturbance level very well by
monotonically increasing.
  \label{fig:FIDscore} }
\end{figure}

\paragraph{Model Selection and Evaluation.}

We compare the two time-scale update rule (TTUR) for GANs with the original GAN
training to see whether TTUR improves the convergence speed and performance of
GANs. We have selected Adam stochastic optimization to reduce the risk of mode
collapsing. The advantage of Adam has been confirmed by MNIST experiments, where
Adam indeed considerably reduced the cases for which we observed mode
collapsing. Although TTUR ensures that the discriminator converges during
learning, practicable learning rates must be found for each experiment.
We face a trade-off since the learning rates should be small enough (e.g.\ for
the generator) to ensure convergence but at the same time should be large enough
to allow fast learning. For each of the experiments, the learning rates have
been optimized to be large while still ensuring stable training which is
indicated by a decreasing FID or Jensen-Shannon-divergence (JSD). We further
fixed the time point for stopping training to the update step when the FID or
Jensen-Shannon-divergence of the best models was no longer decreasing. For some
models, we observed that the FID diverges or starts to increase at a certain
time point. An example of this behaviour is shown in Fig.~\ref{fig:dcgan}. The
performance of generative models is evaluated via the Fr\'{e}chet Inception
Distance (FID) introduced above.
For the One Billion Word experiment, the normalized JSD served as performance
measure. For computing the FID, we propagated all images from the training
dataset through the pretrained Inception-v3 model following the computation of
the Inception Score \cite{Salimans:16}, however, we use the last pooling layer
as coding layer. For this coding layer, we calculated the mean $\Bm_w$ and the
covariance matrix $\BC_w$. Thus, we approximate the first and second central
moment of the function given by the Inception coding layer under the real world
distribution. To approximate these moments for the model distribution, we
generate 50,000 images, propagate them through the Inception-v3 model, and
then compute the mean $\Bm$ and the covariance matrix $\BC$.
For computational efficiency, we evaluate the FID every 1,000 DCGAN mini-batch
updates, every 5,000 WGAN-GP outer iterations for the image experiments, and
every 100 outer iterations for the WGAN-GP language model. For the one
time-scale updates a WGAN-GP outer iteration for the image model consists of five
discriminator mini-batches and ten discriminator mini-batches for the language
model, where we follow the original implementation. For TTUR however, the
discriminator is updated only once per iteration. We repeat the training for each single time-scale
(orig) and TTUR learning rate eight times for the image datasets and ten times for the language benchmark.
Additionally to the mean FID training progress we show the minimum and maximum
FID over all runs at each evaluation time-step. For more details,
implementations and further results see Appendix
Section~\ref{sec:apx_experiments} and \ref{sec:soft}.

\paragraph{Simple Toy Data.}

We first want to demonstrate the difference between a single time-scale
update rule and TTUR on a simple toy min/max problem where a saddle point
should be found.
The objective $f(x,y) = (1+x^2)(100-y^2)$ in Fig.~\ref{fig:toy1}
(left) has a saddle point at $(x,y) = (0,0)$ and fulfills assumption A4. The norm
$\|(x,y)\|$ measures the distance of the parameter vector $(x,y)$ to the saddle point.
We update $(x,y)$ by gradient descent in $x$ and gradient ascent in
$y$ using additive Gaussian noise in order to simulate a stochastic update.
The updates should converge to the saddle point $(x,y)=(0,0)$
with objective value $f(0,0)=100$ and the norm $0$.
In Fig.~\ref{fig:toy1} (right),
the first two rows show one time-scale update rules.
The large learning rate in the first row diverges and has large
fluctuations. The smaller
learning rate in the second row converges but slower than
the TTUR in the third row which has slow $x$-updates. TTUR with slow
$y$-updates in the fourth row also converges but slower.

 \begin{figure}[H] \centering
\includegraphics[width=0.49\linewidth]{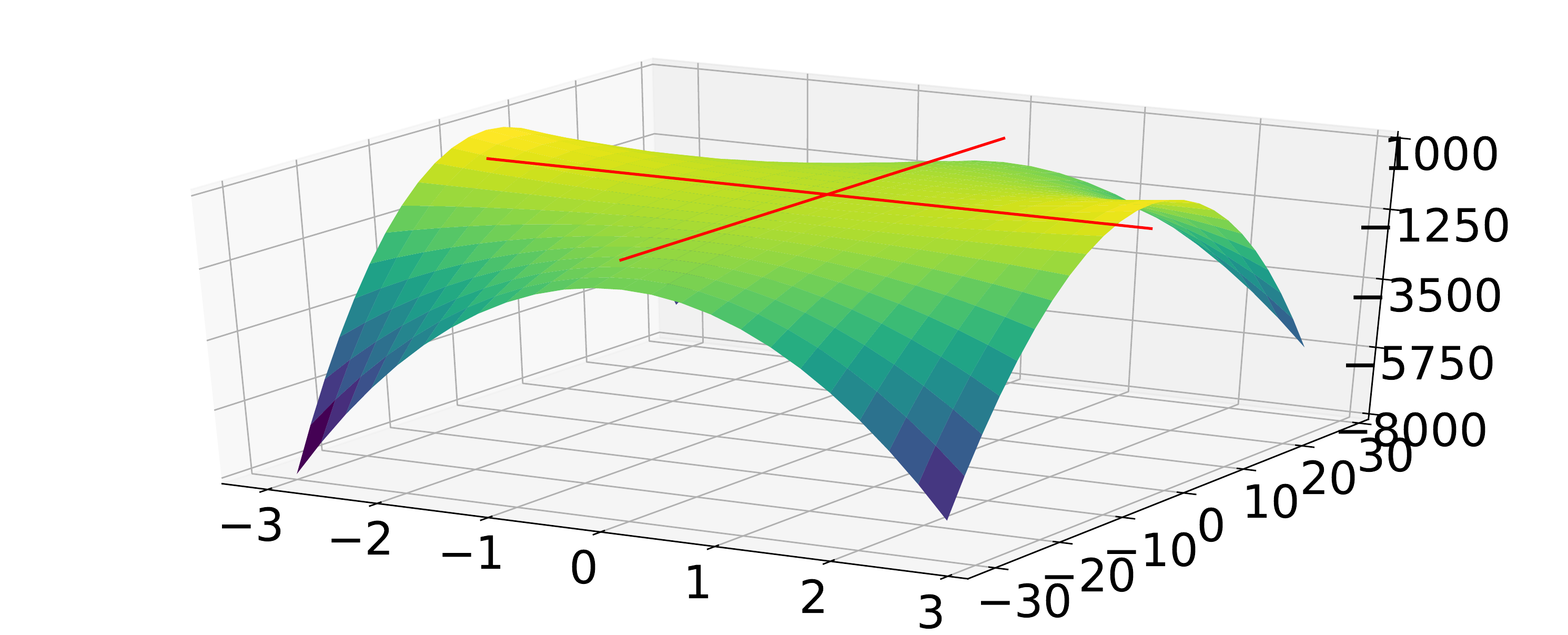}
\includegraphics[width=0.49\linewidth]{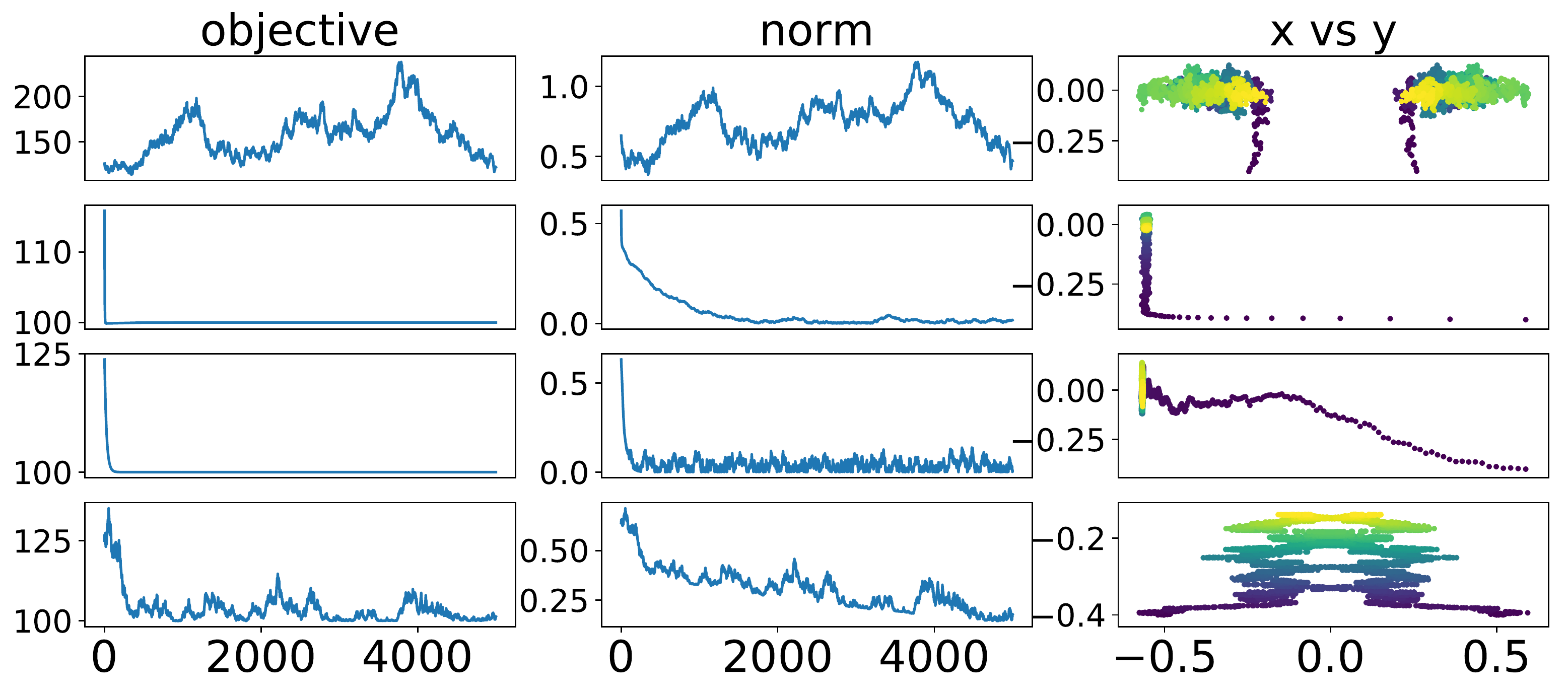}
\caption[TTUR and single time-scale update with toy data.]{{\bf Left:} Plot of
the objective with a saddle point at $(0,0)$. {\bf Right:} Training progress with equal learning rates
of $0.01$ (first row) and $0.001$ (second row)) for $x$ and $y$, TTUR with a
learning rate of $0.0001$ for $x$ vs.\ $0.01$ for $y$ (third row) and a larger
learning rate of $0.01$ for $x$ vs.\ $0.0001$ for $y$ (fourth row). The columns
show the function values (left), norms (middle), and
$(x,y)$ (right). TTUR (third row) clearly converges faster than with equal
time-scale updates and directly moves to the saddle point as shown by the norm
and in the $(x,y)$-plot.
\label{fig:toy1} }
\end{figure}

\paragraph{DCGAN on Image Data.}

We test TTUR for the deep convolutional GAN (DCGAN) \cite{Radford:15} at the
CelebA, CIFAR-10, SVHN and LSUN Bedrooms dataset. Fig.~\ref{fig:dcgan} shows
the FID during learning with the original learning method (orig) and with TTUR.
The original training method is faster at the beginning, but TTUR
eventually achieves better performance. DCGAN trained TTUR
reaches constantly a lower FID than the original method and for CelebA and LSUN
Bedrooms all one time-scale runs diverge. For DCGAN the learning rate of the
generator is larger then that of the discriminator, which, however, does not contradict
the TTUR theory (see the Appendix Section~\ref{sec:lr}). In
Table~\ref{tab:all} we report the best FID with TTUR and
one time-scale training for optimized number of updates and learning
rates. TTUR constantly outperforms standard training and is more stable.

\begin{figure} 
\centering
\includegraphics[width=0.49\linewidth]{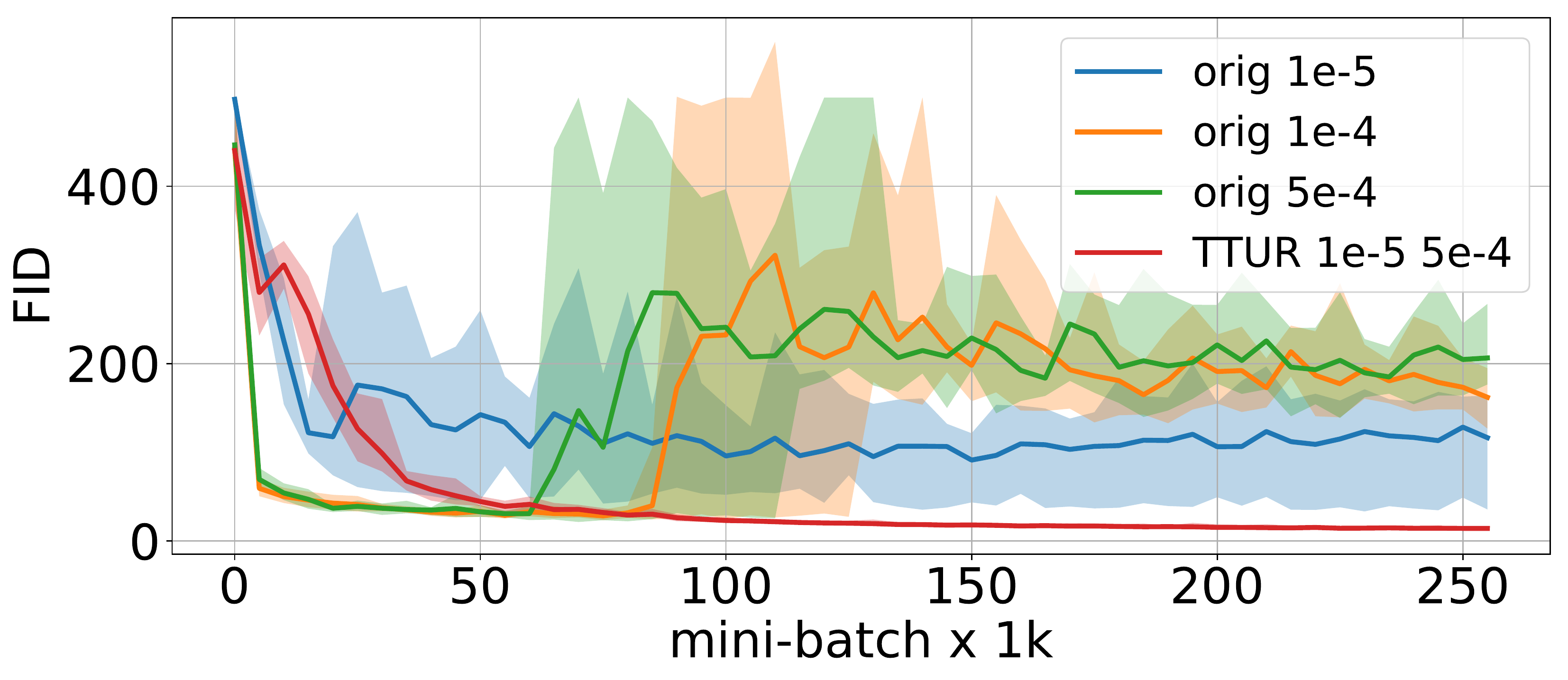}
\includegraphics[width=0.49\linewidth]{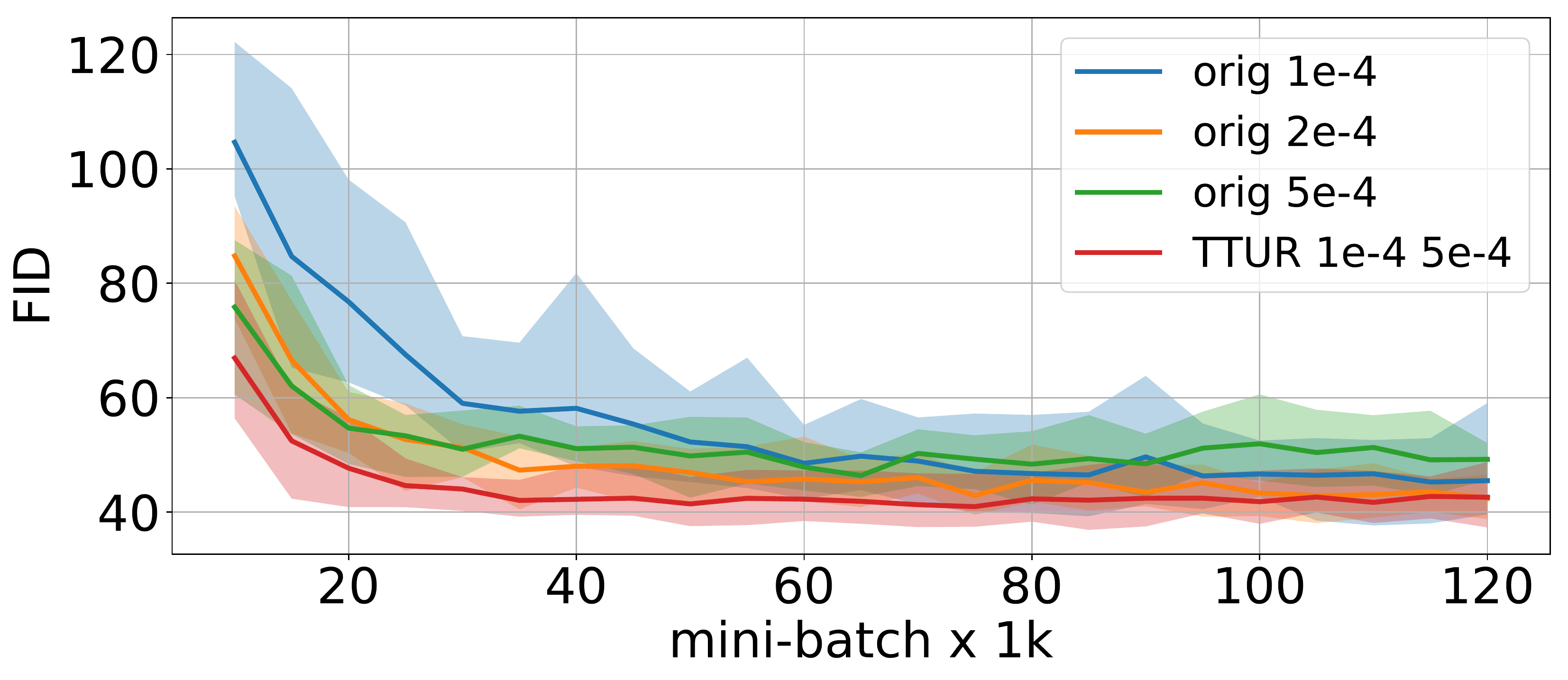}
\includegraphics[width=0.49\linewidth]{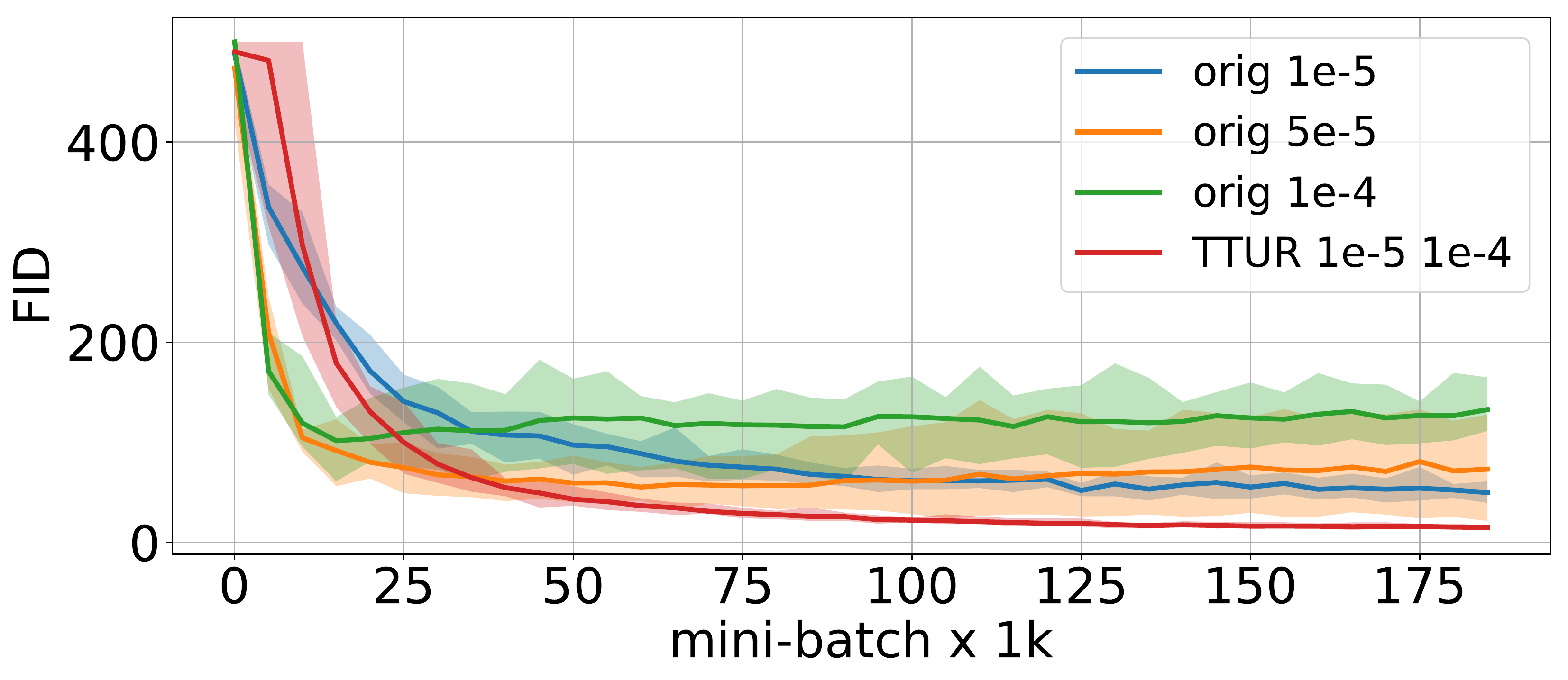}
\includegraphics[width=0.49\linewidth]{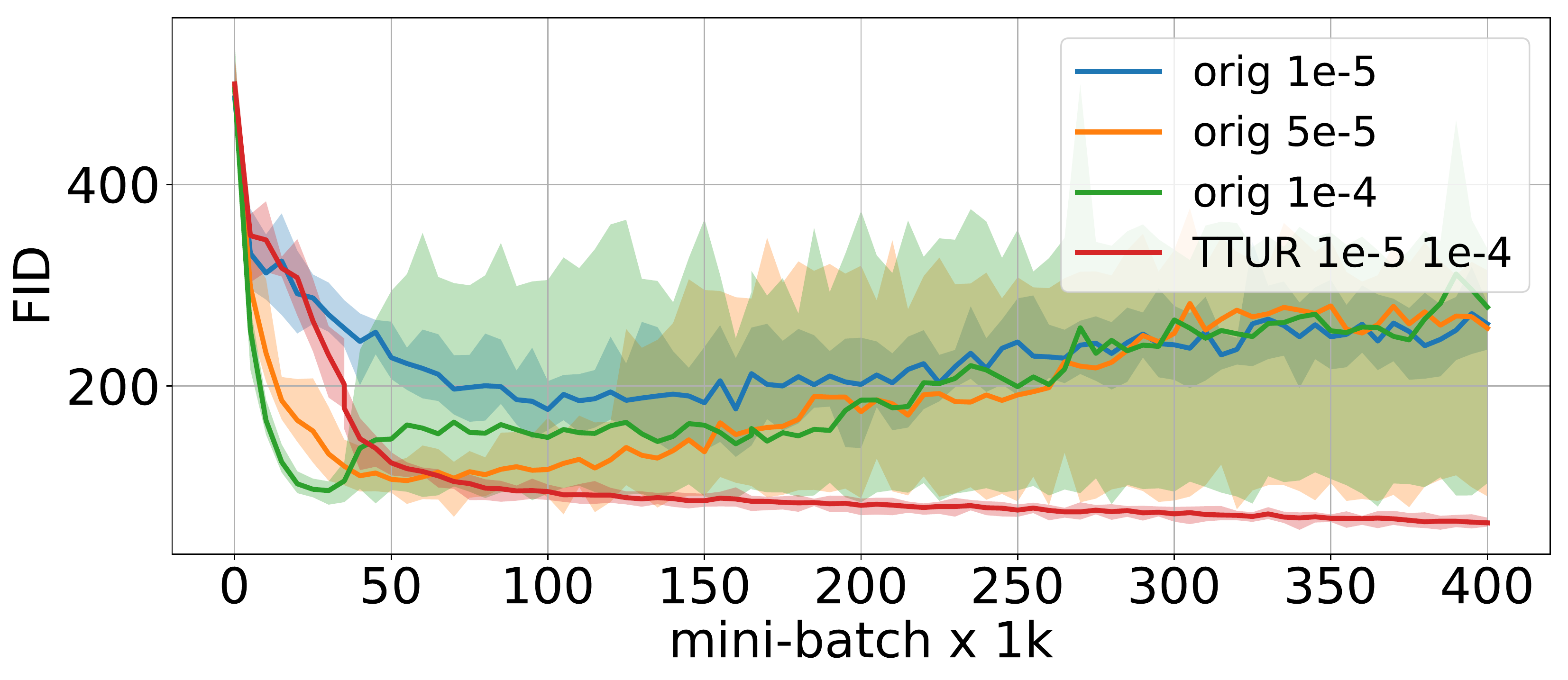} \caption[FID for
DCGAN on CelebA, CIFAR-10, SVHN, and LSUN Bedrooms.]{Mean FID (solid line)
surrounded by a shaded area bounded by the maximum and the minimum over 8 runs for DCGAN on
CelebA, CIFAR-10, SVHN, and LSUN Bedrooms.
TTUR learning rates are given for the discriminator $b$ and generator
$a$ as: ``TTUR $b$ $a$''. {\bf Top Left:} CelebA. {\bf Top
Right:} CIFAR-10, starting at mini-batch update 10k for better visualisation.
{\bf Bottom Left:} SVHN. {\bf Bottom Right:} LSUN Bedrooms.
Training with TTUR (red) is more stable, has much lower variance, and leads to a better FID.
  \label{fig:dcgan} }
\end{figure}

\paragraph{WGAN-GP on Image Data.}

We used the WGAN-GP image model \cite{Gulrajani:17} to test TTUR with the
CIFAR-10 and LSUN Bedrooms datasets. In contrast to the original code where the
discriminator is trained five times for each generator update, TTUR updates the discriminator
only once, therefore we align the training progress with wall-clock time. The
learning rate for the original training was optimized to be large but leads to
stable learning. TTUR can use a higher learning rate for the discriminator since
TTUR stabilizes learning. Fig.~\ref{fig:wgan_gp} shows the FID
during learning with the original learning method and
with TTUR. Table~\ref{tab:all} shows the
best FID with TTUR and one time-scale training for optimized number of iterations and
learning rates. Again TTUR reaches lower FIDs than one time-scale training.

\begin{figure}[H]
\centering
\includegraphics[width=0.49\linewidth]{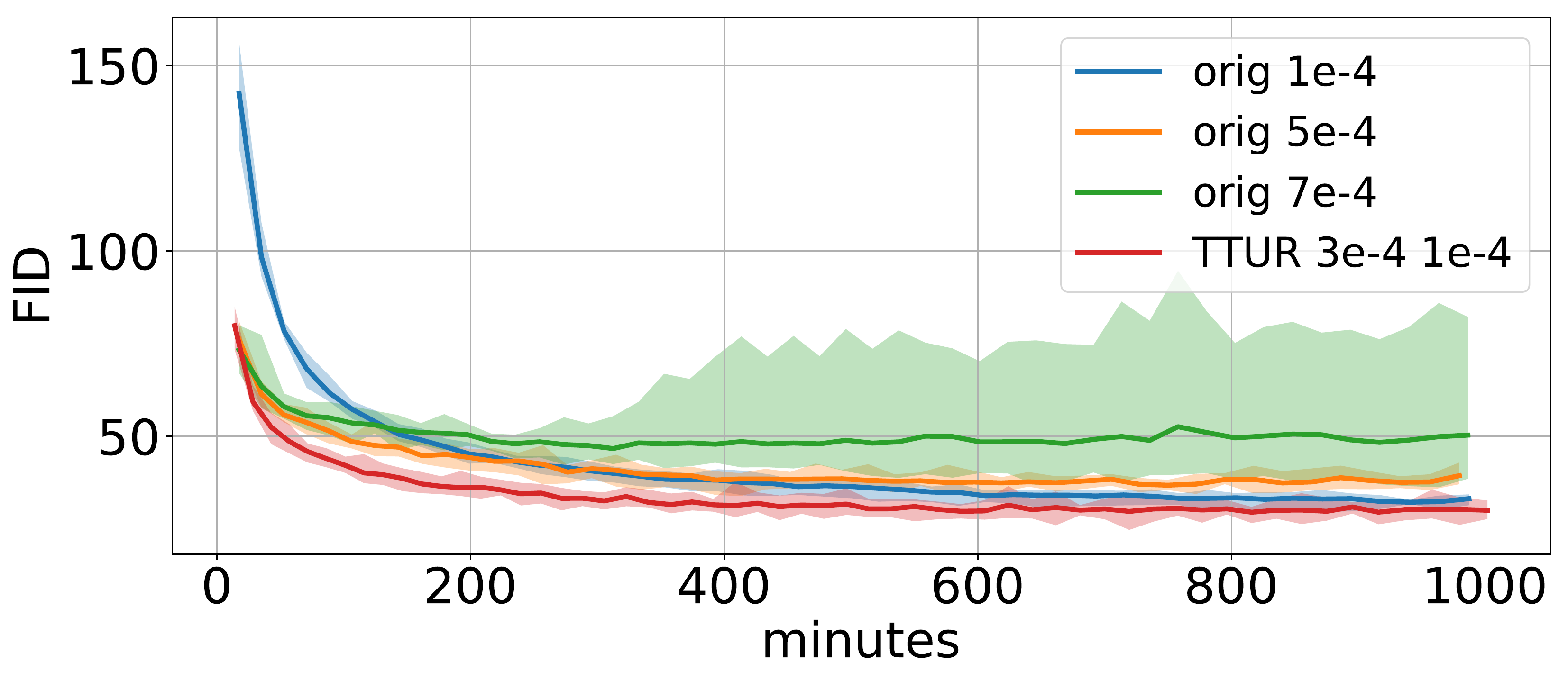}
\includegraphics[width=0.49\linewidth]{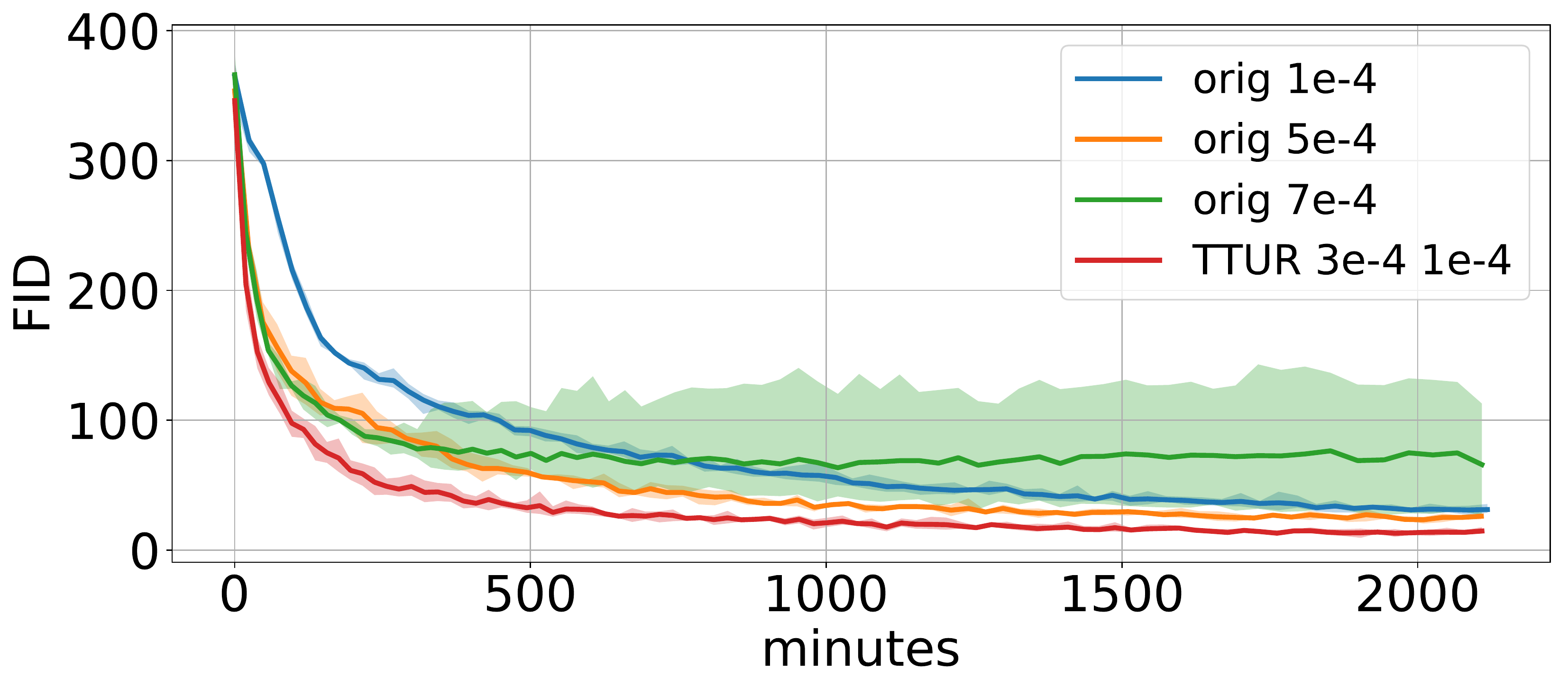} \caption[FID for
WGAN-GP trained on CIFAR-10 and LSUN Bedrooms.]{Mean FID (solid line)
surrounded by a shaded area bounded by the maximum and the minimum over 8 runs for WGAN-GP on
CelebA, CIFAR-10, SVHN, and LSUN Bedrooms.
TTUR learning rates are given for the discriminator $b$ and generator
$a$ as: ``TTUR $b$ $a$''.
{\bf Left:} CIFAR-10, starting at minute 20. {\bf Right:} LSUN
Bedrooms. Training with TTUR (red) has much lower variance and leads to a better FID.
  \label{fig:wgan_gp} }
\end{figure}

\paragraph{WGAN-GP on Language Data.}

Finally the One Billion Word Benchmark \cite{Chelba:13} serves to evaluate TTUR
on WGAN-GP.
The character-level generative language model is a 1D convolutional neural
network (CNN) which maps a latent vector to a sequence of one-hot character
vectors of dimension 32 given by the maximum of a softmax output.
The discriminator is also a 1D CNN applied to sequences of one-hot vectors of 32
characters.
Since the FID criterium only works for images, we measured the performance by
the Jensen-Shannon-divergence (JSD) between the model and the real world
distribution as has been done previously \cite{Gulrajani:17}.
In contrast to the original code where the critic is trained ten times for each
generator update, TTUR updates the discriminator only once, therefore we align
the training progress with wall-clock time. The learning rate for the original
training was optimized to be large but leads to stable learning. TTUR can use a
higher learning rate for the discriminator since TTUR stabilizes learning.
We report for the 4 and 6-gram word evaluation the normalized mean JSD for ten
runs for original training and TTUR training in Fig.~\ref{fig:lang}. In
Table~\ref{tab:all} we report the best JSD at an optimal time-step where TTUR
outperforms the standard training for both  measures. The improvement of TTUR on
the 6-gram statistics over original training shows that TTUR enables to
learn to generate more subtle pseudo-words which better resembles real words.

\begin{figure} 
\centering
\includegraphics[width=0.49\textwidth]{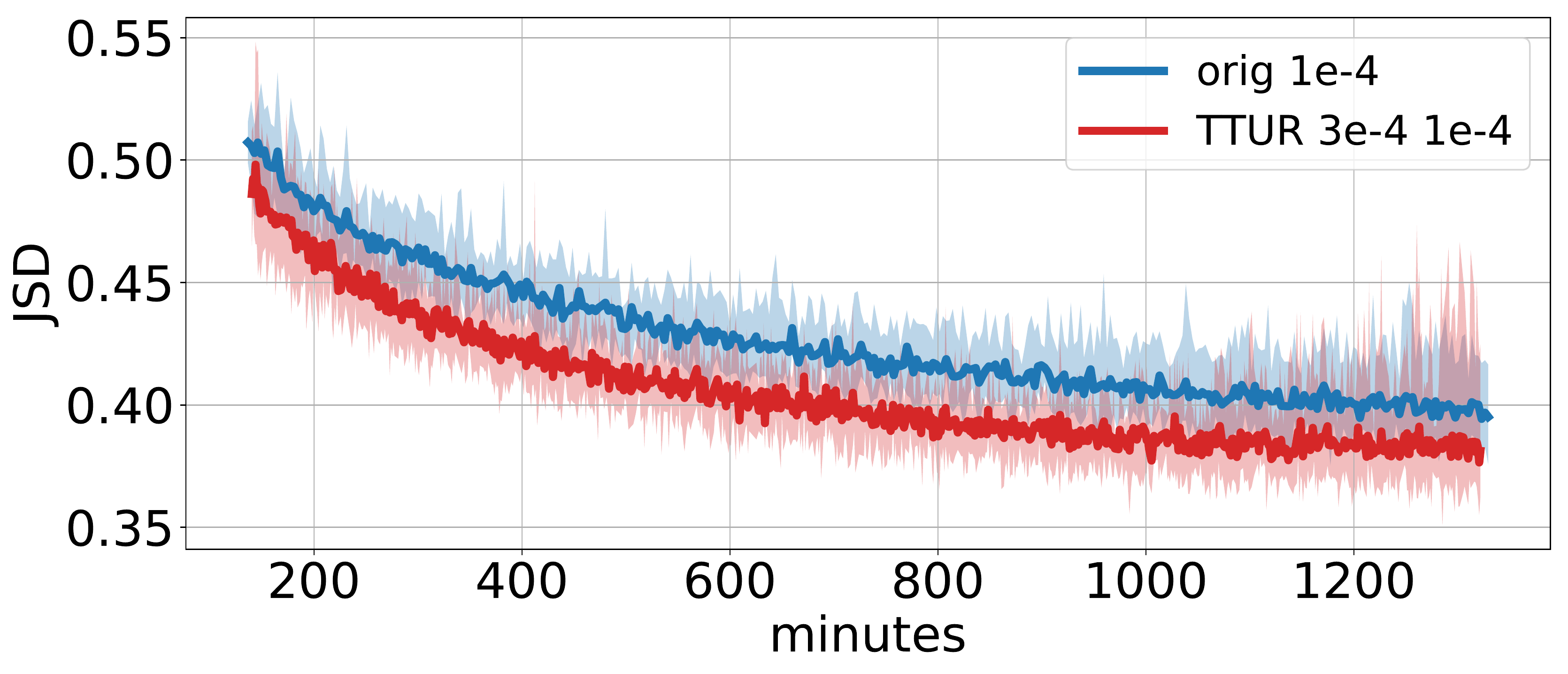}
\includegraphics[width=0.49\textwidth]{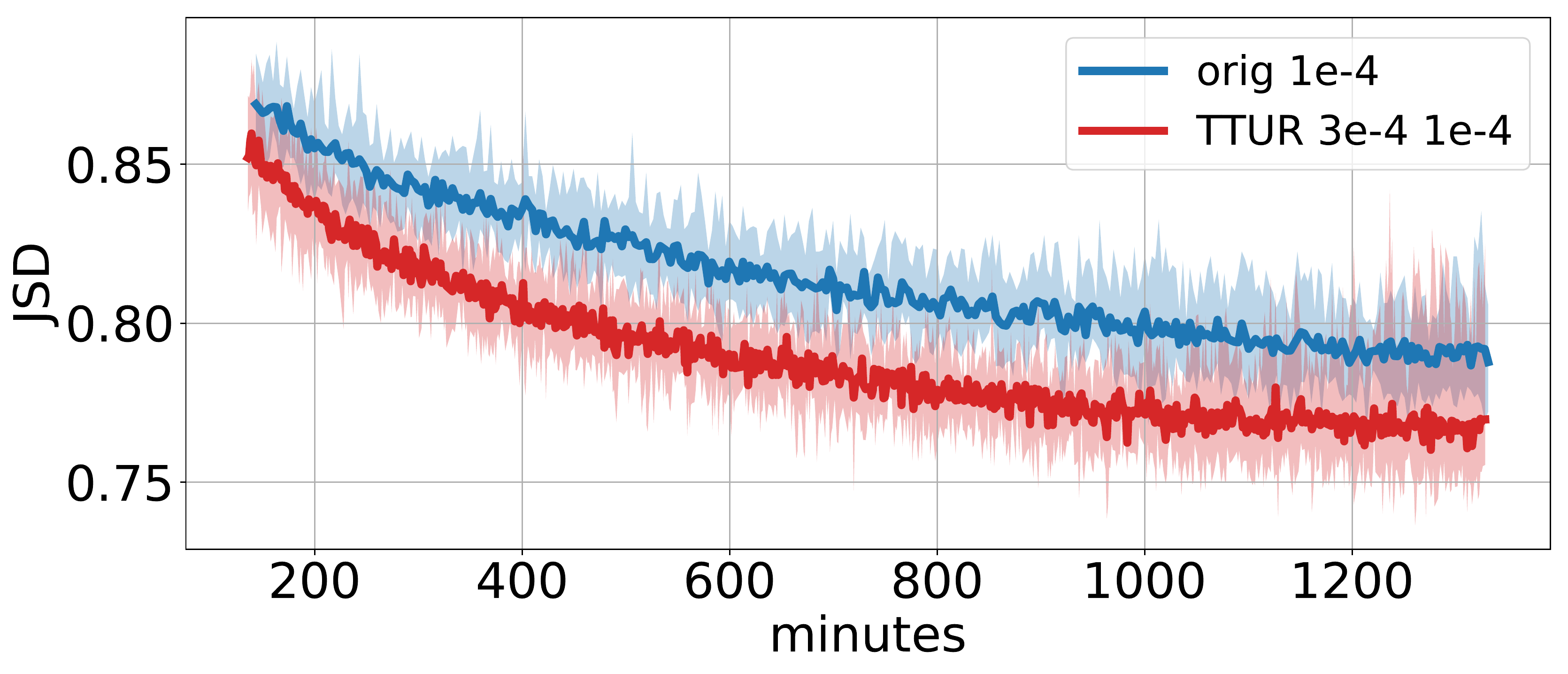}
\caption[Performance of WGAN-GP on One Billion Word.]{Performance of
WGAN-GP models trained with the original (orig) and our TTUR
method on the One Billion Word benchmark.
The performance is measured by
the normalized Jensen-Shannon-divergence based on 4-gram ({\bf left}) and
 6-gram ({\bf right}) statistics averaged (solid line) and surrounded
 by a shaded area bounded by the maximum and the minimum
over 10 runs, aligned to wall-clock
time and starting at minute 150. TTUR learning (red) clearly
outperforms the
original one time-scale learning.}
\label{fig:lang}
\end{figure}

\begin{table}[H]
\begin{center}
\caption[Results DCGAN and WGAN-GP]{The performance of DCGAN and WGAN-GP trained
with the original one time-scale update rule and with TTUR on CelebA, CIFAR-10, SVHN, LSUN Bedrooms and
the One Billion Word Benchmark. During training we compare the performance with respect
to the FID and JSD for optimized number of updates.
TTUR exhibits consistently a better FID and a better JSD.}
\label{tab:all}
\begin{tabular}{lllccllcc}
 \toprule
 \multicolumn{2}{l}{DCGAN Image}\\
 dataset & method & b, a & updates & FID & method & b = a & updates & FID \\
 \hline
 \noalign{\vskip 3pt}
 CelebA & TTUR & 1e-5, 5e-4 & 225k & {\bf 12.5} & orig & 5e-4 & 70k & 21.4 \\
 CIFAR-10 & TTUR &   1e-4, 5e-4 & 75k & {\bf 36.9} & orig & 1e-4 & 100k & 37.7 \\
 SVHN & TTUR &   1e-5, 1e-4 & 165k & {\bf 12.5} & orig & 5e-5 & 185k & 21.4 \\
 LSUN & TTUR &   1e-5, 1e-4 & 340k & {\bf 57.5} & orig & 5e-5 & 70k & 70.4 \\
 \hline
 \noalign{\vskip 3pt}
 \multicolumn{2}{l}{WGAN-GP Image} \\
 dataset & method & b, a & time(m) & FID & method & b = a & time(m) & FID \\
 \hline
 \noalign{\vskip 3pt}
 CIFAR-10 & TTUR &   3e-4, 1e-4 & 700 & {\bf 24.8} & orig & 1e-4 & 800 & 29.3 \\
 LSUN & TTUR & 3e-4, 1e-4 & 1900 & {\bf 9.5} & orig & 1e-4 & 2010 & 20.5 \\
 \hline
 \noalign{\vskip 3pt}
 \multicolumn{2}{l}{WGAN-GP Language}\\
 n-gram & method & b, a & time(m) & JSD & method & b = a & time(m) & JSD \\
 \hline
 \noalign{\vskip 3pt}
 4-gram & TTUR &   3e-4, 1e-4 & 1150 & {\bf 0.35} & orig & 1e-4 & 1040 & 0.38 \\
 6-gram & TTUR &   3e-4, 1e-4 & 1120 & {\bf 0.74} & orig & 1e-4 & 1070 & 0.77 \\
 \hline

 \end{tabular}
 \end{center}
\end{table}

\section*{Conclusion}
\label{sec:conclusion}
For learning GANs, we have introduced the two time-scale update rule (TTUR),
which we have proved to converge to a stationary local Nash equilibrium.
Then we described Adam stochastic optimization as a heavy ball with friction
(HBF) dynamics, which shows that Adam converges and that Adam tends to find flat
minima while avoiding small local minima.
A second order differential equation describes the learning dynamics of Adam as
an HBF system. Via this differential equation, the convergence of GANs trained
with TTUR to a stationary local Nash equilibrium can be extended to Adam.
Finally, to evaluate GANs, we introduced the `Fr\'{e}chet Inception Distance''
(FID) which captures the similarity of generated images to real ones better than
the Inception Score. In experiments we have compared GANs trained with TTUR to
conventional GAN training with a one time-scale update rule
on CelebA, CIFAR-10, SVHN, LSUN Bedrooms, and the One
Billion Word Benchmark. TTUR outperforms conventional GAN training consistently
in all experiments.

\section*{Acknowledgment}
This work was supported by NVIDIA Corporation, Bayer AG with Research Agreement
09/2017, Zalando SE with Research Agreement 01/2016, Audi.JKU Deep Learning
Center, Audi Electronic Venture GmbH, IWT research grant IWT150865 (Exaptation),
H2020 project grant 671555 (ExCAPE) and FWF grant P 28660-N31.

\section*{References}
The references are provided after Section~\ref{sec:references}.

\section*{Appendix}
\renewcommand{\thesection}{A\arabic{section}}
\renewcommand{\thefigure}{A\arabic{figure}}
\renewcommand{\thetable}{A\arabic{table}}

\sectionfont{\large}
\subsectionfont{\normalsize}
\subsubsectionfont{\normalsize}
\paragraphfont{\normalsize}

\tableofcontents

\section{Fr\'{e}chet Inception Distance (FID)}
\label{sec:fid}

We improve the Inception score for comparing the results of GANs
\cite{Salimans:16}.
The Inception score has the disadvantage that it does not use
the statistics of real world samples and compare it to
the statistics of synthetic samples.
Let $p(.)$ be the distribution of model samples and $p_w(.)$ the
distribution of the samples from real world.
The equality $p(.)=p_w(.)$ holds except for a non-measurable set
if and only if $\int p(.) f(x) dx=\int p_w(.) f(x) dx$ for
a basis $f(.)$ spanning the function space in which $p(.)$ and $p_w(.)$
live.
These equalities of expectations are used to describe distributions
by moments or cumulants, where $f(x)$ are polynomials of the data $x$.
We replacing $x$ by the coding layer of
an Inception model in order to obtain vision-relevant features and
consider polynomials of the coding unit functions.
For practical reasons we only consider the first two polynomials, that
is, the first two moments: mean and covariance.
The Gaussian is the maximum entropy distribution for given
mean and covariance, therefore we assume the coding units to follow a
multidimensional Gaussian.
The difference of two Gaussians is measured by the Fr\'{e}chet
distance \cite{Frechet:57}
also known as Wasserstein-2 distance \cite{Wasserstein:69}.
The Fr\'{e}chet distance
$d(.,.)$ between the Gaussian with mean and covariance $(\Bm,\BC)$ obtained
from $p(.)$ and the Gaussian $(\Bm_w,\BC_w)$ obtained
from $p_w(.)$ is called the ``Fr\'{e}chet Inception Distance'' (FID), which is
given by \cite{Dowson:82}:
\begin{align}
d^2((\Bm,\BC),(\Bm_w,\BC_w))=\|\Bm-\Bm_w\|_2^2+  \TR \bigl(\BC+\BC_w-2\bigl(
\BC\BC_w\bigr)^{1/2}\bigr) \ .
\end{align}
Next we show that the FID is consistent with
increasing disturbances and human judgment on the CelebA dataset.
We computed the $(\Bm_w,\BC_w)$ on all CelebA images, while
for computing $(\Bm,\BC)$ we used 50,000 randomly selected samples.
We considered following disturbances of the image $\BX$:
\begin{enumerate}
\item {\bf Gaussian noise}: We constructed a matrix $\BN$ with Gaussian
noise scaled to $[0,255]$. The noisy image is computed as
$(1-\alpha) \BX + \alpha \BN$ for $\alpha \in \{0,0.25,0.5,0.75\}$.
The larger $\alpha$ is, the larger is the noise added to the image,
the larger is the disturbance of the image.

\item {\bf Gaussian blur}: The image is convolved with a Gaussian
kernel with standard deviation $\alpha \in \{0,1,2,4\}$.
The larger $\alpha$ is, the larger is the disturbance of the image, that is,
the more the image is smoothed.

\item {\bf Black rectangles}: To an image five black rectangles are
are added at randomly chosen locations.
The rectangles cover parts of the image. The size of the rectangles
is $\alpha \text{imagesize}$ with  $\alpha \in \{0,0.25,0.5,0.75\}$.
The larger $\alpha$ is, the larger is the disturbance of the image,
that is, the more of the image is covered by black rectangles.

\item {\bf Swirl}: Parts of the image are transformed as a
spiral, that is, as a swirl (whirlpool effect).
Consider the coordinate $(x, y)$ in the noisy (swirled) image for
which we want to find the color. Towards this end we need the
reverse mapping for the swirl transformation which gives the location
which is mapped to $(x, y)$.
We first compute polar coordinates
relative to a center $(x_0, y_0)$ given by the angle
$\theta = \arctan((y - y_0)/(x - x_0))$ and the radius
$r = \sqrt{(x - x_0)^2 + (y - y_0)^2}$.
We transform them according to
$\theta' = \theta + \alpha e^{-5 r / (\ln2 \rho) }$.
Here $\alpha$ is a parameter for the amount of swirl and
$\rho$ indicates the swirl extent in pixels.
The original coordinates, where the color for $(x, y)$ can be found,
are
$x_{\mathrm{org}}=x_0 + r \cos(\theta')$ and
$y_{\mathrm{org}}=y_0 + r \sin(\theta')$.
We set $(x_0, y_0)$ to
the center of the image and $\rho=25$.
The disturbance level is given by the amount of swirl
$\alpha \in \{0,1,2,4\}$.
The larger $\alpha$ is, the larger is the disturbance of the image via the
amount of swirl.

\item {\bf Salt and pepper noise}:
Some pixels of the image are set to black or white, where black is
chosen with 50\% probability (same for white). Pixels are randomly
chosen for being flipped to white or black, where the ratio of pixel
flipped to white or black is given by the noise level
$\alpha \in \{0,0.1,0.2,0.3\}$.
The larger $\alpha$ is, the larger is the noise added to the image
via flipping pixels to white or black, the larger is the disturbance level.

\item {\bf ImageNet contamination}: From each of the 1,000 ImageNet classes,
5 images are randomly chosen, which gives 5,000 ImageNet images.
The images are ensured to be RGB and to have a minimal size of 256x256.
A percentage of $\alpha \in \{0,0.25,0.5,0.75\}$ of the
CelebA images has been replaced by ImageNet images.
$\alpha =0$ means all images are from CelebA, $\alpha=0.25$ means that
75\% of the images are from CelebA and 25\% from ImageNet etc.
The larger $\alpha$ is, the larger is the disturbance of the CelebA dataset
by contaminating it by ImageNet images.
The larger the disturbance level is, the more the dataset
deviates from the reference real world dataset.
\end{enumerate}

We compare the Inception Score \cite{Salimans:16} with the FID.
The Inception Score with $m$ samples and $K$ classes is
\begin{align}
&\exp \big( \frac{1}{m} \sum_{i=1}^{m} \sum_{k=1}^{K} p(y_k\mid \BX_i)
  \log \frac{p(y_k\mid \BX_i)}{p(y_k)} \big) \ .
\end{align}
The FID is a distance, while the Inception Score is a score.
To compare FID and Inception Score,
we transform the Inception Score to a distance,
which we call ``Inception Distance'' (IND).
This transformation to a distance
is possible since the Inception Score has a
maximal value.
For zero probability $p(y_k\mid \BX_i)=0$,
we set the value $p(y_k\mid \BX_i)
  \log \frac{p(y_k\mid \BX_i)}{p(y_k)}=0$.
We can bound the $\log$-term by
\begin{align}
&\log \frac{p(y_k\mid \BX_i)}{p(y_k)} \ \leq \  \log \frac{1}{1/m} \
  = \ \log m \ .
\end{align}
Using this bound, we obtain an upper bound on the Inception Score:
\begin{align}
&\exp \big( \frac{1}{m} \sum_{i=1}^{m} \sum_{k=1}^{K} p(y_k\mid \BX_i)
  \log \frac{p(y_k\mid \BX_i)}{p(y_k)} \big) \\
&\leq \ \exp \big( \log m \frac{1}{m} \sum_{i=1}^{m} \sum_{k=1}^{K} p(y_k\mid \BX_i) \big) \\
&= \ \exp \big( \log m \frac{1}{m} \sum_{i=1}^{m} 1 \big) \ = \ m \ .
\end{align}
The upper bound is tight and achieved if $m \leq K$ and every sample is from a
different class and the sample is classified correctly with
probability 1.
The IND is computed ``IND = $m$ - Inception Score'', therefore the IND is
zero for a perfect subset of the ImageNet with $m<K$ samples,
where each sample stems from a different class.
Therefore both distances should increase with increasing disturbance level.
In Figure~\ref{fig:fidind} we present the evaluation
for each kind of disturbance. The larger the disturbance level is, the larger
the FID and IND  should be. In Figure~\ref{fig:fid1},
\ref{fig:fid2}, \ref{fig:fid3}, and \ref{fig:fid3} we show examples of images
generated with DCGAN trained on CelebA with FIDs 500, 300, 133, 100, 45, 13, and FID 3 achieved with
WGAN-GP on CelebA.

\begin{figure}[H]
\includegraphics[width=0.49\textwidth, height=3.5cm]{figures/gnoise_FID}
\includegraphics[width=0.49\textwidth, height=3.5cm]{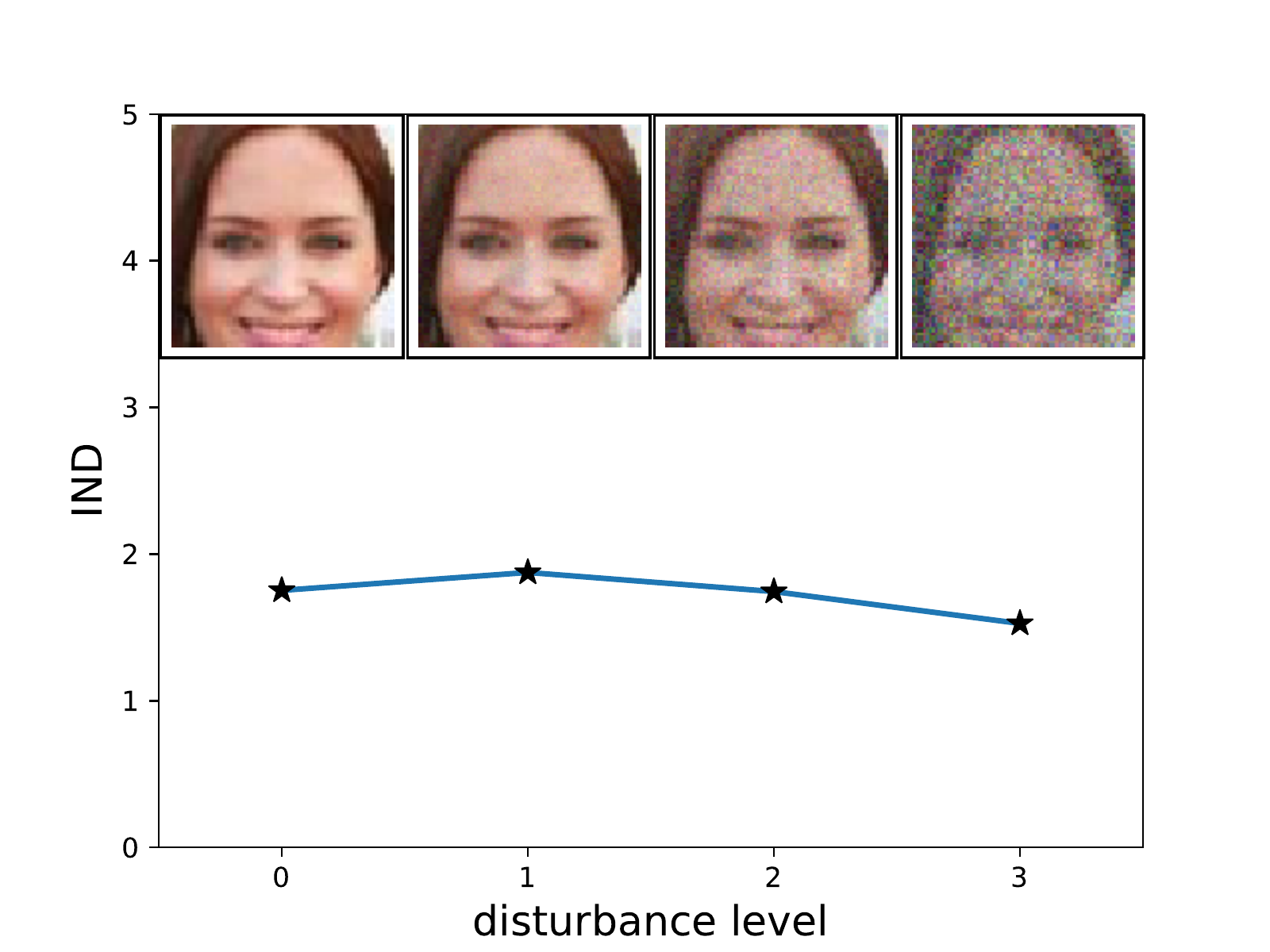}
\includegraphics[width=0.49\textwidth, height=3.5cm]{figures/blur_FID}
\includegraphics[width=0.49\textwidth, height=3.5cm]{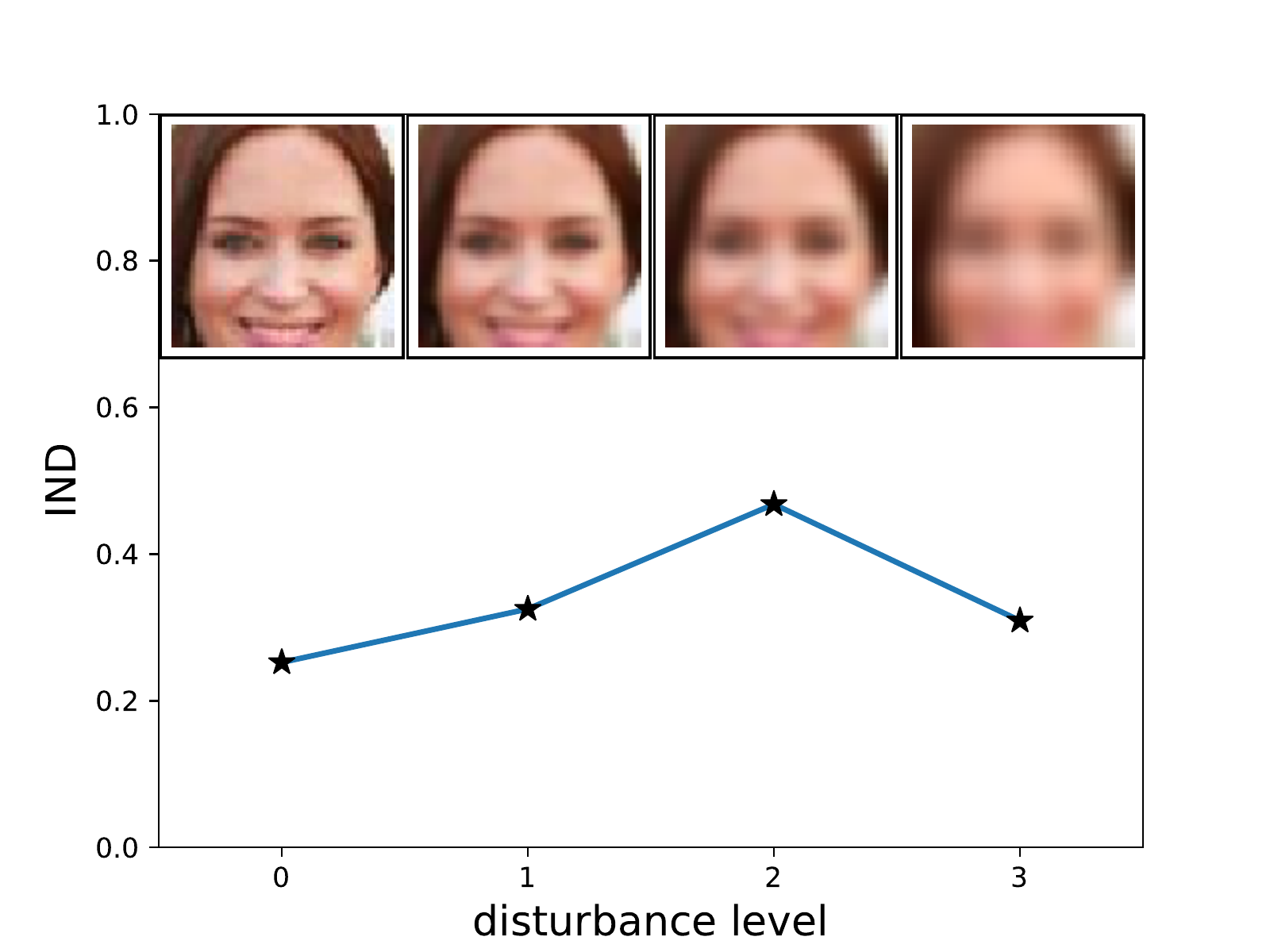}
\includegraphics[width=0.49\textwidth, height=3.5cm]{figures/rect_FID}
\includegraphics[width=0.49\textwidth, height=3.5cm]{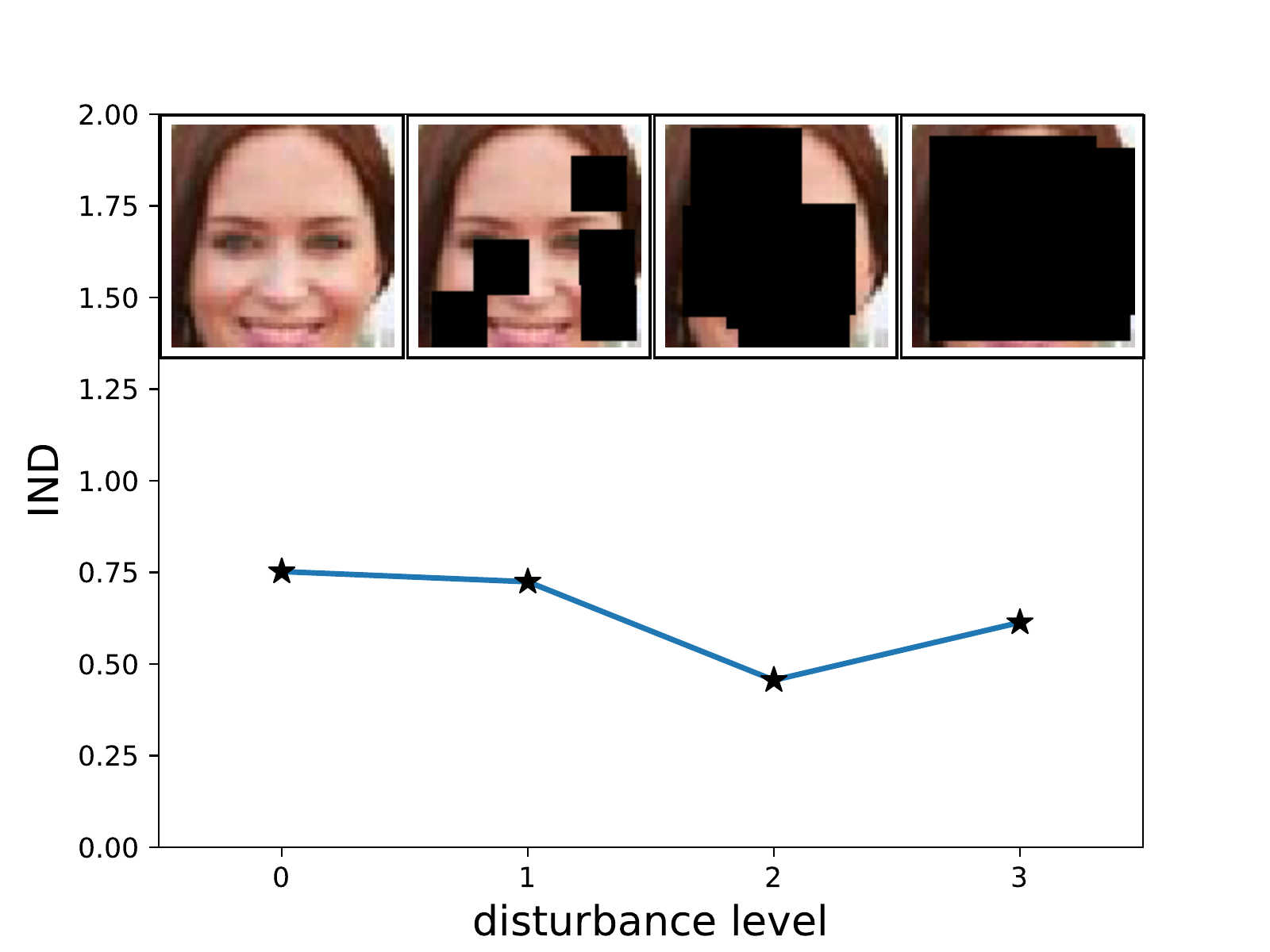}
\includegraphics[width=0.49\textwidth, height=3.5cm]{figures/swirl_FID}
\includegraphics[width=0.49\textwidth, height=3.5cm]{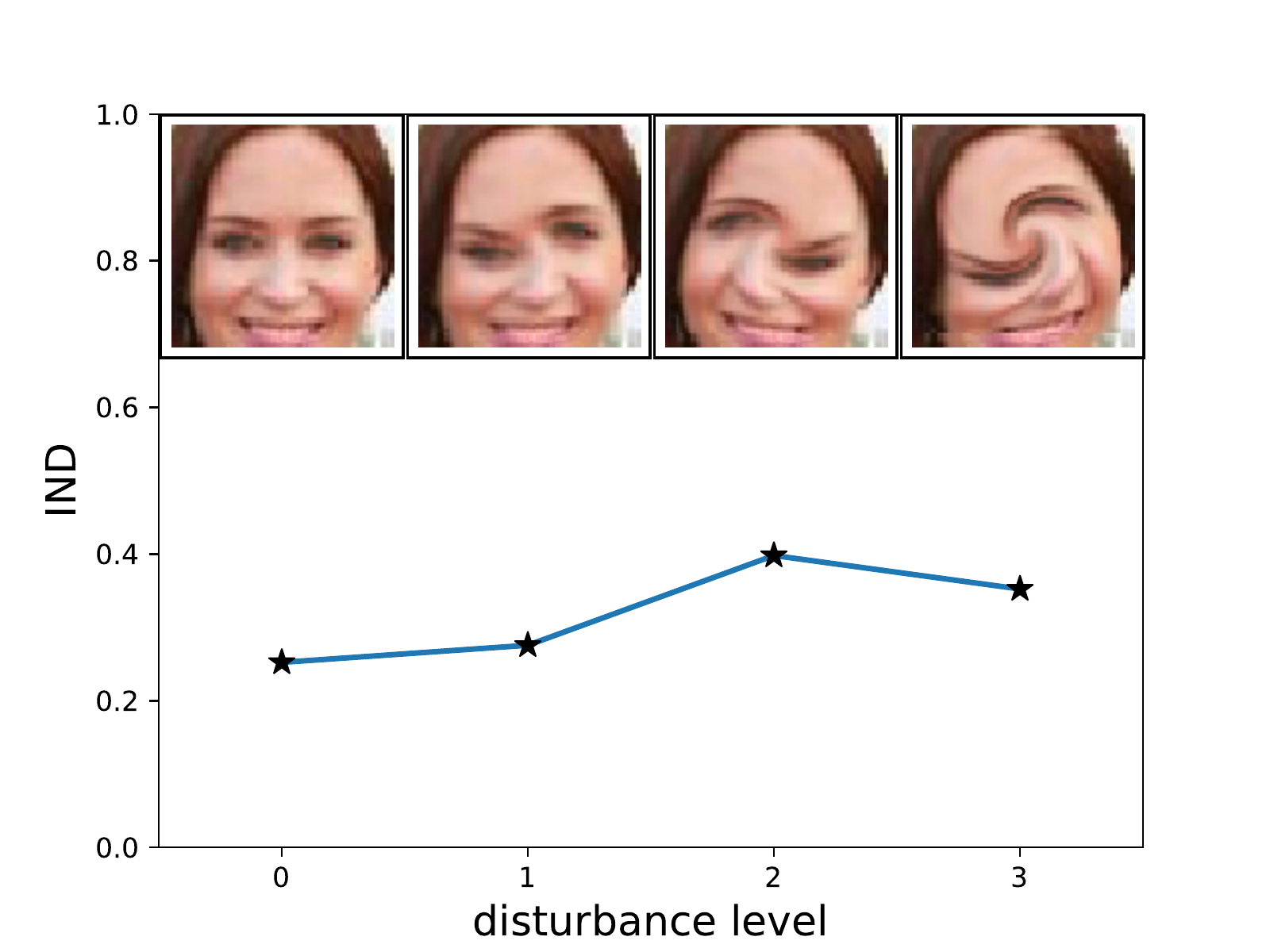}
\includegraphics[width=0.49\textwidth, height=3.5cm]{figures/sp_FID}
\includegraphics[width=0.49\textwidth, height=3.5cm]{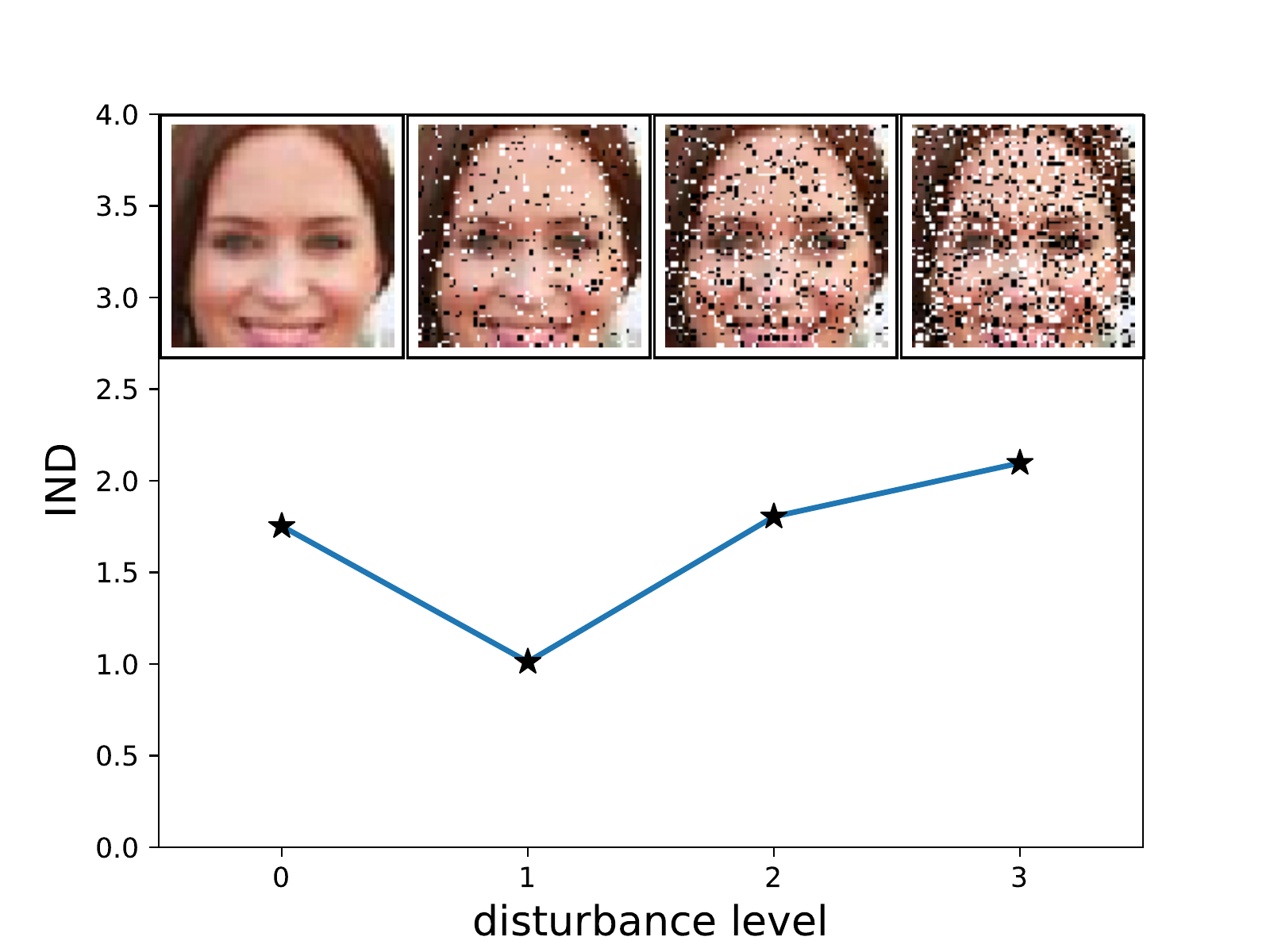}
\includegraphics[width=0.49\textwidth, height=3.5cm]{figures/mixed_FID}
\includegraphics[width=0.49\textwidth, height=3.5cm]{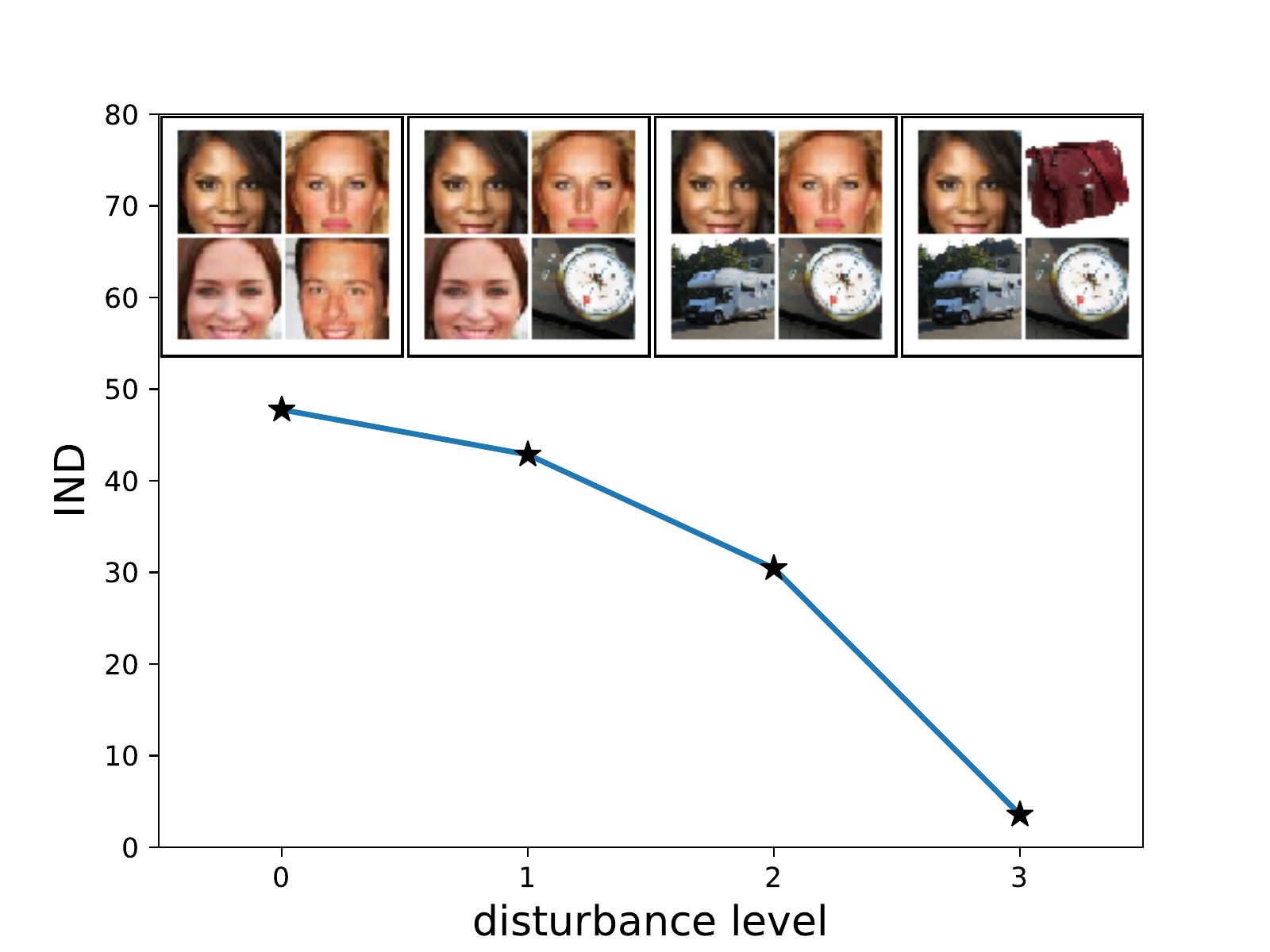}
\caption[FID and Inception Score Comparison]{{\bf Left:} FID and {\bf right:}
Inception Score are evaluated for {\bf first row:} Gaussian noise, {\bf second row:} Gaussian blur, {\bf third row:}
implanted black rectangles, {\bf fourth row:} swirled images, {\bf fifth row.}
salt and pepper noise, and {\bf sixth row:} the CelebA dataset contaminated by
ImageNet images.
Left is the smallest disturbance level of zero, which increases to the highest
level at right. The FID captures the disturbance level very well by
monotonically increasing whereas the Inception Score fluctuates, stays flat or
even, in the worst case, decreases.
  \label{fig:fidind} }
\end{figure}

\begin{figure}[H]
\includegraphics[width=0.49\textwidth]{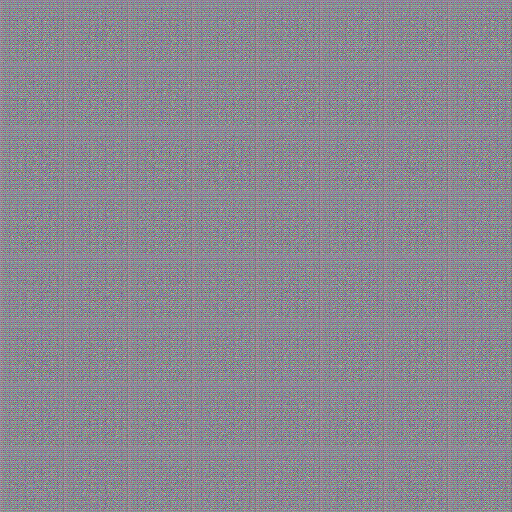}
\includegraphics[width=0.49\textwidth]{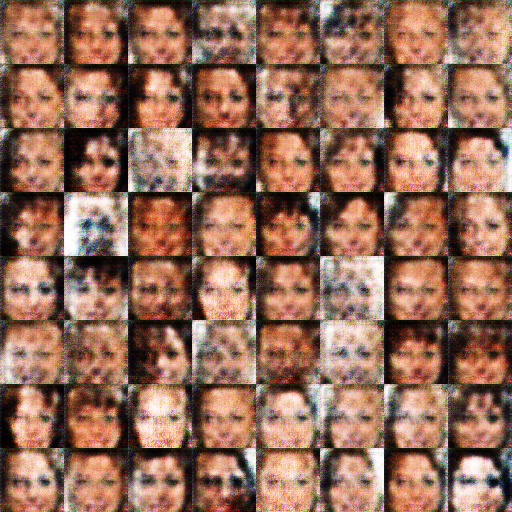}
\caption[CelebA Samples with FID 500 and 300]{Samples generated from DCGAN
trained on CelebA with different FIDs.
{\bf Left:} FID 500 and {\bf Right:} FID 300.
  \label{fig:fid1} }
\end{figure}

\begin{figure}[H]
\includegraphics[width=0.49\textwidth]{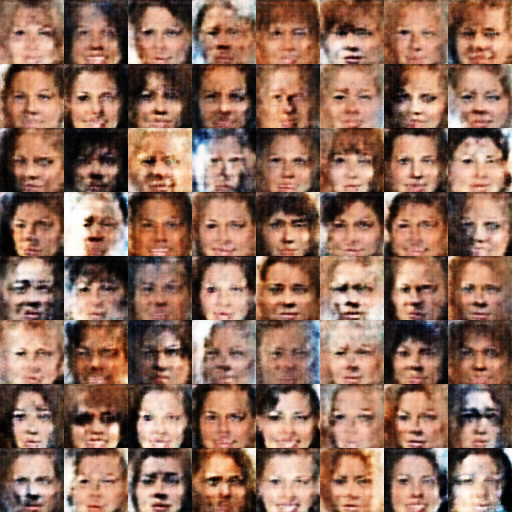}
\includegraphics[width=0.49\textwidth]{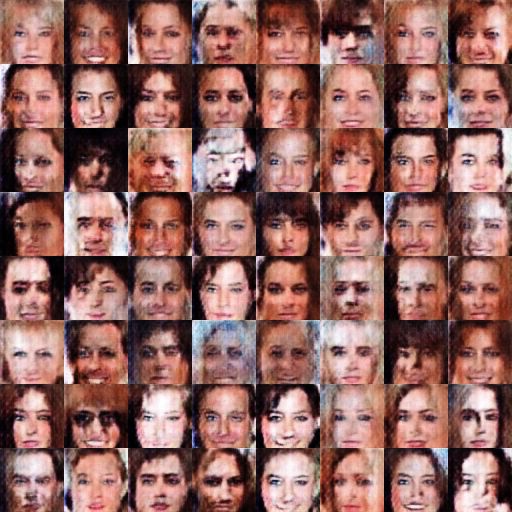}
\caption[CelebA Samples with FID 133 and 100]{Samples generated from DCGAN
trained on CelebA with different FIDs.
{\bf Left:} FID 133 and {\bf Right:} FID 100.
  \label{fig:fid2} }
\end{figure}

\begin{figure}[H]
\includegraphics[width=0.49\textwidth]{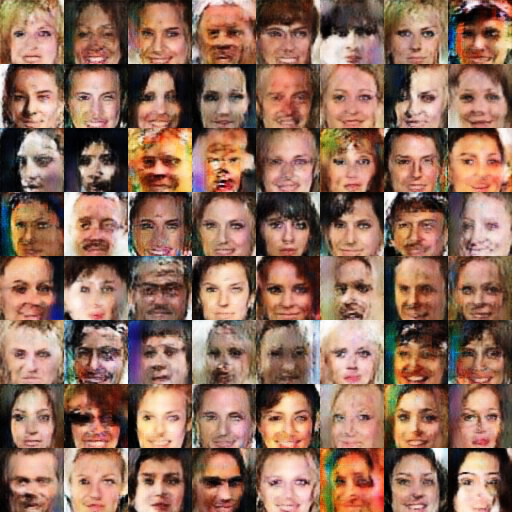}
\includegraphics[width=0.49\textwidth]{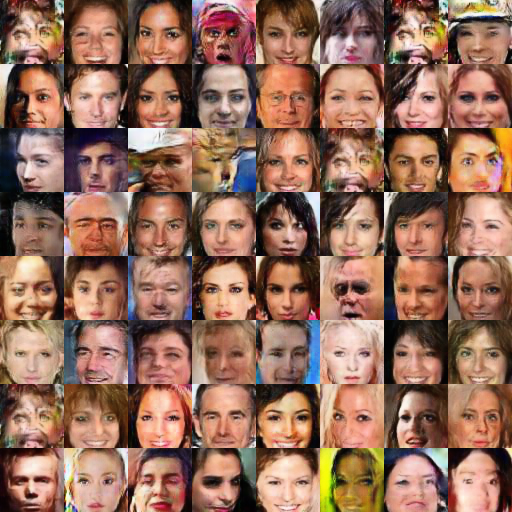}
\caption[CelebA Samples with FID 45 and 13]{Samples generated from DCGAN
trained on CelebA with different FIDs.
{\bf Left:} FID 45 and {\bf Right:} FID 13.
  \label{fig:fid3} }
\end{figure}

\begin{figure}[H]
\includegraphics[width=0.9\textwidth]{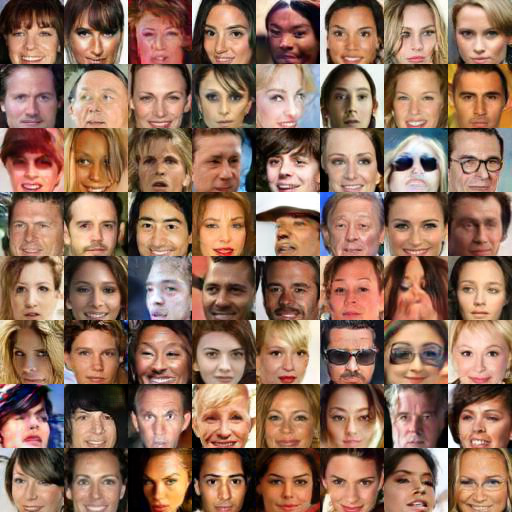}
\caption[CelebA Samples with FID 3]{Samples generated from WGAN-GP
trained on CelebA with a FID of 3.
  \label{fig:fid4} }
\end{figure}

\section{Two Time-Scale Stochastic Approximation Algorithms}
\label{sec:background}

Stochastic approximation algorithms are iterative procedures to find
a root or a stationary point (minimum, maximum, saddle point) of a
function when only noisy observations of its values or its
derivatives are provided.
Two time-scale stochastic approximation algorithms are two coupled
iterations with different step sizes. For proving convergence of these
interwoven iterates it is assumed that one
step size is considerably smaller than the other.
The slower iterate (the one with smaller step size) is assumed to be
slow enough to allow the fast iterate converge while being perturbed
by the the slower. The perturbations of the slow should be small
enough to ensure convergence of the faster.

The iterates map at time step $n\geq 0$ the fast variable $\Bw_n \in \dR^k$ and the slow
variable $\Bth_n \in \dR^m$ to their new values:

\begin{align}
\label{eq:iter1}
\Bth_{n+1} \ &= \ \Bth_n \ + \ a(n) \ \left(\Bh\big(\Bth_n, \Bw_n,
               \BZ^{(\theta)}_n \big) \ + \ \BM^{(\theta)}_{n}\right) \ ,\\
\label{eq:iter2}
\Bw_{n+1} \ &= \ \Bw_n  \ + \ b(n)\ \left(\Bg\big(\Bth_n, \Bw_n,
              \BZ^{(w)}_n\big) \ + \ \BM^{(w)}_{n}\right) \ .
\end{align}
The iterates use
\begin{itemize}
\item $\Bh(.)\in \dR^m$:
mapping for the slow iterate Eq.~\eqref{eq:iter1},
\item $\Bg(.)\in \dR^k$:
mapping for the fast iterate Eq.~\eqref{eq:iter2},
\item $a(n)$:
step size for the slow iterate Eq.~\eqref{eq:iter1},
\item $b(n)$:
step size for the fast iterate Eq.~\eqref{eq:iter2},
\item $\BM^{(\theta)}_n$:
additive random Markov process for the slow iterate
Eq.~\eqref{eq:iter1},
\item $\BM^{(w)}_n$:
additive random Markov process for the fast iterate
Eq.~\eqref{eq:iter2},
\item $\BZ^{(\theta)}_n$:
random Markov process for the slow iterate
Eq.~\eqref{eq:iter1},
\item $\BZ^{(w)}_n$:
random Markov process for the fast iterate
Eq.~\eqref{eq:iter2}.
\end{itemize}

\subsection{Convergence of Two Time-Scale Stochastic Approximation Algorithms}
\label{sec:convergence}

\subsubsection{Additive Noise}
The first result is from Borkar 1997 \cite{Borkar:97}
which was generalized in Konda and Borkar 1999 \cite{Konda:99}.
Borkar considered the iterates:
\begin{align}
\label{eq:iter1Borkar}
\Bth_{n+1} \ &= \ \Bth_n \ + \ a(n) \ \left(\Bh\big(\Bth_n, \Bw_n
               \big) \ + \ \BM^{(\theta)}_{n}\right) \ ,\\
\label{eq:iter2Borkar}
\Bw_{n+1} \ &= \ \Bw_n  \ + \ b(n)\ \left(\Bg\big(\Bth_n, \Bw_n\big) \
              + \ \BM^{(w)}_{n}\right) \ .
\end{align}

\paragraph{Assumptions.}
We make the following assumptions:
\begin{enumerate}[label=\textbf{(A\arabic*)}]
\item Assumptions on the update functions:
The functions $\Bh: \dR^{k+m} \mapsto \dR^{m}$ and  $\Bg: \dR^{k+m}
\mapsto \dR^{k}$ are Lipschitz.
\item
Assumptions on the learning rates:
\begin{align}
&\sum_{n} a(n) \ = \ \infty \quad , \quad
\sum_{n} a^2(n) \ < \ \infty \ , \\
&\sum_{n} b(n) \ = \ \infty \quad , \quad
\sum_{n} b^2(n) \ < \ \infty \ , \\
&a(n) \ = \ \Ro(b(n))\ ,
\end{align}

\item Assumptions on the noise:
For the increasing $\sigma$-field
\begin{align} \nonumber
\cF_n \ &= \
\sigma(\Bth_l, \Bw_l, \BM^{(\theta)}_{l}, \BM^{(w)}_{l},
l \leq n), n \geq 0 \ ,
\end{align}
the sequences of random variables
$(\BM^{(\theta)}_{n},\cF_n)$ and
$(\BM^{(w)}_{n},\cF_n)$ satisfy
\begin{align}
\sum_n a(n) \ \BM^{(\theta)}_{n} \ &< \ \infty \ \text{a.s.} \\
\sum_n b(n) \ \BM^{(w)}_{n} \ &< \ \infty \ \text{a.s.} \ .
\end{align}

\item Assumption on the existence of a solution of the fast iterate:
For each $\Bth \in \dR^m$, the ODE
\begin{align}
\dot{\Bw}(t) \ &= \ \Bg\big(\Bth, \Bw(t)\big) \
\end{align}
has a unique global asymptotically stable equilibrium
$\Bla(\Bth)$ such that $\Bla: \dR^m \mapsto \dR^k$ is Lipschitz.

\item Assumption on the existence of a solution of the slow iterate:
The ODE
\begin{align}
\dot{\Bth}(t) \ &= \ \Bh\big(\Bth(t), \Bla(\Bth(t))\big) \
\end{align}
has a unique global asymptotically stable
equilibrium $\Bth^{*}$.

\item Assumption of bounded iterates:
\begin{align}
\sup_n \| \Bth_n \| \ &< \ \infty \ , \\
\sup_n \| \Bw_n \| \ &< \ \infty \ .
\end{align}
\end{enumerate}

\paragraph{Convergence Theorem}

The next theorem is from Borkar 1997 \cite{Borkar:97}.
\begin{theorem}[Borkar]
\label{th:borkar}
If the assumptions are satisfied,
then the iterates Eq.~\eqref{eq:iter1Borkar} and Eq.~\eqref{eq:iter2Borkar}
converge to $(\Bth^{*}, \Bla(\Bth^{*}))$ a.s.
\end{theorem}

\paragraph{Comments}
\begin{enumerate}[label=\textbf{(C\arabic*)}]
\item
According to Lemma~2 in \cite{Bertsekas:00} Assumption (A3)
is fulfilled if
$\{\BM^{(\theta)}_n\}$ is a martingale difference sequence
w.r.t $\cF_n$ with
\begin{align} \nonumber
\rE \left[\|\BM^{(\theta)}_n \|^2 \mid \cF_n^{(\theta)} \right]
  \ &\leq \ B_1
\end{align} and
$\{\BM^{(w)}_n\}$ is a martingale difference sequence
w.r.t  $\cF_n$ with
\begin{align} \nonumber
\rE \left[\|\BM^{(w)}_n \|^2 \mid \cF_n^{(w)} \right]
  \ &\leq \ B_2 \ ,
\end{align}
where $B_1$ and $B_2$ are positive deterministic constants.
\item
Assumption (A3) holds for mini-batch learning which is the most
frequent case of stochastic gradient.
The batch gradient is
$\BG_{n}:=\nabla_{\theta} (\frac{1}{N} \sum_{i=1}^{N}
f(\Bx_i,\theta)), 1 \leq i \leq N$ and
the mini-batch gradient for batch size $s$ is
$\Bh_{n}:=\nabla_{\theta} (\frac{1}{s} \sum_{i=1}^{s}
f(\Bx_{u_i},\theta)), 1 \leq u_i \leq N$, where the indexes $u_i$
are randomly and uniformly chosen.
For the noise $\BM^{(\theta)}_{n}:=\Bh_{n}-\BG_{n}$ we have
$\rE[\BM^{(\theta)}_{n}]=\rE[\Bh_{n}]-\BG_{n}=\BG_{n}-\BG_{n}=0$.
Since the indexes are chosen without knowing past events,
we have a martingale difference sequence.
For bounded gradients we have bounded $\|\BM^{(\theta)}_n \|^2$.

\item We address assumption (A4)
with weight decay in two ways: (I) Weight decay avoids problems with a
discriminator that is region-wise constant and, therefore, does not have a
locally stable generator. If the generator is perfect, then the discriminator is
0.5 everywhere. For generator with mode collapse, (i) the discriminator is 1 in
regions without generator examples, (ii) 0 in regions with generator examples
only, (iii) is equal to the local ratio of real world examples for regions with
generator and real world examples. Since the discriminator is locally constant,
the generator has gradient zero and cannot improve. Also the discriminator
cannot improve, since it has minimal error given the current generator. However,
without weight decay the Nash Equilibrium is not stable since the second order
derivatives are zero, too. (II) Weight decay avoids that the generator is driven
to infinity with unbounded weights. For example a linear discriminator can supply a gradient for the
generator outside each bounded region.

\item
The main result used in the proof of the theorem relies on work on
perturbations of ODEs according to Hirsch 1989 \cite{Hirsch:89}.

\item
Konda and Borkar 1999 \cite{Konda:99} generalized the convergence
proof to distributed asynchronous update rules.

\item
Tadi\'{c} relaxed the assumptions for showing convergence \cite{Tadic:04a}.
In particular the noise assumptions (Assumptions A2 in
\cite{Tadic:04a}) do not have to be martingale
difference sequences and are more general than in
\cite{Borkar:97}. In another result the assumption of bounded iterates
is not necessary if other assumptions are ensured \cite{Tadic:04a}.
Finally, Tadi\'{c} considers the case of non-additive noise \cite{Tadic:04a}.
{\bf Tadi\'{c} does not provide proofs for his results.}
We were not able to find such proofs even in other publications of Tadi\'{c}.

\end{enumerate}

\subsubsection{Linear Update, Additive Noise, and Markov Chain}
\label{sec:linupnoisemc}

In contrast to the previous subsection, we assume that an additional Markov
chain influences the iterates \cite{Konda:02,Konda:03}.
The Markov chain allows applications in reinforcement learning, in
particular in actor-critic setting where the Markov chain is used to
model the environment. The slow iterate is the actor
update while the fast iterate is the critic update.
For reinforcement learning both the actor and the critic observe the
environment which is driven by the actor actions. The environment
observations are assumed to be a Markov chain. The Markov chain can
include eligibility traces which are modeled as explicit states in
order to keep the Markov assumption.

The Markov chain is the sequence of observations of the environment
which progresses via transition probabilities.
The transitions are not affected by the critic but
by the actor.

Konda et al. considered the iterates \cite{Konda:02,Konda:03}:
\begin{align}
\label{eq:iter1Konda}
\Bth_{n+1} \ &= \ \Bth_n \ + \ a(n) \ \BH_n \ ,\\
\label{eq:iter2Konda}
\Bw_{n+1} \ &= \ \Bw_n  \ + \ b(n)\ \left(
\Bg\big(\BZ^{(w)}_n;\Bth_n\big) \ + \ \BG\big( \BZ^{(w)}_n;\Bth_n\big) \ \ \Bw_n
              + \ \BM^{(w)}_{n} \ \Bw_n \right) \ .
\end{align}
$\BH_n$ is a random process that drives the changes of
$\Bth_n$. We assume that $\BH_n$ is a slow enough process.
We have a linear update rule for the fast iterate using
the vector function $\Bg(.)\in \dR^k$ and
the matrix function $\BG(.)\in \dR^{k\times k}$.

\paragraph{Assumptions.}
We make the following assumptions:
\begin{enumerate}[label=\textbf{(A\arabic*)}]
\item Assumptions on the Markov process, that is, the transition kernel:
The stochastic process $\BZ^{(w)}_n$ takes values
in a Polish (complete, separable, metric) space $\dZ$
with the Borel $\sigma$-field
\begin{align} \nonumber
\cF_n \ &= \
\sigma(\Bth_l, \Bw_l, \BZ^{(w)}_l,\BH_l,
l \leq n), n \geq 0 \ .
\end{align}
For every measurable set $A \subset \dZ$ and the parametrized transition kernel
$\rP(.;\Bth_n)$ we have:
\begin{align}
\rP(\BZ^{(w)}_{n+1} \in A \mid \cF_n) \ &= \
\rP(\BZ^{(w)}_{n+1} \in A \mid  \BZ^{(w)}_n;\Bth_n ) \ = \
\rP(\BZ^{(w)}_n ,A ; \Bth_n ) \ .
\end{align}
We define for every measurable function $f$
\begin{align} \nonumber
\rP_{\Bth}f(\Bz) \ &:= \ \int \rP(\Bz ,\Rd \bar{\Bz} ; \Bth_n )
\ f( \bar{\Bz}) \ .
\end{align}

\item
Assumptions on the learning rates:
\begin{align}
&\sum_{n} b(n) \ = \ \infty \quad , \quad
\sum_{n} b^2(n) \ < \ \infty \ , \\
&\sum_{n} \left(\frac{a(n)}{b(n)}\right)^d \ < \ \infty \ ,
\end{align}
for some $d>0$.

\item Assumptions on the noise:
The sequence $\BM^{(w)}_{n}$ is a $k \times k$-matrix valued
$\cF_n$-martingale difference with bounded moments:
\begin{align}
\rE \left[\BM^{(w)}_{n} \mid  \cF_n\right] \ &= \ 0 \ , \\
\sup_n \rE \left[\left\|\BM^{(w)}_{n} \right\|^d \right] \ &< \ \infty
  \ , \ \forall d > 0 \ .
\end{align}

We assume slowly changing $\Bth$, therefore the random process $\BH_n$
satisfies
\begin{align}
\sup_n \rE \left[\left\|\BH_{n} \right\|^d \right] \ &< \ \infty
  \ , \ \forall d > 0 \ .
\end{align}

\item Assumption on the existence of a solution of the fast iterate:
We assume the existence of a solution to the Poisson equation for
the fast iterate.
For each $\Bth \in \dR^m$, there exist functions
$\bar{\Bg}(\Bth) \in \dR^k$, $\bar{\BG}(\Bth) \in \dR^{k \times k}$,
$\hat{\Bg}(\Bz;\Bth): \ \dZ \to \dR^k$,
and $\hat{\BG}(\Bz;\Bth): \  \dZ \to \dR^{k \times k}$ that satisfy
the Poisson equations:
\begin{align}
\hat{\Bg}(\Bz;\Bth) \ &= \ \Bg(\Bz;\Bth) \ - \ \bar{\Bg}(\Bth) \ + \
(\rP_{\Bth} \hat{\Bg}(.;\Bth))(\Bz) \ , \\
\hat{\BG}(\Bz;\Bth) \ &= \ \BG(\Bz;\Bth) \ - \ \bar{\BG}(\Bth) \ + \
(\rP_{\Bth} \hat{\BG}(.;\Bth))(\Bz) \ .
\end{align}

\item Assumptions on the update functions and solutions to the Poisson
equation:

\begin{enumerate}
\item Boundedness of solutions: For some constant $C$ and for all $\Bth$:
\begin{align}
\max\{ \| \bar{g}(\Bth) \| \} \ &\leq \ C \ , \\
\max\{ \| \bar{G}(\Bth) \| \} \ &\leq \ C \ .
\end{align}
\item Boundedness in expectation: All moments are bounded. For any
  $d>0$, there exists $C_d>0$ such that
\begin{align}
\sup_n \rE \left[\left\| \hat{\Bg}(\BZ^{(w)}_{n};\Bth) \right\|^d \right] \ &\leq \ C_d \ , \\
\sup_n \rE \left[\left\| \Bg(\BZ^{(w)}_{n};\Bth)\right\|^d \right] \ &\leq \ C_d \ , \\
\sup_n \rE \left[\left\| \hat{\BG}(\BZ^{(w)}_{n};\Bth) \right\|^d \right] \ &\leq \ C_d \ , \\
\sup_n \rE \left[\left\| \BG(\BZ^{(w)}_{n};\Bth)\right\|^d \right] \ &\leq \ C_d \ .
\end{align}
\item Lipschitz continuity of solutions:
For some constant $C>0$ and for all $\Bth$,$\bar{\Bth} \in \dR^m$:
\begin{align}
\left\| \bar{\Bg}(\Bth) \ - \
       \bar{\Bg}(\bar{\Bth}) \right\|
\ &\leq \ C \ \| \Bth - \bar{\Bth} \| \ ,\\
\left\| \bar{\BG}(\Bth) \ - \
       \bar{\BG}(\bar{\Bth}) \right\|
\ &\leq \ C \ \| \Bth - \bar{\Bth} \| \ .
\end{align}

\item Lipschitz continuity in expectation:
There exists a positive measurable function $C(.)$ on $\dZ$ such that
\begin{align}
\sup_n \rE \left[C(\BZ^{(w)}_{n} )^d \right] \ &< \ \infty
  \ , \ \forall d > 0 \ .
\end{align}
Function $C(.)$ gives the Lipschitz constant for every $\Bz$:
\begin{align}
\left\| (\rP_{\Bth} \hat{\Bg}(.;\Bth))(\Bz) \ - \
       (\rP_{\bar{\Bth}} \hat{\Bg}(.;\bar{\Bth}))(\Bz) \right\|
\ &\leq \ C(\Bz) \ \| \Bth - \bar{\Bth} \| \ ,\\
\left\| (\rP_{\Bth} \hat{\BG}(.;\Bth))(\Bz) \ - \
       (\rP_{\bar{\Bth}} \hat{\BG}(.;\bar{\Bth}))(\Bz) \right\|
\ &\leq \ C(\Bz) \ \| \Bth - \bar{\Bth} \| \ .
\end{align}

\item Uniform positive definiteness:
There exists some $\alpha>0$ such that for all $\Bw \in \dR^k$ and $\Bth
\in \dR^m$:
\begin{align}
\Bw^{T} \ \bar{G}(\Bth) \ \Bw \ &\geq \ \alpha \ \| \Bw \|^2 \ .
\end{align}
\end{enumerate}

\end{enumerate}

\paragraph{Convergence Theorem.}

We report Theorem~3.2 (see also  Theorem~7 in \cite{Konda:03})
and Theorem 3.13 from \cite{Konda:02}:

\begin{theorem}[Konda \& Tsitsiklis]
\label{th:konda}
If the assumptions are satisfied,
then for the iterates Eq.~\eqref{eq:iter1Konda} and Eq.~\eqref{eq:iter2Konda}
holds:
\begin{align}
&\lim_{n \to \infty} \left\| \bar{G}(\Bth_n) \ \Bw_n \ - \
  \bar{g}(\Bth_n) \right\| \ = \ 0 \ \ \ \text{a.s.} \ , \\
&\lim_{n \to \infty} \left\| \Bw_n \ - \ \bar{G}^{-1}(\Bth_n) \
  \bar{g}(\Bth_n) \right\| \ = \ 0 \ .
\end{align}
\end{theorem}

\paragraph{Comments.}

\begin{enumerate}[label=\textbf{(C\arabic*)}]
\item
The proofs only use the boundedness of the moments of $\BH_n$ \cite{Konda:02,Konda:03},
therefore $\BH_n$ may depend on $\Bw_n$.
In his PhD thesis \cite{Konda:02}, Vijaymohan Konda used this framework for
the actor-critic learning, where $\BH_n$ drives the updates of the
actor parameters $\Bth_n$.
However, the actor updates are based on the current parameters $\Bw_n$ of
the critic.
\item
The random process $\BZ^{(w)}_{n}$ can affect $\BH_n$ as long as
boundedness is ensured.
\item Nonlinear update rule.
$\Bg\big(\BZ^{(w)}_n; \Bth_n \big) \ + \ \BG\big(
\BZ^{(w)}_n ; \Bth_n\big)\Bw_n$ can be viewed as a linear approximation of a
nonlinear update rule. The nonlinear case has been considered in
\cite{Konda:02} where additional approximation errors due to
linearization were addressed. These errors are
treated in the given framework \cite{Konda:02}.
\end{enumerate}

\subsubsection{Additive Noise and Controlled Markov Processes}
\label{sec:addnoisecmp}
The most general iterates use nonlinear update functions $\Bg$ and
$\Bh$, have additive noise,
and have controlled Markov processes \cite{Karmakar:17}.

\begin{align}
\label{eq:iter1Karmakar}
\Bth_{n+1} \ &= \ \Bth_n \ + \ a(n) \ \left(\Bh\big(\Bth_n, \Bw_n,
               \BZ^{(\theta)}_n \big) \ + \ \BM^{(\theta)}_{n}\right) \ ,\\
\label{eq:iter2Karmakar}
\Bw_{n+1} \ &= \ \Bw_n  \ + \ b(n)\ \left(\Bg\big(\Bth_n, \Bw_n,
              \BZ^{(w)}_n\big) \ + \ \BM^{(w)}_{n}\right) \ .
\end{align}

\paragraph{Required Definitions.}

\textit{Marchaud Map}: A set-valued map $\Bh: \dR^l \to \{\text{subsets of}\ \dR^k$\}
is called a \textit{Marchaud map} if it satisfies
the following properties:
\begin{itemize}
 \item[(i)] For each $\Bth \in \dR^l$, $\Bh(\Bth)$ is convex and compact.
 \item[(ii)] \textit{(point-wise boundedness)} For each $\Bth \in \dR^l$,
 $\underset{\Bw \in \Bh(\Bth)}{\sup} \
\lVert \Bw \rVert < K \left( 1 + \lVert \Bth \rVert \right)$
for some $K > 0$.
 \item[(iii)] $\Bh$ is an \textit{upper-semicontinuous} map. \\ \indent
 We say that $\Bh$ is upper-semicontinuous,
  if given sequences $\{ \Bth_{n} \}_{n \ge 1}$ (in $\dR^l$) and
  $\{ \By_{n} \}_{n \ge 1}$ (in $\dR^k$)  with
  $\Bth_{n} \to \Bth$, $\By_{n} \to \By$ and
  $\By_{n} \in \Bh(\Bth_{n}), n \ge 1, \By \in \Bh(\Bth)$.
   In other words, the graph of
   $\Bh, \ \left\{ (\Bx, \By) \ : \ \By \in \Bh(\Bx),\ \Bx\in \dR^l \right\}$,
  is closed in $\dR^l \times \dR^k$.
\end{itemize}

If the set-valued map $H: \dR^m \to \{\text{subsets of}\ \dR^m \}$
is  Marchaud, then
the differential inclusion (DI) given by
\begin{align}
\label{eq:di}
&\dot{\Bth}(t) \ \in \ H(\Bth(t))
\end{align}
is guaranteed to have at least one solution that is absolutely continuous.
If $\BTh$ is an absolutely continuous map satisfying  Eq.~\eqref{eq:di} then we say that $\BTh \in \BSi$.

\textit{Invariant Set}:
$M \subseteq \dR^m$ is \textit{invariant} if for every $\Bth \in M$ there exists
a trajectory, $\BTh$, entirely in $M$
with $\BTh(0) = \Bth$.
In other words, $\BTh \in \BSi$ with $\BTh(t) \in M$,
for all $t \ge 0$.
\\ \indent
\textit{Internally Chain Transitive Set}:
$M \subset \dR^m$ is said to be
internally chain transitive if $M$ is compact and for every $\Bth, \By \in M$,
$\epsilon >0$ and $T > 0$ we have the following: There exist $\Phi^{1}, \ldots, \Phi^{n}$ that
are $n$ solutions to the differential inclusion $\dot{\Bth}(t) \in h(\Bth(t))$,
a sequence $\Bth_1(=\Bth), \ldots, \Bth_{n+1} (=\By) \subset M$
and $n$ real numbers
$t_{1}, t_{2}, \ldots, t_{n}$ greater than $T$ such that:
$\Phi^i_{t_{i}}(\Bth_i) \in N^\epsilon(\Bth_{i+1})$
where $N^\epsilon(\Bth)$ is the open $\epsilon$-neighborhood of $\Bth$ and
$\Phi^{i}_{[0, t_{i}]}(\Bth_i) \subset M$
for $1 \le i \le n$. The sequence $(\Bth_{1}(=\Bth), \ldots, \Bth_{n+1}(=\By))$
is called an $(\epsilon, T)$ chain in $M$ from $\Bth$ to $\By$.

\paragraph{Assumptions.}

We make the following assumptions \cite{Karmakar:17}:
\begin{enumerate}[label=\textbf{(A\arabic*)}]
\item Assumptions on the controlled Markov processes:
The controlled Markov process
$\{\BZ^{(w)}_n\}$ takes values in a compact metric space $S^{(w)}$.
The controlled Markov process
$\{\BZ^{(\theta)}_n\}$ takes values in a compact metric space $S^{(\theta)}$.
Both processes are controlled by the iterate sequences $\{\Bth_n\}$
and  $\{\Bw_n\}$. Furthermore $\{\BZ^{(w)}_n\}$ is additionally
controlled by a random process $\{\BA^{(w)}_n\}$ taking values in a
compact metric space $U^{(w)}$ and
$\{\BZ^{(\theta)}_n\}$ is additionally
controlled by a random process $\{\BA^{(\theta)}_n\}$
taking values in a
compact metric space $U^{(\theta)}$.
The $\{\BZ^{(\theta)}_n\}$ dynamics is
\begin{align}
\rP(\BZ^{(\theta)}_{n+1} \in B^{(\theta)} |\BZ^{(\theta)}_l,
\BA^{(\theta)}_l, \Bth_l, \Bw_l, l\leq n) \ &= \
\int_{B^{(\theta)}} p^{(\theta)}(\Rd z| \BZ^{(\theta)}_n,
\BA^{(\theta)}_n, \Bth_n, \Bw_n), n\geq 0 \ ,
\end{align}
for $B^{(\theta)}$ Borel in $S^{(\theta)}$.
The $\{\BZ^{(w)}_n\}$ dynamics is
\begin{align}
\rP(\BZ^{(w)}_{n+1} \in B^{(w)} |\BZ^{(w)}_l,
\BA^{(w)}_l, \Bth_l, \Bw_l, l\leq n) \ &= \
\int_{B^{(w)}} p^{(w)}(\Rd z| \BZ^{(w)}_n,
\BA^{(w)}_n, \Bth_n, \Bw_n), n\geq 0 \ ,
\end{align}
for $B^{(w)}$ Borel in $S^{(w)}$.

\item Assumptions on the update functions:
$\Bh :  \dR^{m+k} \times S^{(\theta)} \to \dR^m$ is
jointly continuous as well as Lipschitz in
its first two arguments uniformly w.r.t.\ the third.
The latter condition means that
\begin{align}
&\forall \Bz^{(\theta)} \in S^{(\theta)}: \
 \|\Bh(\Bth, \Bw, \Bz^{(\theta)}) \ - \ \Bh(\Bth', \Bw',
  \Bz^{(\theta)})\| \ \leq \ L^{(\theta)} \ (\|\Bth-\Bth'\| + \|\Bw -
  \Bw'\|) \ .
\end{align}
Note that the Lipschitz constant $L^{(\theta)}$ does not depend on
$\Bz^{(\theta)}$.

$\Bg :  \dR^{k+m} \times S^{(w)} \to \dR^k$ is
jointly continuous as well as Lipschitz in
its first two arguments uniformly w.r.t.\ the third.
The latter condition means that
\begin{align}
&\forall \Bz^{(w)} \in S^{(w)}: \
 \|\Bg(\Bth, \Bw, \Bz^{(w)}) \ - \ \Bg(\Bth', \Bw',
  \Bz^{(w)})\| \ \leq \ L^{(w)} \ (\|\Bth-\Bth'\| + \|\Bw -
  \Bw'\|) \ .
\end{align}
Note that the Lipschitz constant $L^{(w)}$ does not depend on
$\Bz^{(w)}$.

\item Assumptions on the additive noise:
$\{\BM^{(\theta)}_{n}\}$ and $\{\BM^{(w)}_{n}\}$
are martingale difference sequence with second
moments bounded by $K(1+\|\Bth_n\|^2 + \|\Bw_n\|^2)$.
More precisely,
$\{\BM^{(\theta)}_{n}\}$ is a martingale difference sequence
w.r.t.\ increasing $\sigma$-fields
\begin{align}
\cF_n \ &= \ \sigma(\Bth_l, \Bw_l, \BM^{(\theta)}_{l}, \BM^{(w)}_{l},
          \BZ^{(\theta)}_{l}, \BZ^{(w)}_{l}, l \leq n), \ n \geq 0 \ ,
\end{align}
satisfying
\begin{align}
\rE \left[\|\BM^{(\theta)}_{n+1}\|^2 \mid \cF_n \right] \ &\leq \
K \ (1 \ + \ \|\Bth_n\|^2 \ + \ \|\Bw_n\|^2) \ ,
\end{align}
for $n \geq 0$ and a given constant $K>0$.

$\{\BM^{(w)}_{n}\}$ is a martingale difference sequence
w.r.t.\ increasing $\sigma$-fields
\begin{align}
\cF_n \ &= \ \sigma(\Bth_l, \Bw_l, \BM^{(\theta)}_{l}, \BM^{(w)}_{l},
          \BZ^{(\theta)}_{l}, \BZ^{(w)}_{l}, l \leq n), \ n \geq 0 \ ,
\end{align}
satisfying
\begin{align}
\rE \left[\|\BM^{(w)}_{n+1}\|^2 \mid \cF_n \right] \ &\leq \
K \ (1 \ + \ \|\Bth_n\|^2 \ + \ \|\Bw_n\|^2) \ ,
\end{align}
for $n \geq 0$ and a given constant $K>0$.

\item
Assumptions on the learning rates:
\begin{align}
&\sum_{n} a(n) \ = \ \infty \quad , \quad
\sum_{n} a^2(n) \ < \ \infty \ , \\
&\sum_{n} b(n) \ = \ \infty \quad , \quad
\sum_{n} b^2(n) \ < \ \infty \ , \\
&a(n) \ = \ \Ro(b(n))\ ,
\end{align}
Furthermore, $a(n), b(n),
n \geq 0$ are non-increasing.

\item Assumptions on the controlled
Markov processes, that is, the transition kernels:
The state-action map
\begin{align}
S^{(\theta)} \times U^{(\theta)} \times \dR^{m+k} \ni (\Bz^{(\theta)}, \Ba^{(\theta)}, \Bth, \Bw)  \
&\to \ \Rp^{(\theta)}(\Rd\By \mid \Bz^{(\theta)}, \Ba^{(\theta)},
  \Bth, \Bw)
\end{align}
and the state-action map
\begin{align}
S^{(w)} \times U^{(w)} \times \dR^{m+k} \ni (\Bz^{(w)}, \Ba^{(w)}, \Bth, \Bw)  \
&\to \ \Rp^{(w)}(\Rd\By \mid \Bz^{(w)}, \Ba^{(w)},
  \Bth, \Bw)
\end{align}
are continuous.

\item Assumptions on the existence of a solution:

We consider {\em occupation measures}
which give for
the controlled Markov process the probability or
density to observe a particular
state-action pair from $S \times U$
for given $\Bth$ and a given control policy $\pi$.
We denote by $D^{(w)}(\Bth,\Bw)$ the set of all ergodic occupation
measures
for the prescribed $\Bth$ and $\Bw$
on state-action space $S^{(w)} \times U^{(\theta)}$
for the controlled Markov process $\BZ^{(w)}$
with policy $\pi^{(w)}$.
Analogously we denote, by
$D^{(\theta)}(\Bth,\Bw)$ the set of all ergodic occupation measures
for the prescribed $\Bth$ and $\Bw$
on state-action space $S^{(\theta)} \times U^{(\theta)}$
for the controlled Markov process $\BZ^{(\theta)}$
with policy $\pi^{(\theta)}$.
Define
\begin{align}
\tilde{\Bg}(\Bth, \Bw, \Bnu) \ &= \ \int \Bg(\Bth,\Bw,\Bz) \ \Bnu(\Rd\Bz, U^{(w)})
\end{align}
for $\Bnu$ a measure on $S^{(w)}\times U^{(w)}$ and the Marchaud map
\begin{align}
\hat{\Bg}(\Bth,\Bw) \ &= \ \{\tilde{\Bg}(\Bth, \Bw, \Bnu): \  \Bnu \in
  D^{(w)}(\Bth, \Bw)\} \ .
\end{align}
We assume that the set $D^{(w)}(\Bth, \Bw)$ is singleton, that is,
$\hat{\Bg}(\Bth,\Bw)$ contains a single function and we use the same
notation for the set and its single element.
If the set is not a singleton,
the assumption of a solution can be expressed by the differential
inclusion
$\dot{\Bw}(t) \in \hat{\Bg}(\Bth,\Bw(t))$ \cite{Karmakar:17}.

$\forall \Bth \in \dR^m$, the ODE
\begin{align}
\dot{\Bw}(t) \ &= \ \hat{\Bg}(\Bth,\Bw(t))
\end{align}
has an asymptotically stable equilibrium $\Bla(\Bth)$
with domain of attraction $G_\theta$
where $\Bla :  \dR^m \to \dR^k$ is a Lipschitz map with constant $K$.
Moreover, the function $V: G \to
[0,\infty)$ is continuously differentiable where $V(\Bth,.)$ is the Lyapunov function
for $\Bla(\Bth)$ and $G=\{(\Bth,\Bw): \ \Bw \in G_\theta, \Bth \in \dR^m\}$. This extra condition is needed
so that the set $\{(\Bth,\Bla(\Bth)): \ \Bth \in \dR^m\}$ becomes an asymptotically stable set of the coupled ODE
\begin{align}
\dot{\Bw}(t) \ &= \ \hat{\Bg}(\Bth(t),\Bw(t)) \, \\
\dot{\Bth}(t) \ &= \ 0 \ .
\end{align}

\item Assumption of bounded iterates:
\begin{align}
\sup_n \| \Bth_n \| \ &< \ \infty \ \text{a.s.} \ , \\
\sup_n \| \Bw_n \| \ &< \ \infty \ \text{a.s.}
\end{align}

\end{enumerate}

\paragraph{Convergence Theorem.}
The following theorem is from Karmakar \& Bhatnagar \cite{Karmakar:17}:
\begin{theorem}[Karmakar \& Bhatnagar]
\label{th:karmakar}
Under above assumptions
if for all $\Bth \in \dR^m$, with probability 1, $\{\Bw_n\}$
belongs to a compact subset $Q_\theta$ (depending
on the sample point) of $G_\theta$ ``eventually'', then
\begin{align}
(\Bth_n, \Bw_n) \ \to \ \cup_{\Bth^* \in A_0}(\Bth^*, \Bla(\Bth^*))
\ \ \text{a.s.} \quad \text{as} \ n \ \to \ \infty \ ,
\end{align}
where $A_0 = \cap_{t\geq 0}\overline{\{\bar{\Bth}(s): s\geq t\}}$
which is almost everywhere an internally chain transitive set
of the differential inclusion
\begin{align}
\dot{\Bth}(t) \ &\in \ \hat{\Bh}(\Bth(t)),
\end{align}
where $\hat{\Bh}(\Bth)=\{\tilde{\Bh}(\Bth,\Bla(\Bth),\Bnu) :  \Bnu
\in D^{(w)}(\Bth, \Bla(\Bth))\}$.
\end{theorem}

\paragraph{Comments.}
\begin{enumerate}[label=\textbf{(C\arabic*)}]
\item
This framework allows to show convergence for gradient descent methods
beyond stochastic gradient
like for the ADAM procedure where current learning parameters are
memorized and updated.
The random processes $\BZ^{(w)}$ and $\BZ^{(\theta)}$ may track the
current learning status for the fast and slow iterate, respectively.
\item
Stochastic regularization like dropout is covered via
the random processes $A^{(w)}$ and $A^{(\theta)}$.
\end{enumerate}

\subsection{Rate of Convergence of Two Time-Scale Stochastic  Approximation
Algorithms}
\label{sec:convergenceRate}

\subsubsection{Linear Update Rules}
\label{sec:linur}
First we consider linear iterates according to the PhD thesis of Konda
\cite{Konda:02} and Konda \& Tsitsiklis \cite{Konda:04}.
\begin{align}
\label{eq:iter1ConvLinear}
\Bth_{n+1} \ &= \ \Bth_n \ + \ a(n) \ \left(\Ba_1 \ - \ \BA_{11} \
               \Bth_n \ - \ \BA_{12} \ \Bw_n \ + \ \BM^{(\theta)}_{n}\right) \ ,\\
\label{eq:iter2ConvLinear}
\Bw_{n+1} \ &= \ \Bw_n  \ + \ b(n)\ \left(\Ba_2  \ - \ \BA_{21} \
               \Bth_n \ - \ \BA_{22} \ \Bw_n \ + \ \BM^{(w)}_{n}\right) \ .
\end{align}

\paragraph{Assumptions.}
We make the following assumptions:
\begin{enumerate}[label=\textbf{(A\arabic*)}]
\item The random variables $(\BM^{(\theta)}_{n},\BM^{(w)}_{n}),
  n=0,1,\ldots$, are independent of $\Bw_0,\Bth_0$ and of each other.
The have zero mean: $\rE[\BM^{(\theta)}_{n}]=0$ and
$\rE[\BM^{(w)}_{n}]=0$. The covariance is
\begin{align}
\rE \left[ \BM^{(\theta)}_{n}\ (\BM^{(\theta)}_{n})^T\right]
\ &= \ \BGa_{11} \ , \\
\rE \left[ \BM^{(\theta)}_{n}\ (\BM^{(w)}_{n})^T\right]
\ &= \ \BGa_{12} \ = \ \BGa_{21}^{T} \ , \\
\rE \left[ \BM^{(w)}_{n}\ (\BM^{(w)}_{n})^T\right]
\ &= \ \BGa_{22} \ .
\end{align}

\item The learning rates are deterministic, positive, nondecreasing
  and satisfy with $\epsilon \leq 0$:
\begin{align}
&\sum_{n} a(n) \ = \ \infty \quad , \quad
\lim_{n \to \infty} a(n) \ = \ \ 0 \ , \\
&\sum_{n} b(n) \ = \ \infty \quad , \quad
\lim_{n \to \infty} b(n) \ = \ \ 0 \ , \\
& \frac{a(n)}{b(n)} \ \to \ \epsilon \ .
\end{align}
We often consider the case $\epsilon=0$.

\item Convergence of the iterates:
We define
\begin{align}
\BDe \ &:= \ \BA_{11} \ - \ \BA_{12} \BA_{22}^{-1} \BA_{21} \ .
\end{align}
A matrix is {\em Hurwitz} if the real part of each eigenvalue is strictly
negative.
We assume that the matrices $-\BA_{22}$ and $-\BDe$ are Hurwitz.

\item Convergence rate remains simple:

\begin{enumerate}
\item
There exists a constant $\bar{a} \leq 0$ such that
\begin{align}
\lim_{n}(a(n+1)^{-1} \ - \ a(n)^{-1}) \ &= \ \bar{a} \ .
\end{align}
\item
If $\epsilon=0$, then
\begin{align}
\lim_{n}(b(n+1)^{-1} \ - \ b(n)^{-1}) \ &= \ 0 \ .
\end{align}
\item
The matrix
\begin{align}
- \ \left( \BDe \ - \ \frac{\bar{a}}{2} \ \BI \right)
\end{align}
is Hurwitz.
\end{enumerate}

\end{enumerate}

\paragraph{Rate of Convergence Theorem.}

The next theorem is taken from  Konda
\cite{Konda:02} and Konda \& Tsitsiklis \cite{Konda:04}.

Let $\Bth^{*} \in \dR^m$ and $\Bw^{*} \in \dR^k$ be the unique
solution to the system of linear equations
\begin{align}
\label{eq:Linear1}
\BA_{11} \ \Bth_n \ + \ \BA_{12} \ \Bw_n \ &= \ \Ba_1 \ , \\
\label{eq:Linear2}
\BA_{21} \ \Bth_n \ + \ \BA_{22} \ \Bw_n   \ &= \ \Ba_2 \ .
\end{align}
For each $n$, let
\begin{align}
\hat{\Bth}_n \ &= \ \Bth_n \ - \ \Bth^{*} \ ,\\
\hat{\Bw}_n \ &= \ \Bw_n \ - \ \BA_{22}^{-1} \ \left(\Ba_2 \ - \
\BA_{21} \ \Bth_n \right) \ ,\\
\BSi_{11}^n \ &= \ \Bth_n^{-1} \ \rE\left[ \hat{\Bth}_n \hat{\Bth}_n^T\right]\ ,\\
\BSi_{12}^n \ &= \ \big( \BSi_{21}^n \big)^T \ = \
\Bth_n^{-1} \ \rE\left[ \hat{\Bth}_n \hat{\Bw}_n^T\right]\ ,\\
\BSi_{22}^n \ &= \ \Bw_n^{-1} \ \rE\left[ \hat{\Bw}_n \hat{\Bw}_n^T\right]\ ,\\
\BSi^n \ &= \
\begin{pmatrix}
\BSi_{11}^n & \BSi_{12}^n \\
\BSi_{21}^n & \BSi_{22}^n
\end{pmatrix} \ .
\end{align}

\begin{theorem}[Konda \& Tsitsiklis]
\label{th:kondaRate}
Under above assumptions and when the constant $\epsilon$ is
sufficiently small, the limit matrices
\begin{align}
& \BSi_{11}^{(\epsilon)} \ = \ \lim_{n} \BSi_{11}^n \ , \quad
\BSi_{12}^{(\epsilon)} \ = \ \lim_{n} \BSi_{12}^n \ , \quad
\BSi_{22}^{(\epsilon)} \ = \ \lim_{n} \BSi_{22}^n \ .
\end{align}
exist. Furthermore, the matrix
\begin{align}
\BSi^{(0)} \ &= \
\begin{pmatrix}
\BSi_{11}^{(0)} & \BSi_{12}^{(0)} \\
\BSi_{21}^{(0)} & \BSi_{22}^{(0)}
\end{pmatrix}
\end{align}
is the unique solution to the following system of equations
\begin{align}
&\BDe \ \BSi_{11}^{(0)} \ + \ \BSi_{11}^{(0)} \ \BDe^{T} \ - \ \bar{a}
  \ \BSi_{11}^{(0)} \ + \ \BA_{12} \   \BSi_{21}^{(0)}
 \ + \  \BSi_{12}^{(0)} \ \BA_{12}^{T} \ = \ \BGa_{11} \ , \\
& \BA_{12} \   \BSi_{22}^{(0)}
 \ + \  \BSi_{12}^{(0)} \ \BA_{22}^{T} \ = \ \BGa_{12} \ , \\
& \BA_{22} \   \BSi_{22}^{(0)}
 \ + \  \BSi_{22}^{(0)} \ \BA_{22}^{T} \ = \ \BGa_{22} \ .
\end{align}
Finally,
\begin{align}
& \lim_{\epsilon \downarrow 0}  \BSi_{11}^{(\epsilon)} \ = \
  \BSi_{11}^{(0)} \ , \quad
\lim_{\epsilon \downarrow 0}  \BSi_{12}^{(\epsilon)} \ = \
  \BSi_{12}^{(0)} \ , \quad
\lim_{\epsilon \downarrow 0}  \BSi_{22}^{(\epsilon)} \ = \
  \BSi_{22}^{(0)} \ .
\end{align}

\end{theorem}

The next theorems shows that the asymptotic covariance matrix of
$a(n)^{-1/2}\Bth_n$ is the same as
that of $a(n)^{-1/2}\bar{\Bth}_n$, where
$\bar{\Bth}_n$
evolves according to the single time-scale stochastic
iteration:
\begin{align}
\bar{\Bth}_{n+1} \ &= \ \bar{\Bth}_n \ + \ a(n) \ \left(\Ba_1 \ - \ \BA_{11} \
               \bar{\Bth}_n \ - \ \BA_{12} \ \bar{\Bw}_n \ + \ \BM^{(\theta)}_{n}\right) \ ,\\
\BZe\ &= \ \Ba_2  \ - \ \BA_{21} \
               \bar{\Bth}_n \ - \ \BA_{22} \ \bar{\Bw}_n \ + \ \BM^{(w)}_{n}\ .
\end{align}

The next theorem combines Theorem 2.8 of Konda \& Tsitsiklis
and Theorem 4.1 of Konda \& Tsitsiklis:
\begin{theorem}[Konda \& Tsitsiklis 2nd]
\label{th:kondaRate2}
Under above assumptions
\begin{align}
\BSi_{11}^{(0)} \ &= \ \lim_{n} a(n)^{-1} \
\rE\left[ \bar{\Bth}_n  \bar{\Bth}_n^T \right] \ .
\end{align}

If the assumptions hold with $\epsilon=0$, then
$a(n)^{-1/2}\hat{\Bth}_n$ converges in distribution to
$\cN(\BZe,\BSi_{11}^{(0)})$.
\end{theorem}

\paragraph{Comments.}
\begin{enumerate}[label=\textbf{(C\arabic*)}]
\item
In his PhD thesis \cite{Konda:02} Konda extended the analysis to the
nonlinear case.
Konda makes a linearization of the nonlinear function $\Bh$ and $\Bg$ with
\begin{align}
& \BA_{11} \ = \ - \ \frac{\partial \Bh}{\partial \Bth} \ , \ \
\BA_{12} \ = \ - \ \frac{\partial \Bh}{\partial \Bw}  \  , \ \
\BA_{21} \ = \ - \ \frac{\partial \Bg}{\partial \Bth}  \ , \ \
\BA_{22} \ = \ - \ \frac{\partial \Bg}{\partial \Bw} \ .
\end{align}
There are additional errors due to linearization which have to be
considered.
However, only a sketch of a proof is provided but not a complete proof.

\item
Theorem 4.1 of Konda \& Tsitsiklis is important to generalize to the
nonlinear case.

\item
The convergence rate is governed by
$\BA_{22}$ for the fast and $\BDe$
for the slow iterate.
$\BDe$ in turn is affected by the interaction effects
captured by $\BA_{21}$ and $\BA_{12}$ together with the inverse of
$\BA_{22}$.

\end{enumerate}

\subsubsection{Nonlinear Update Rules}
\label{sec:nonlinur}
The rate of convergence for nonlinear update rules according to
Mokkadem \& Pelletier is considered \cite{Mokkadem:06}.

The iterates are
\begin{align}
\label{eq:iter1Mokkadem}
\Bth_{n+1} \ &= \ \Bth_n \ + \ a(n) \ \left(\Bh\big(\Bth_n, \Bw_n\big) \ + \
               \BZ^{(\theta)}_n \ + \ \BM^{(\theta)}_{n}\right) \ ,\\
\label{eq:iter2Mokkadem}
\Bw_{n+1} \ &= \ \Bw_n  \ + \ b(n)\ \left(\Bg\big(\Bth_n, \Bw_n\big) \ + \
              \BZ^{(w)}_n \ + \ \BM^{(w)}_{n}\right) \ .
\end{align}
with the increasing $\sigma$-fields
\begin{align}
\cF_n \ &= \ \sigma(\Bth_l, \Bw_l, \BM^{(\theta)}_{l}, \BM^{(w)}_{l},
          \BZ^{(\theta)}_{l}, \BZ^{(w)}_{l}, l \leq n), \ n \geq 0 \ .
\end{align}
The terms $\BZ^{(\theta)}_n$ and $\BZ^{(w)}_n$ can be used to address
the error through linearization, that is, the difference of the
nonlinear functions to their linear approximation.

\paragraph{Assumptions.}
We make the following assumptions:
\begin{enumerate}[label=\textbf{(A\arabic*)}]
\item Convergence is ensured:
\begin{align}
\lim_{n \to \infty} \Bth_n \ &= \ \Bth^{*} \ \ \text{a.s.} \ , \\
\lim_{n \to \infty} \Bw_n \ &= \ \Bw^{*} \ \ \text{a.s.} \ .
\end{align}

\item Linear approximation and Hurwitz:

There exists a neighborhood $\cU$ of $(\Bth^{*},\Bw^{*})$ such that,
for all $(\Bth,\Bw)\in \cU$
\begin{align}
\begin{pmatrix}
\Bh\big(\Bth, \Bw\big) \\
\Bg\big(\Bth, \Bw\big)
\end{pmatrix}
\ &= \
\begin{pmatrix}
\BA_{11} & \BA_{12} \\
\BA_{21} & \BA_{22}
\end{pmatrix} \
\begin{pmatrix}
\Bth \ - \ \Bth^{*} \\
\Bw \ - \ \Bw^{*}
\end{pmatrix}
\ + \
\rO \left(
\begin{Vmatrix}
\Bth \ - \ \Bth^{*} \\
\Bw \ - \ \Bw^{*}
\end{Vmatrix}^2
\right) \ .
\end{align}

We define
\begin{align}
\BDe \ &:= \ \BA_{11} \ - \ \BA_{12} \BA_{22}^{-1} \BA_{21} \ .
\end{align}
A matrix is {\em Hurwitz} if the real part of each eigenvalue is strictly
negative.
We assume that the matrices $\BA_{22}$ and $\BDe$ are Hurwitz.

\item Assumptions on the learning rates:

\begin{align}
a(n) \ &= \ a_0 \ n^{-\alpha} \\
b(n) \ &= \ b_0 \ n^{-\beta} \ ,
\end{align}
where $a_0>0$ and $b_0>0$ and $1/2 < \beta < \alpha \leq 1$.
If $\alpha=1$, then $a_0>1/(2 e_{\mathrm{min}})$ with $e_{\mathrm{min}}$
as the absolute value of the largest eigenvalue of $\BDe$ (the
eigenvalue closest to 0).

\item Assumptions on the noise and error:
\begin{enumerate}
\item martingale difference sequences:
\begin{align}
\rE \left[\BM^{(\theta)}_{n+1} \mid  \cF_n\right]
\ &= \ 0 \ \ \text{a.s.} \ , \\
\rE \left[\BM^{(w)}_{n+1} \mid  \cF_n\right]
\ &= \ 0 \ \ \text{a.s.} \ .
\end{align}
\item existing second moments:
\begin{align}
\lim_{n \to \infty} \rE \left[
\begin{pmatrix}
\BM^{(\theta)}_{n+1} \\
\BM^{(w)}_{n+1}
\end{pmatrix} \
\begin{pmatrix}
(\BM^{(\theta)}_{n+1})^T &
(\BM^{(w)}_{n+1})^T
\end{pmatrix}
 \mid  \cF_n\right] \ = \
\BGa \ = \
\begin{pmatrix}
\BGa_{11} & \BGa_{12} \\
\BGa_{21} & \BGa_{22}
\end{pmatrix} \ \
 \text{a.s.}
\end{align}

\item bounded moments:

There exist $l>2/\beta$ such that
\begin{align}
\sup_n \rE \left[ \| \BM^{(\theta)}_{n+1}\|^l \mid  \cF_n\right]
\ &< \ \infty \ \ \text{a.s.} \ , \\
\sup_n \rE \left[ \| \BM^{(w)}_{n+1}\|^l \mid  \cF_n\right]
\ &< \ \infty \ \ \text{a.s.}
\end{align}

\item bounded error:
\begin{align}
\BZ^{(\theta)}_{n} \ &= \ r^{(\theta)}_n \ + \
\rO\big( \| \Bth \ - \ \Bth^{*} \|^2 \ + \
\| \Bw \ - \  \Bw^{*}\|^2 \big) \ , \\
\BZ^{(w)}_{n} \ &= \ r^{(w)}_n \ + \
\rO\big( \| \Bth \ - \ \Bth^{*} \|^2 \ + \
\| \Bw \ - \  \Bw^{*}\|^2 \big) \ ,
\end{align}
with
\begin{align}
\| r^{(\theta)}_n\| \ +  \ \| r^{(w)}_n \| \ = \ \Ro(\sqrt{a(n)}) \ \ \text{a.s.}
\end{align}
\end{enumerate}

\end{enumerate}

\paragraph{Rate of Convergence Theorem.}

We report a theorem and a proposition from
Mokkadem \& Pelletier \cite{Mokkadem:06}.
However, first we have to define the covariance matrices
$\BSi_{\theta}$ and $\BSi_{w}$ which govern the rate of convergence.

First we define
\begin{align}
& \BGa_{\theta} \ := \
\lim_{n \to \infty} \rE \left[
\left( \BM^{(\theta)}_{n+1} \ - \ \BA_{12} \ \BA_{22}^{-1} \
\BM^{(w)}_{n+1} \right) \
\left( \BM^{(\theta)}_{n+1} \ - \ \BA_{12} \ \BA_{22}^{-1} \
\BM^{(w)}_{n+1} \right)^T
 \mid  \cF_n\right] \ = \\\nonumber
&\BGa_{11} \ + \  \BA_{12} \ \BA_{22}^{-1} \ \BGa_{22} \
  (\BA_{22}^{-1})^T \ \BA_{12}^T \ - \ \BGa_{12} (\BA_{22}^{-1})^T \
  \BA_{12}^T \ - \  \BA_{12} \ \BA_{22}^{-1} \ \BGa_{21} \ .
\end{align}
We now define the asymptotic covariance matrices $\BSi_{\theta}$ and
$\BSi_{w}$:
\begin{align}
\BSi_{\theta} \ &= \ \int_{0}^{\infty}
\exp \left(  \left(\BDe \ + \ \frac{\mathbbm{1}_{a=1}}{2 \ a_0} \ \BI
\right) \ t \right) \
\BGa_{\theta} \
\exp \left( \left(\BDe^T \ + \ \frac{\mathbbm{1}_{a=1}}{2 \ a_0} \ \BI
\right) \ t \right) \ \Rd t \ , \\
\BSi_{w} \ &= \ \int_{0}^{\infty}
\exp \left(   \BA_{22} \ t \right) \
\BGa_{22} \
\exp \left( \BA_{22} \ t \right) \ \Rd t \ .
\end{align}

$\BSi_{\theta}$ and $\BSi_{w}$ are solutions of the Lyapunov
equations:
\begin{align}
 \left(\BDe \ + \ \frac{\mathbbm{1}_{a=1}}{2 \ a_0} \ \BI
\right) \ \BSi_{\theta} \ + \ \BSi_{\theta} \ \left(\BDe^T \ + \
  \frac{\mathbbm{1}_{a=1}}{2 \ a_0} \ \BI \right)
\ &= \ - \ \BGa_{\theta} \ , \\
\BA_{22} \ \BSi_{w} \ + \ \BSi_{w}\ \BA_{22}^T \ &= \ - \  \BGa_{22} \ .
\end{align}

\begin{theorem}[Mokkadem \& Pelletier: Joint weak convergence]
\label{th:Mokkadem}
Under above assumptions:
\begin{align}
\begin{pmatrix}
\sqrt{a(n)^{-1}} \ (\Bth \ - \ \Bth^{*}) \\
\sqrt{b(n)^{-1}} \ (\Bw \ - \ \Bw^{*})
\end{pmatrix}
\ &\xrightarrow{\cD} \
\cN \left( \BZe \ , \
\begin{pmatrix}
\BSi_{\theta} & \BZe \\
\BZe & \BSi_{w}
\end{pmatrix}
\right) \ .
\end{align}
\end{theorem}

\begin{theorem}[Mokkadem \& Pelletier: Strong convergence]
\label{th:Mokkadem1}
Under above assumptions:
\begin{align}
\| \Bth \ - \ \Bth^{*} \| \ &= \
\rO \left(
\sqrt{a(n) \ \log \left( \sum_{l=1}^{n} a(l) \right) }
\right) \ \  \text{a.s.}  \ , \\
\| \Bw \ - \ \Bw^{*} \| \ &= \
\rO \left(
\sqrt{b(n) \ \log \left( \sum_{l=1}^{n} b(l) \right) }
\right) \ \  \text{a.s.}
\end{align}
\end{theorem}

\paragraph{Comments.}
\begin{enumerate}[label=\textbf{(C\arabic*)}]
\item
Besides the learning steps $a(n)$ and $b(n)$,
the convergence rate is governed by
$\BA_{22}$ for the fast and $\BDe$
for the slow iterate.
$\BDe$ in turn is affected by interaction effects which are
captured by $\BA_{21}$ and $\BA_{12}$ together with the inverse of
$\BA_{22}$.

\end{enumerate}

\subsection{Equal Time-Scale Stochastic Approximation Algorithms}
\label{sec:equalTime}

In this subsection we consider the case when the learning rates have
equal time-scale.

\subsubsection{Equal Time-Scale for Saddle Point Iterates}

If equal time-scales assumed then the iterates revisit infinite often
an environment of the solution \cite{Zhang:07}.
In Zhang 2007, the functions of the
iterates are the derivatives of a Lagrangian with respect to the dual
and primal variables \cite{Zhang:07}.
The iterates are
\begin{align}
\label{eq:iter1Zhang}
\Bth_{n+1} \ &= \ \Bth_n \ + \ a(n) \ \left(\Bh\big(\Bth_n, \Bw_n\big) \ + \
               \BZ^{(\theta)}_n \ + \ \BM^{(\theta)}_{n}\right) \ ,\\
\label{eq:iter2Zhang}
\Bw_{n+1} \ &= \ \Bw_n  \ + \ a(n)\ \left(\Bg\big(\Bth_n, \Bw_n\big) \ + \
              \BZ^{(w)}_n \ + \ \BM^{(w)}_{n}\right) \ .
\end{align}
with the increasing $\sigma$-fields
\begin{align}
\cF_n \ &= \ \sigma(\Bth_l, \Bw_l, \BM^{(\theta)}_{l}, \BM^{(w)}_{l},
          \BZ^{(\theta)}_{l}, \BZ^{(w)}_{l}, l \leq n), \ n \geq 0 \ .
\end{align}
The terms $\BZ^{(\theta)}_n$ and $\BZ^{(w)}_n$ subsum biased estimation
errors.

\paragraph{Assumptions.}
We make the following assumptions:
\begin{enumerate}[label=\textbf{(A\arabic*)}]
\item Assumptions on update function:
$\Bh$ and $\Bg$ are continuous, differentiable, and bounded.
The Jacobians
\begin{align}
\frac{\partial \Bg}{\partial \Bw} \quad \text{and} \quad \frac{\partial \Bh}{\partial \Bth}
\end{align}
are Hurwitz.
A matrix is {\em Hurwitz} if the real part of each eigenvalue is strictly
negative.
This assumptions corresponds to the assumption in \cite{Zhang:07}
that the Lagrangian is concave in $\Bw$ and convex in $\Bth$.

\item Assumptions on noise:

$\{\BM^{(\theta)}_{n}\}$ and $\{\BM^{(w)}_{n}\}$
are a martingale difference sequences
w.r.t.\ the increasing $\sigma$-fields $\cF_n$.
Furthermore they are mutually independent.

Bounded second moment:
\begin{align}
\rE \left[\|\BM^{(\theta)}_{n+1}\|^2 \mid \cF_n \right]
\ &< \ \infty \ \ \text{a.s.} \ , \\
\rE \left[\|\BM^{(w)}_{n+1}\|^2 \mid \cF_n \right]
\ &< \ \infty \ \ \text{a.s.} \ .
\end{align}

\item Assumptions on the learning rate:
\begin{align}
&a(n)\ > \ 0 \quad , \quad a(n)\ \to \ 0 \quad , \quad \sum_{n} a(n) \ = \ \infty \quad , \quad
\sum_{n} a^2(n) \ < \ \infty \ .
\end{align}

\item Assumption on the biased error:

Boundedness:
\begin{align}
\lim_n \sup  \| \BZ^{(\theta)}_n \| \  &\leq  \ \alpha^{(\theta)} \ \text{a.s.} \\
\lim_n \sup \| \BZ^{(w)}_n \| \  &\leq  \ \alpha^{(w)} \ \text{a.s.}
\end{align}

\end{enumerate}

\paragraph{Theorem.}

Define the ``contraction region'' $A_{\eta}$ as follows:
\begin{align}
A_{\eta} \ &= \ \{(\Bth,\Bw): \alpha^{(\theta)} \geq \eta \
             \|\Bh(\Bth,\Bw)\| \quad \text{or} \quad  \alpha^{(w)} \geq \eta
             \ \|\Bg(\Bth,\Bw)\|, \ 0 \leq \eta < 1 \} \ .
\end{align}

\begin{theorem}[Zhang]
\label{th:Zhang}
Under above assumptions the iterates return to $A_{\eta}$
infinitely often with probability one (a.s.).
\end{theorem}

\paragraph{Comments.}
\begin{enumerate}[label=\textbf{(C\arabic*)}]
\item
The proof of the theorem in \cite{Zhang:07} does not use the
saddle point condition and not the fact that the functions of the
iterates are derivatives of the same function.

\item
For the unbiased case, Zhang showed in Theorem~3.1 of \cite{Zhang:07}
that the iterates converge.
However, he used the saddle point condition of the Lagrangian.
He considered
iterates with functions that are the derivatives of a Lagrangian
with respect to the dual
and primal variables \cite{Zhang:07}.

\end{enumerate}

\subsubsection{Equal Time Step for Actor-Critic Method}
If equal time-scales assumed then the iterates revisit infinite often
an environment of the solution of DiCastro \& Meir \cite{DiCastro:10}.
The iterates of DiCastro \& Meir are derived for actor-critic
learning.

To present the actor-critic update iterates,
we have to define some functions and terms.
$\Bmu(\Bu \mid \Bx, \Bth)$ is the policy function parametrized by
$\Bth \in \dR^m$ with observations $\Bx \in \cX$ and actions $\Bu \in \cU$.
A Markov chain given by $\rP(\By \mid \Bx, \Bu)$ gives the next
observation $\By$ using the observation $\Bx$ and the action $\Bu$.
In each state $\Bx$ the agent receives a reward $r(\Bx)$.

The average reward per stage is for the recurrent state $\Bx^{*}$:
\begin{align}
\tilde{\eta}(\Bth) \ &= \ \lim_{T \to \infty} \rE \left[
\frac{1}{T} \sum_{n=0}^{T-1} r(\Bx_n) \mid \Bx_0=\Bx^{*},\Bth
\right] \ .
\end{align}
The estimate of $\tilde{\eta}$ is denoted by $\eta$.

The differential value function is
\begin{align}
\tilde{h}(\Bx,\Bth) \ &= \ \rE \left[
 \sum_{n=0}^{T-1} ( r(\Bx_n) \ - \ \tilde{\eta}(\Bth) )
\mid \Bx_0=\Bx,\Bth
\right] \ .
\end{align}

The temporal difference is
\begin{align}
\tilde{d}(\Bx,\By,\Bth) \ &= \ r(\Bx) \ - \ \tilde{\eta}(\Bth)
\ + \ \tilde{h}(\By,\Bth) \ - \ \tilde{h}(\Bx,\Bth) \ .
\end{align}
The estimate of $\tilde{d}$ is denoted by $d$.

The likelihood ratio derivative $\BPs \in \dR^m$ is
\begin{align}
\BPs(\Bx,\Bu,\Bth) \ &= \frac{\nabla_{\theta} \Bmu(\Bu \mid \Bx, \Bth)}{\Bmu(\Bu \mid \Bx, \Bth)} \ .
\end{align}

The value function $\tilde{h}$ is approximated by
\begin{align}
h(\Bx,\Bw) \ &= \ \Bph(\Bx)^T \ \Bw \ ,
\end{align}
where $\Bph(\Bx) \in \dR^k$.
We define $\BPh \in \dR^{|\cX| \times k}$
\begin{align}
\BPh \ &= \
\begin{pmatrix}
\Bph_1(\Bx_1) & \Bph_2(\Bx_1) & \ldots & \Bph_k(\Bx_1) \\
\Bph_1(\Bx_2) & \Bph_2(\Bx_2) & \ldots & \Bph_k(\Bx_2) \\
\vdots & \vdots & & \vdots & \\
\Bph_1(\Bx_{|\cX|}) & \Bph_2(\Bx_{|\cX|}) & \ldots & \Bph_k(\Bx_{|\cX|})
\end{pmatrix}
\end{align}
and
\begin{align}
h(\Bw) \ &= \ \BPh \ \Bw \ .
\end{align}

For TD($\lambda$) we have an eligibility trace:
\begin{align}
e_n \ &= \ \lambda \ e_{n-1} \ + \  \Bph(\Bx_n) \ .
\end{align}

We define the approximation error with optimal parameter $\Bw^{*}(\Bth)$:
\begin{align}
\epsilon_{\mathrm{app}}(\Bth) \ &= \ \inf_{\Bw \in \dR^k} \| \tilde{h}(\Bth)
\ - \ \BPh \ \Bw \|_{\pi(\Bth)} \ = \ \| \tilde{h}(\Bth)
\ - \ \BPh \ \Bw^{*}(\Bth) \|_{\pi(\Bth)} \ ,
\end{align}
where $\pi(\Bth)$ is an projection operator into the span of
$\BPh \Bw$.
We bound this error by
\begin{align}
\epsilon_{\mathrm{app}} \ &= \ \sup_{\Bth \in \dR^k}
 \epsilon_{\mathrm{app}}(\Bth) \ .
\end{align}

We denoted by $\tilde{\eta}$, $\tilde{d}$, and $\tilde{h}$ the exact
functions and used for their approximation $\eta$, $d$, and $h$,
respectively.
We have learning rate adjustments $\Gamma_{\eta}$ and $\Gamma_{w}$ for the
critic.

The update rules are:\newline
{\bf Critic:}
\begin{align}
\eta_{n+1} \ &= \ \eta_{n} \ + \ a(n) \ \Gamma_{\eta} \
(r(\Bx_n) \ - \ \eta_n) \ , \\
h(\Bx,\Bw_n) \ &= \ \Bph(\Bx)^T \ \Bw_n \ , \\
d(\Bx_n,\Bx_{n+1},\Bw_n) \ &= \
r(\Bx_n) \ - \ \eta_n
\ + \ h(\Bx_{n+1},\Bw_n) \ - \ h(\Bx_n,\Bw_n) \ , \\
e_n \ &= \ \lambda \ e_{n-1} \ + \  \Bph(\Bx_n) \ , \\
\Bw_{n+1} \ &= \ \Bw_{n} \ + \ a(n) \ \Gamma_{w} \
d(\Bx_n,\Bx_{n+1},\Bw_n) \ e_n \ .
\end{align}
{\bf Actor:}
\begin{align}
\Bth_{n+1} \ &= \ \Bth_{n} \ + \ a(n) \ \BPs(\Bx_n,\Bu_n,\Bth_n) \
d(\Bx_n,\Bx_{n+1},\Bw_n) \ .
\end{align}

\paragraph{Assumptions.}
We make the following assumptions:
\begin{enumerate}[label=\textbf{(A\arabic*)}]
\item Assumption on rewards:

The rewards $\{ r(\Bx)\}_{\Bx \in \cX}$ are uniformly bounded by a
finite constant $B_r$.

\item Assumption on the Markov chain:

Each Markov chain for each $\Bth$ is aperiodic, recurrent, and
irreducible.

\item Assumptions on the policy function:

The conditional probability function $\Bmu(\Bu \mid \Bx, \Bth)$
is twice differentiable. Moreover, there exist positive constants,
$B_{\mu_1}$ and $B_{\mu_2}$, such that for all $\Bx \in \cX$,
$\Bu  \in \cU$, $\Bth \in \dR^m$ and $1 \leq l_1,l_2 \leq m$ we
have
\begin{align}
&\left\| \frac{\partial \Bmu(\Bu \mid \Bx, \Bth)}{\partial
  \Bth_l}\right\| \ \leq \ B_{\mu_1} \ , \quad
\left\| \frac{\partial^2 \Bmu(\Bu \mid \Bx, \Bth)}{\partial
  \Bth_{l_1} \ \partial \Bth_{l_2}}\right\| \ \leq \ B_{\mu_2} \ .
\end{align}

\item Assumption on the likelihood ratio derivative:

For all $\Bx \in \cX$,
$\Bu \in \cU$, and $\Bth \in \dR^m$, there exists a positive
constant $B_{\Psi}$, such that
\begin{align}
\| \BPs(\Bx,\Bu,\Bth) \|_2 \ &\leq \ B_{\Psi} \ < \ \infty \ ,
\end{align}
where $\| . \|_2$ is the Euclidean $L_2$ norm.

\item Assumptions on the approximation space given by $\BPh$:

The columns of the matrix $\BPh$ are independent, that is, the form a
basis of dimension $k$.
The norms of the columns vectors of the matrix $\BPh$ are bounded
above by $1$, that is, $\| \Bph_l \|_2 \leq 1$ for $1 \leq l \leq k$.

\item Assumptions on the learning rate:
\begin{align}
&\sum_{n} a(n) \ = \ \infty \quad , \quad
\sum_{n} a^2(n) \ < \ \infty \ .
\end{align}

\end{enumerate}

\paragraph{Theorem.}

The algorithm converged if $\nabla_{\theta} \tilde{\eta}(\Bth) =
\BZe$, since the actor reached a stationary point where the updates
are zero.
We assume that $\| \nabla_{\theta} \tilde{\eta}(\Bth)\|$
hints at how close we are to the convergence point.

The next theorem from DiCastro \& Meir \cite{DiCastro:10}
implies that the trajectory visits a neighborhood of
a local maximum infinitely often.
Although it may leave the local vicinity of the maximum, it is
guaranteed to return to it infinitely often.

\begin{theorem}[DiCastro \& Meir]
\label{th:DiCastro}
Define
\begin{align}
B_{\nabla \tilde{\eta}} \ &= \ \frac{B_{\Delta t  d 1}}{\Gamma_{w}}
\ + \ \frac{B_{\Delta t  d 2}}{\Gamma_{\eta}} \ + \
B_{\Delta t  d 3} \ \epsilon_{\mathrm{app}} \ ,
\end{align}
where $B_{\Delta t  d 1}$, $B_{\Delta t  d 2}$, and
$B_{\Delta t  d 3}$ are finite constants depending on the Markov
decision process and the agent parameters.

Under above assumptions
\begin{align}
\lim_{t \to \infty} \inf \ \| \nabla_{\theta} \tilde{\eta}(\Bth_t)\|
\ &\leq \ B_{\nabla \tilde{\eta}} \ .
\end{align}
The trajectory visits a neighborhood of
a local maximum infinitely often.
\end{theorem}

\paragraph{Comments.}
\begin{enumerate}[label=\textbf{(C\arabic*)}]
\item
The larger the critic learning rates
$\Gamma_{w}$ and $\Gamma_{\eta}$ are,
the smaller is the region around the local maximum.

\item
The results are in agreement with those of Zhang 2007 \cite{Zhang:07}.

\item
Even if the results are derived for a special actor-critic setting,
they carry over to a more general setting of the iterates.

\end{enumerate}

\section{ADAM Optimization as Stochastic Heavy Ball with Friction}
\label{sec:adam}

The Nesterov Accelerated Gradient
Descent (NAGD) \cite{Nesterov:83} has raised considerable interest due to its numerical
simplicity and its low complexity. Previous to NAGD and its derived
methods there was Polyak's Heavy Ball method \cite{Polyak:64}.
The idea of the Heavy Ball is a ball that evolves over the graph of a function $f$
with damping (due to friction) and acceleration. Therefore, this
second-order dynamical system can be described by the ODE for the
Heavy Ball with Friction (HBF) \cite{Gadat:16}:
\begin{align}
&\ddot{\Bth}_t \ + \ a(t) \ \dot{\Bth}_t \ + \ \nabla f(\Bth_t)
\ = \ \BZe  \ ,
\end{align}
where $a(n)$ is the damping coefficient with
$a(n)=\frac{a}{n^{\beta}}$ for $\beta \in (0,1]$.
This ODE is equivalent to the integro-differential equation
\begin{align}
&\dot{\Bth}_t \ = \ - \ \frac{1}{k(t)} \ \int_{0}^{t} h(s) \nabla
  f(\Bth_s) \Rd s \ ,
\end{align}
where $k$ and $h$ are two memory functions related to $a(t)$.
For polynomially memoried HBF we have $k(t)=t^{\alpha+1}$ and $h(t)=(\alpha +1 ) t^{\alpha}$ for some positive
$\alpha$, and for exponentially
memoried HBF we have $k(t)=\lambda \exp(\lambda \ t)$ and
$h(t)=\exp(\lambda \ t)$.
For the sum of the learning rates, we obtain
\begin{align}
\sum_{l=1}^{n} a(l) \ &= \
a \
\begin{cases}
\ln(n) \ + \ \gamma \ + \ \frac{1}{2n}
\ + \ \rO\big(\frac{1}{n^2} \big) & \mbox{ for } \beta=1 \\
\frac{n^{1-\beta}}{1-\beta} & \mbox{ for } \beta<1
\end{cases} \ ,
\end{align}
where $\gamma=0.5772156649$ is the Euler-Mascheroni constant.

Gadat et al.\ derived a discrete and stochastic version of the HBF \cite{Gadat:16}:
\begin{align}
\label{eq:hbf}
\Bth_{n+1} \ &= \ \Bth_n \ - \ a(n+1) \ \Bm_n \\ \nonumber
\Bm_{n+1} \ &= \ \Bm_n \ + \  a(n+1) \ r(n) \
\big( \nabla f(\Bth_n) \ - \ \Bm_n \big) \ + \
a(n+1) \ r(n) \ \BM_{n+1} \ ,
\end{align}
where
\begin{align}
r(n) \ &= \
\begin{cases}
r  & \mbox{ for exponentially memoried HBF} \\
\frac{r}{\sum_{l=1}^{n} a(l) } & \mbox{ for polynomially memoried HBF}
\end{cases} \ .
\end{align}

This recursion can be rewritten as
\begin{align}
\label{eq:hbfADAM}
\Bth_{n+1} \ &= \ \Bth_n \ - \ a(n+1) \ \Bm_n \\
\Bm_{n+1} \ &= \ \big( 1 \ - \  a(n+1) \ r(n) \big) \ \Bm_n \ + \  a(n+1) \ r(n) \
\big( \nabla f(\Bth_n) \ + \ \BM_{n+1} \big) \ .
\end{align}
The recursion Eq.~\eqref{eq:hbfADAM} is the first moment update of ADAM \cite{Kingma:14}.

For the term $r(n) a(n)$ we obtain for the polynomial memory the approximations
\begin{align}
r(n) \ a(n) \ &\approx \
r \ \begin{cases}
\frac{1}{n \ \log n} & \mbox{ for } \beta=1 \\
\frac{1\ - \ \beta}{n} & \mbox{ for } \beta<1
\end{cases} \ ,
\end{align}

Gadat et al.\ showed that the recursion
Eq.~\eqref{eq:hbf} converges
for functions with at most quadratic grow \cite{Gadat:16}.
The authors mention that convergence can be proofed
for functions $f$ that are $L$-smooth, that is, the gradient is $L$-Lipschitz.

Kingma et al.\ \cite{Kingma:14} state in Theorem 4.1
convergence of ADAM while assuming that $\beta_1$, the
first moment running average coefficient, decays exponentially.
Furthermore they assume that $\frac{\beta_1^2}{\sqrt{\beta_2}}<1$
and the learning rate $\alpha_t$ decays with
$\alpha_t=\frac{\alpha}{\sqrt{t}}$.

ADAM divides $\Bm_n$ of the recursion Eq.~\eqref{eq:hbfADAM}
by the bias-corrected second raw moment estimate.
Since the bias-corrected second raw moment estimate changes slowly,
we consider it as an error.

\begin{align}
\frac{1}{\sqrt{v+\Delta v}} \ &\approx \
\frac{1}{\sqrt{v}} \ - \ \frac{1}{2 \ v \ \sqrt{v}} \ \Delta v
\ + \ \rO(\Delta v^2) \ .
\end{align}

ADAM assumes the second moment
$\rE \left[g^2\right]$ to be stationary with its  approximation
$v_n$:
\begin{align}
v_n \ &= \ \frac{1 \ - \ \beta_2}{1 \ - \ \beta_2^n} \ \sum_{l=1}^{n} \beta_2^{n-l} \ g_l^2 \ .
\end{align}

\begin{align}
\Delta_n v_n \ &= \  v_n \ - \ v_{n-1} \ = \
\frac{1 \ - \ \beta_2}{1 \ - \ \beta_2^n} \
\sum_{l=1}^{n} \beta_2^{n-l} \ g_l^2 \ - \
\frac{1 \ - \ \beta_2}{1 \ - \ \beta_2^{n-1}} \
\sum_{l=1}^{n-1} \beta_2^{n-l-1} \ g_l^2 \\ \nonumber
&= \ \frac{1 \ - \ \beta_2}{1 \ - \ \beta_2^n} \  g_n^2 \ + \
\frac{\beta_2 \ (1 \ - \ \beta_2)}{1 \ - \ \beta_2^n} \
\sum_{l=1}^{n-1} \beta_2^{n-l-1} \ g_l^2 \ - \
\frac{1 \ - \ \beta_2}{1 \ - \ \beta_2^{n-1}} \
\sum_{l=1}^{n-1} \beta_2^{n-l-1} \ g_l^2 \\ \nonumber
&= \ \frac{1 \ - \ \beta_2}{1 \ - \ \beta_2^n} \  \left( g_n^2 \ + \
\big(\beta_2 \ - \  \frac{1 \ - \ \beta_2^n}{1 \ - \ \beta_2^{n-1}}
\big)  \ \sum_{l=1}^{n-1} \beta_2^{n-l-1} \ g_l^2 \right) \\ \nonumber
&= \ \frac{1 \ - \ \beta_2}{1 \ - \ \beta_2^n} \  \left( g_n^2 \ - \
 \frac{1 \ - \ \beta_2}{1 \ - \ \beta_2^{n-1}} \  \sum_{l=1}^{n-1}
  \beta_2^{n-l-1} \ g_l^2 \right) \ .
\end{align}

Therefore
\begin{align}
\rE \left[\Delta_n v_n \right] \ &= \ \rE \left[ v_n \ - \ v_{n-1} \right] \ = \
 \frac{1 \ - \ \beta_2}{1 \ - \ \beta_2^n} \  \left( \rE \left[g^2\right] \ - \
 \frac{1 \ - \ \beta_2}{1 \ - \ \beta_2^{n-1}} \  \sum_{l=1}^{n-1}
  \beta_2^{n-l-1} \ \rE \left[g^2\right] \right) \\ \nonumber
&= \  \frac{1 \ - \ \beta_2}{1 \ - \ \beta_2^n} \
\left( \rE \left[g^2\right] \ - \  \rE \left[g^2\right]  \right)
\ = \ 0 \ .
\end{align}

We are interested in the difference of actual stochastic $v_n$ to the
true stationary $v$:
\begin{align}
\Delta v_n \ &= \  v_n \ - \ v \ = \
\frac{1 \ - \ \beta_2}{1 \ - \ \beta_2^n} \  \sum_{l=1}^{n}
        \beta_2^{n-l} \ \left( g_l^2 \ - \ v \right) \ .
\end{align}

For a stationary second moment of $\Bm_n$ and $\beta_2=1-\alpha a(n+1)r(n)$, we have
$\Delta v_n \propto a(n+1)r(n)$.
We use a linear approximation to ADAM's second moment normalization
$1/\sqrt{v+\Delta v_n} \approx 1/\sqrt{v} - 1/(2 v \sqrt{v})  \Delta v_n
+ \rO(\Delta^2 v_n)$.
If we set $\BM_{n+1}^{(v)}=- (\Bm_n \Delta v_n)/(2 v
\sqrt{v}a(n+1)r(n))$, then $\Bm_n / \sqrt{v_n} \approx \Bm_n / \sqrt{v} +
a(n+1)r(n)\BM_{n+1}^{(v)}$ and
$\rE \left[\BM_{n+1}^{(v)} \right] = 0$, since $\rE \left[g_l^2 - v \right] = 0$.
For a stationary second moment of $\Bm_n$, $\{\BM^{(v)}_n\}$ is a martingale
difference sequence with a bounded second moment.
Therefore $\{\BM^{(v)}_{n+1}\}$ can be subsumed into  $\{\BM_{n+1}\}$ in update
rules Eq.~\eqref{eq:hbfADAM}. The factor $1 / \sqrt{v}$ can be
incorporated into $a(n+1)$ and $r(n)$.

\section{Experiments: Additional Information}
\label{sec:apx_experiments}

\subsection{WGAN-GP on Image Data.}

\begin{table}[htp]
\begin{center}
\caption[Results WGAN-GP on Image Data]{The performance of WGAN-GP trained with
the original procedure and with TTUR on CIFAR-10 and LSUN Bedrooms. We compare
the performance with respect to the FID at the optimal number of
iterations during training and wall-clock time in minutes.}
\label{tab:wgan_gp_image}
\begin{tabular}{lllcccllccc}
 \toprule
 dataset & method & b, a & iter & time(m) & FID & method & b = a & iter & time(m) & FID \\
 \midrule
 CIFAR-10 & TTUR & 3e-4, 1e-4 & 168k & 700 &  {\bf 24.8} &
 orig & 1e-4 & 53k & 800 & 29.3 \\
 LSUN & TTUR & 3e-4, 1e-4 & 80k & 1900 &  {\bf 9.5} &
 orig & 1e-4 & 23k & 2010 & 20.5 \\
 \bottomrule
 \end{tabular}
 \end{center}
\end{table}

\subsection{WGAN-GP on the One Billion Word Benchmark.}

\begin{table}[H]
\centering
\caption[Samples of the One Billion Word benchmark generated by
WGAN-GP.]{Samples generated by WGAN-GP trained on fhe One Billion Word benchmark
with TTUR (left) the original method (right).
\label{tab:lang_samp}}
\begin{tabular}{ll}
\begin{minipage}{0.49\textwidth}
\begin{verbatim}

Dry Hall Sitning tven the concer
There are court phinchs hasffort
He scores a supponied foutver il
Bartfol reportings ane the depor
Seu hid , it 's watter 's remold
Later fasted the store the inste
Indiwezal deducated belenseous K
Starfers on Rbama 's all is lead
Inverdick oper , caldawho 's non
She said , five by theically rec
RichI , Learly said remain .````
Reforded live for they were like
The plane was git finally fuels
The skip lifely will neek by the
SEW McHardy Berfect was luadingu
But I pol rated Franclezt is the
\end{verbatim}
\end{minipage}
&
\begin{minipage}{0.49\textwidth}
\begin{verbatim}

No say that tent Franstal at Bra
Caulh Paphionars tven got corfle
Resumaly , braaky facting he at
On toipe also houd , aid of sole
When Barrysels commono toprel to
The Moster suprr tent Elay diccu
The new vebators are demases to
Many 's lore wockerssaow 2 2 ) A
Andly , has le wordd Uold steali
But be the firmoters is no 200 s
Jermueciored a noval wan 't mar
Onles that his boud-park , the g
ISLUN , The crather wilh a them
Fow 22o2 surgeedeto , theirestra
Make Sebages of intarmamates , a
Gullla " has cautaria Thoug ly t
\end{verbatim}
\end{minipage}
\end{tabular}
\end{table}

\begin{table}[htp]
\begin{center}
\caption[Results WGAN-GP on One Billion Word]{The performance of WGAN-GP trained
with the original procedure and with TTUR on the One Billion Word Benchmark. We
compare the performance with respect to the JSD at the optimal number
of iterations and wall-clock time in minutes during training.
WGAN-GP trained with TTUR exhibits consistently a better FID.}
\label{tab:wgan_gp_lang}
\begin{tabular}{lllcccllccc}
 \toprule
  n-gram & method & b, a & iter & time(m) & JSD &  method &
  b = a  & iter & time(m) & JSD \\
 \midrule
 4-gram & TTUR &   3e-4, 1e-4 & 98k & 1150 &  {\bf 0.35} &
 orig &   1e-4 & 33k & 1040 & 0.38   \\
 6-gram & TTUR &   3e-4, 1e-4 & 100k & 1120 &  {\bf 0.74} &
 orig &   1e-4 & 32k & 1070 & 0.77   \\
 \bottomrule
 \end{tabular}
 \end{center}
\end{table}

\subsection{BEGAN}
The Boundary Equilibrium GAN (BEGAN) \cite{Berthelot:17} maintains an
equilibrium between the discriminator and generator loss (cf. Section 3.3 in
\cite{Berthelot:17})
\begin{align}
\EXP[\mathcal{L}(G(\Bz))] = \gamma \EXP [\mathcal{L}(\Bx)]
\end{align}
which, in turn, also leads to a fixed relation
between the two gradients, therefore, a two time-scale update is not ensured by
solely adjusting the learning rates. Indeed, for stable learning rates, we see
no differences in the learning progress between orig and TTUR as depicted in Figure~\ref{fig:began}.

\begin{figure}[H]
\centering
\includegraphics[width=0.49\textwidth]{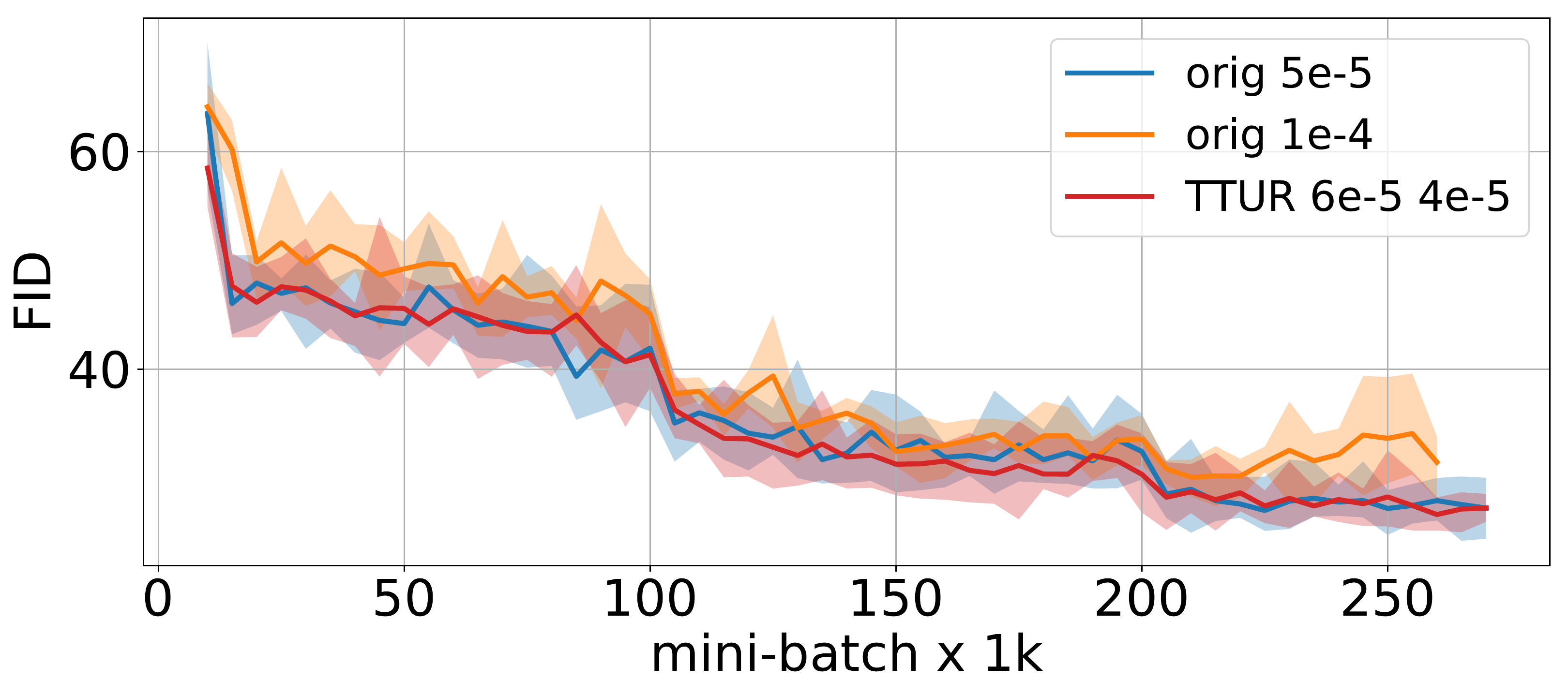}
\includegraphics[width=0.49\textwidth]{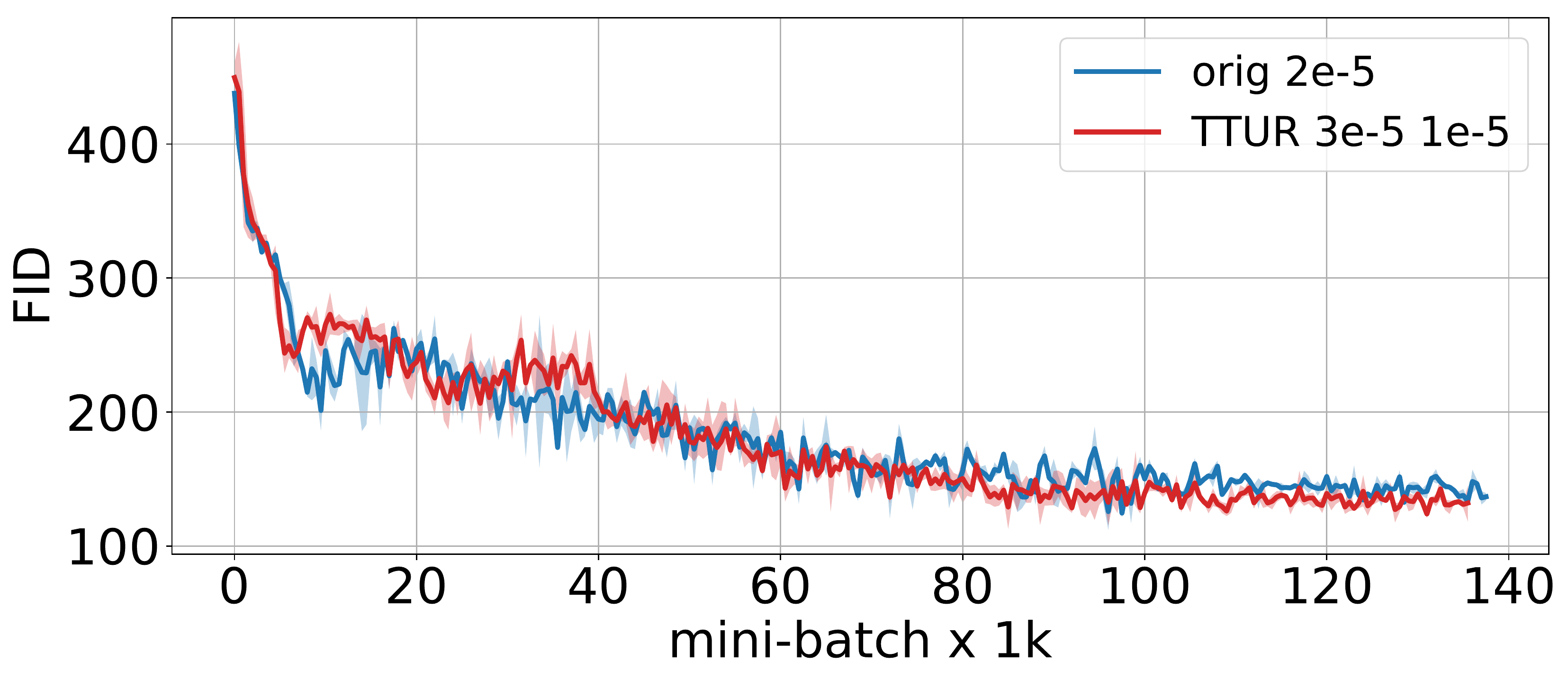} \caption[FID for
BEGAN trained on CelebA and LSUN Bedrooms.]{Mean, maximum and minimum FID over eight
runs for BEGAN training on CelebA and LSUN Bedrooms. TTUR learning rates are
given as pairs $(b,a)$ of discriminator learning rate $b$ and generator learning
rate $a$: ``TTUR $b$ $a$''.
{\bf Left:} CelebA, starting at mini-batch 10k for better visualisation. {\bf
Right:} LSUN Bedrooms.
Orig and TTUR behave similar. For BEGAN we cannot ensure TTUR by
adjusting learning rates.
  \label{fig:began} }
\end{figure}

\section{Discriminator vs. Generator Learning Rate}
\label{sec:lr}

The convergence proof for learning GANs with TTUR
assumes that the generator learning rate
will eventually become small enough to ensure
convergence of the discriminator learning.
At some time point, the perturbations of the discriminator updates
by updates of the generator parameters are sufficient small
to assure that the discriminator converges.
Crucial for discriminator convergence is the magnitude of the
perturbations which the generator induces into the
discriminator updates.
These perturbations are not only determined by the generator learning
rate but also by its loss function, current value of the loss
function, optimization method, size of the
error signals that reach the generator (vanishing or exploding
gradient), complexity of generator's learning task, architecture of
the generator, regularization, and others.
Consequently, the size of generator learning rate
does not solely determine how
large the perturbations of the discriminator updates are but serve to
modulate them.
Thus, the generator learning rate may be much larger than the
discriminator learning rate without inducing large perturbation into
the discriminator learning.

Even the learning dynamics of the generator is different from the
learning dynamics of the discriminator, though they both have the same
learning rate. Figure~\ref{fig:lr} shows the loss of the generator and
the discriminator for an experiment with DCGAN on
CelebA, where the learning rate was
0.0005 for both the discriminator and the generator.
However, the discriminator loss is decreasing while the generator loss
is increasing.
This example shows that the learning rate neither determines the
perturbations nor the progress in learning for two coupled update
rules.
The choice of the learning rate for the generator should be
independent from choice for the
discriminator.
Also the search ranges of discriminator and generator learning
rates should be independent from each other, but adjusted to
the corresponding architecture, task, etc.

\begin{figure}[H]
  \centering
  \includegraphics[width=1.0\linewidth]{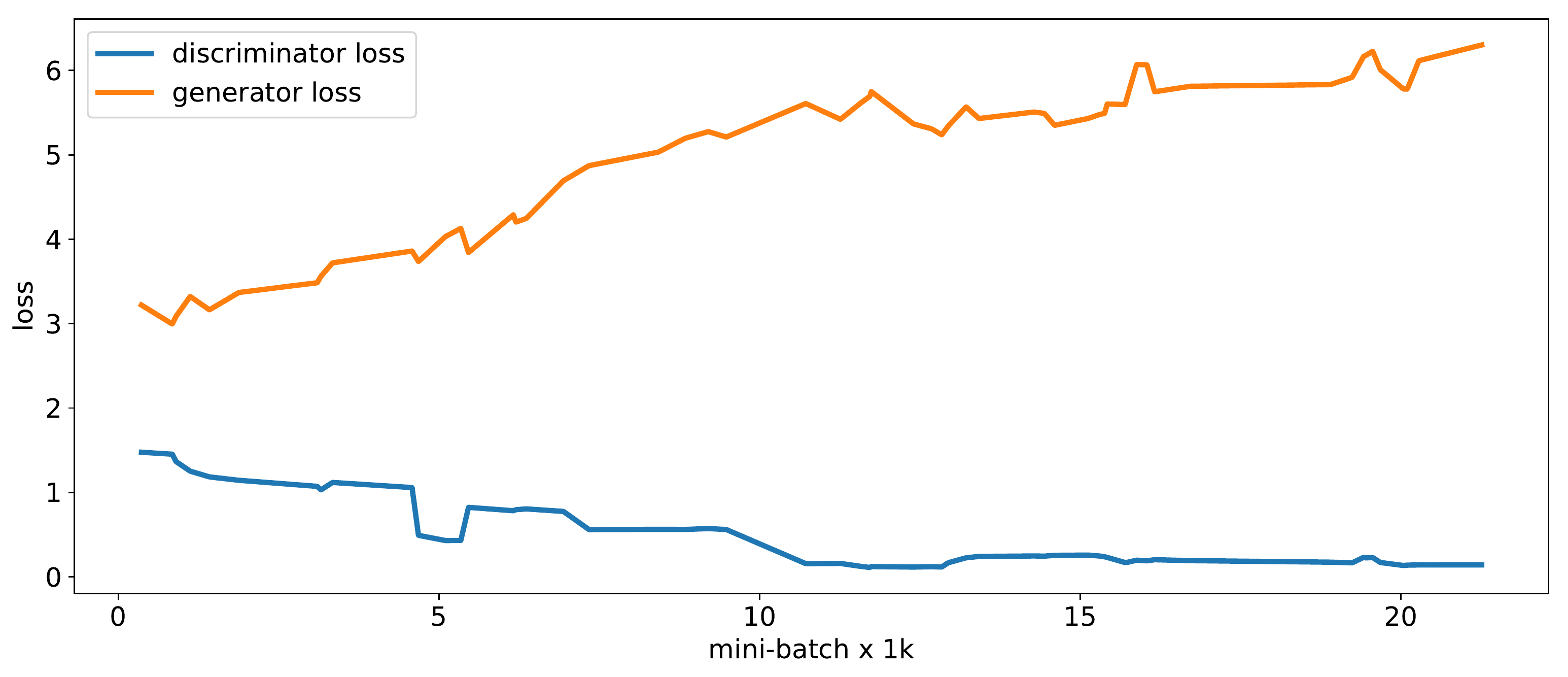}
  \caption[Learning dynamics of two networks.]{The respective losses of the
  discriminator and the generator show the different learning dynamics of the two networks.}
  \label{fig:lr}
\end{figure}%

\section{Used Software, Datasets, Pretrained Models, and Implementations}
\label{sec:soft}

We used the following datasets to evaluate GANs: The Large-scale CelebFaces
Attributes (CelebA) dataset, aligned and cropped \cite{Liu:15}, the training
dataset of the bedrooms category of the large scale image database (LSUN)
\cite{Yu:15}, the CIFAR-10 training dataset \cite{Krizhevsky:12}, the Street View
House Numbers training dataset (SVHN) \cite{Netzer:11}, and the One Billion Word
Benchmark \cite{Chelba:13}.

All experiments rely on the respective reference implementations for
the corresponding GAN model. The software framework for our experiments was
Tensorflow 1.3 \cite{Tensorflow:16, Tensorflow:16_2} and Python 3.6.
We used following software, datasets and pretrained models:
\begin{itemize}
  \item BEGAN in Tensorflow,
  \url{https://github.com/carpedm20/BEGAN-tensorflow}, Fixed random
  seeds removed. Accessed: 2017-05-30
  \item DCGAN in Tensorflow,
  \url{https://github.com/carpedm20/DCGAN-tensorflow}, Fixed random
  seeds removed. Accessed: 2017-04-03
  \item Improved Training of Wasserstein GANs, image model,
  \url{https://github.com/igul222/improved_wgan_training/blob/master/gan_64x64.py},
  Accessed: 2017-06-12
  \item Improved Training of Wasserstein GANs, language
  model, \url{https://github.com/igul222/improved_wgan_training/blob/master/gan_language.py}, Accessed: 2017-06-12
  \item Inception-v3 pretrained,
  \url{http://download.tensorflow.org/models/image/imagenet/inception-2015-12-05.tgz},
  Accessed: 2017-05-02
\end{itemize}

Implementations are available at
\begin{itemize}
  \item \url{https://github.com/bioinf-jku/TTUR}
\end{itemize}


\bibliography{twoTime}
\bibliographystyle{plain}
\label{sec:references}

\addcontentsline{toc}{section}{List of Figures}
\listoffigures
\addcontentsline{toc}{section}{List of Tables}
\listoftables

\end{document}

%% file: preamble.tex
\usepackage[utf8]{inputenc} 
\usepackage[T1]{fontenc}    
\usepackage{hyperref}       
\usepackage{url}            
\usepackage{booktabs}       
\usepackage{nicefrac}       
\usepackage{microtype}      

\usepackage{float}
\usepackage{graphicx}
\usepackage{tocstyle}

\usepackage{color}
\usepackage{amsmath}
\usepackage{amsfonts}
\usepackage{amssymb}
\usepackage{amsthm}
\usepackage{mathtools}
\usepackage{bm}
\usepackage{bbm}
\usepackage{relsize}
\usepackage{graphics}
\usepackage{subfigure}
\usepackage{wrapfig}
\usepackage{picinpar}
\usepackage{multirow}
\usepackage{enumitem} 

\usepackage{array}
\usepackage{ragged2e}

\hypersetup{
   colorlinks=true,
   linkcolor=blue,
   citecolor=green,
   urlcolor=magenta,
   pdfborder=0 0 0,
   pdfsubject=Two Time-Scale GANs
   pdfstartview=FitH
}

\newtheorem{theorem}{Theorem}

\newcommand\Ba{\bm{a}}

\newcommand\Bg{\bm{g}}
\newcommand\Bh{\bm{h}}

\newcommand\Bm{\bm{m}}

\newcommand\Bu{\bm{u}}
\newcommand\Bv{\bm{v}}
\newcommand\Bw{\bm{w}}
\newcommand\Bx{\bm{x}}
\newcommand\By{\bm{y}}
\newcommand\Bz{\bm{z}}
\newcommand\BA{\bm{A}}

\newcommand\BC{\bm{C}}

\newcommand\BG{\bm{G}}
\newcommand\BH{\bm{H}}
\newcommand\BI{\bm{I}}

\newcommand\BM{\bm{M}}
\newcommand\BN{\bm{N}}

\newcommand\BX{\bm{X}}

\newcommand\BZ{\bm{Z}}
\newcommand\Bla{\bm{\lambda}}

\newcommand\Bmu{\bm{\mu}}
\newcommand\Bnu{\bm{\nu}}

\newcommand\Bth{\bm{\theta}}

\newcommand\Bph{\bm{\phi}}
\newcommand\BDe{\bm{\Delta}}

\newcommand\BPh{\bm{\Phi}}
\newcommand\BPs{\bm{\Psi}}
\newcommand\BSi{\bm{\Sigma}}

\newcommand\BGa{\bm{\Gamma}}
\newcommand\BTh{\bm{\Theta}}
\newcommand\BZe{\bm{0}}

 \newcommand{\dR}{\mathbb{R}}

 \newcommand{\dZ}{\mathbb{Z}}

\newcommand{\rE}{\mathrm{E}}

\newcommand{\rO}{\mathrm{O}} \newcommand{\rP}{\mathrm{P}}

 \newcommand{\cD}{\mathcal{D}}
 \newcommand{\cF}{\mathcal{F}}

 \newcommand{\cL}{\mathcal{L}}
 \newcommand{\cN}{\mathcal{N}}

\newcommand{\cU}{\mathcal{U}} 
 \newcommand{\cX}{\mathcal{X}}

 \newcommand{\Rd}{\mathrm{d}}

\newcommand{\Ro}{\mathrm{o}} \newcommand{\Rp}{\mathrm{p}}

\newcommand\EXP{\mathbf{\mathrm{E}}}

\newcommand\TR{\mathbf{\mathrm{Tr}}}

\renewcommand{\leq}{\leqslant}
\renewcommand{\geq}{\geqslant}

